\documentclass[12pt]{article}

\usepackage{url}            
\usepackage{booktabs}       
\usepackage{amsfonts}       

\usepackage{graphicx} 
\usepackage{tabularx}

\usepackage{algorithmic}
\usepackage{algorithm}

\usepackage{hyperref}
\hypersetup{colorlinks,linkcolor={blue},citecolor={blue},urlcolor={red}} 
\usepackage{natbib}
\usepackage{geometry}
\usepackage{fullpage}
\usepackage[T1]{fontenc}
\fontfamily{georgia}\selectfont

\usepackage{caption}
\usepackage{subcaption}

\usepackage{amsmath, amsthm, amssymb, nicefrac, color}

\usepackage{thmtools}
\usepackage{thm-restate}
\declaretheorem[name=Theorem,numberwithin=section]{theorem}
\theoremstyle{definition}

\theoremstyle{remark}

\theoremstyle{plain}
\newtheorem{lemma}[theorem]{Lemma}

\newtheorem{corollary}[theorem]{Corollary}
\newtheorem{fact}[theorem]{Fact}

\usepackage{mathtools}

\newcommand{\argmax}{\mathop{\mathrm{argmax}}}

\newcommand{\E}{\mathbb{E}}
\newcommand{\RR}{\mathbb{R}}

\newcommand{\epsDP}{\varepsilon}
\newcommand{\epsStop}{\alpha}

\newcommand{\deltaStop}{\beta}

\newcommand{\Ep}{\mathcal{E}}

\newcommand{\Xavg}{\overline{X_t}}
\newcommand{\Xavgk}{\overline{X_{2^k}}}
\newcommand{\Xst}{X_{1:t}}

\newcommand{\calP}{\mathcal{P}}
\newcommand{\calQ}{\mathcal{Q}}
\newcommand{\calM}{\mathcal{M}}

\newcommand{\dtv}{\mathrm{d}_{\rm TV}}

\newcommand{\muAvg}{\overline{\mu}}
\newcommand{\DeltaAvg}{\overline{\Delta}}
\newcommand{\muPriv}{\widetilde{\mu}}
\newcommand{\DeltaPriv}{\widetilde{\Delta}}
\newcommand{\poly}{\mathrm{poly}}
\newcommand{\Lap}{\mathsf{Lap}}

\renewcommand{\paragraph}[1]{\vspace{2mm}\noindent\textbf{#1}}

\begin{document}

\title{An Optimal Private Stochastic-MAB Algorithm\\Based on an Optimal Private Stopping Rule}

\author{Touqir Sajed \thanks{Department of Computing Science at the University of Alberta, touqir@ualberta.ca}
\and Or Sheffet \thanks{Department of Computing Science at the University of Alberta, osheffet@ualberta.ca}}
\maketitle





\begin{abstract}
We present a provably optimal differentially private algorithm for the stochastic multi-arm bandit problem, as opposed to the private analogue of the UCB-algorithm~\citep{mishra2015nearly, tossou2016algorithms} which doesn't meet the recently discovered lower-bound of $\Omega \left(\nicefrac{K\log(T)}{\epsilon} \right)$~\citep{shariff2018differentially}. Our construction is based on a different algorithm, Successive Elimination~\citep{even2002pac}, that repeatedly pulls all remaining arms until an arm is found to be suboptimal and is then eliminated. In order to devise a private analogue of Successive Elimination we visit the problem of private \emph{stopping rule}, that takes as input a stream of i.i.d samples from an unknown distribution and returns a \emph{multiplicative} $(1 \pm \alpha)$-approximation of the distribution's mean, and prove the optimality of our private stopping rule. We then present the private Successive Elimination algorithm which meets both the non-private lower bound~\citep{lai1985asymptotically} and the above-mentioned private lower bound. We also compare empirically the performance of our algorithm with the private UCB algorithm.
\end{abstract}

\section{Introduction}
\label{sec:intro}

The well-known \emph{stochastic multi-armed bandit} (MAB) is a
sequential decision-making task in which a learner repeatedly chooses
an action (or arm) and receives a noisy reward.  The learner's objective is to
maximize cumulative reward by \emph{exploring} the actions to discover
optimal ones (having the highest expected reward), balanced with
\emph{exploiting} them. The problem, originally stemming from experiments in medicine~\citep{robbins1952}, has applications in fields such as ranking~\citep{Kveton2015}, recommendation systems (collaborative filtering)~\citep{caron2013snakdd}, investment portfolio design~\citep{Hoffman2011} and online advertising~\citep{Schwartz2017}, to name a few. Such applications, relying on sensitive data, raise privacy concerns.

Differential privacy~\citep{DworkCalibratingNoiseSensitivity2006} has become in recent years the gold-standard for privacy preserving data-analysis alleviating such concerns, as it requires that the output of the data-analysis algorithm has a limited dependency on any single datum. Differentially private variants of online learning algorithms have been successfully devised in various settings~\citep{SmithT2013}, including a private UCB-algorithm for the MAB problem (details below)~\citep{mishra2015nearly, tossou2016algorithms} as well as UCB variations in the linear~\citep{KannanMRWW18} and contextual~\citep{shariff2018differentially} settings. 

More formally, in the MAB problem at every timestep $t$ the learner selects an arm $a$ out of $K$ available arms, pulls it and receives a random reward $r_{a, t}$ drawn i.i.d from a distribution $\calP_a$~--- of support $[0,1]$ and unknown mean $\mu_a$. The Upper Confidence Bound (UCB) algorithm for the MAB problem was developed in a series of works~\citep{BanditBook85,Agrawal95}
culminating in~\citep{Auer2002}, and is provably optimal for the MAB problem~\citep{lai1985asymptotically}. 
The UCB algorithm maintains a time-dependent high-probability upper-bound $B_{a,t}$ for each arm's mean, and at each timestep optimistically pulls the arm with the highest bound.
The above-mentioned $\epsDP$-differentially private ($\epsDP$-DP) analogues of the UCB-algorithm follow the same procedure except for maintaining noisy estimations $\widetilde{B_{a,t}}$ using the ``tree-based mechanism''~\citep{ChanPrivateContinualRelease2010,Dwork2010}. This mechanism continuously releases aggregated statistics over a stream of $T$ observations, introducing only $\nicefrac{\poly\log(T)}\epsDP$ noise in each timestep. The details of this poly-log factor are the focus of this work.

It was recently shown~\citep{shariff2018differentially} that any $\epsDP$-DP stochastic MAB algorithm\footnote{In this work, we focus on pure $\epsDP$-DP, rather than $(\epsDP,\delta)$-DP.} must incur an added pseudo regret of $\Omega(\nicefrac{K\log(T)}\epsDP)$. However, it is commonly known that any algorithm that relies on the tree-based mechanism must incur an added pseudo regret of $\omega(\nicefrac{K\log(T)}\epsDP)$. Indeed, the tree-based mechanism maintains a binary tree over the $T$ streaming observations, a tree of depth $\log_2(T)$, where each node in this tree holds an i.i.d sample from a $\Lap(\frac{\log_2(T)}{\epsDP})$ distribution. At each timestep $t$, the mechanism outputs the sum of the first $t$ observations added to the sum of the $\log_2(T)$ nodes on the root-to-$t$th-leaf path in the binary tree. As a result, the variance of the added noise at \emph{each} timestep is $\Theta(\frac{\log^3(T)}{\epsDP^2})$, making the noise per timestep $\omega(\nicefrac{\log(T)}{\epsDP})$. (In fact, most analyses\footnote{\citep{tossou2016algorithms} claim a $O(\nicefrac{\log(T)}{\epsDP})$ bound, but (i) rely on $(\epsDP,\delta)$-DP rather than pure-DP and more importantly (ii) ``sweep under the rug'' several factors that are themselves on the order of $\log(T)$.}\footnote{\citep{mishra2015nearly} shows a bound of $O(\nicefrac{\log^{3} (T)}{\epsDP})$} of the tree-based mechanism rely on the union bound over all $T$ timesteps, obtaining a bound of $\nicefrac{\log^{5/2}(T)}{\epsDP}$; consequentially the added-regret bound of the DP-UCB algorithm is $O(\frac{K\log^{2.5}(T)}{\epsDP})$.) Thus, in a setting where each of the $K$ tree-mechanisms (one per arm) is run over $\poly(T)$ observations (say, if all arms have suboptimality gap of $T^{-0.1}$), the private UCB-algorithm must unavoidably obtain an added regret of $\omega(\nicefrac{K\log(T)}\epsDP)$ (on top of the regret of the UCB-algorithm). It is therefore clear that the challenge in devising an \emph{optimal} DP algorithm for the MAB problem, namely an algorithm with added regret of $O(\nicefrac{K\log(T)}\epsDP)$, is \emph{algorithmic} in nature~--- we must replace the suboptimal tree-based mechanism with a different, simpler, mechanism. 

\paragraph{Our Contribution and Organization.} In this work, we present an optimal algorithm for the stochastic MAB-problem, which meets both the non-private lower-bound of~\citep{lai1985asymptotically} and the private lower-bound of~\citep{shariff2018differentially}. Our algorithm is a DP variant of the Successive Elimination (SE) algorithm~\citep{even2002pac}, a different optimal algorithm for stochastic MAB. SE works by pulling all arms sequentially, maintaining the same confidence interval around the empirical average of each arm's reward (as all remaining arms are pulled the exact same number of times); and when an arm is found to be noticeably suboptimal in comparison to a different arm, it is then eliminated from the set of viable arms (all arms are viable initially). To design a DP-analogue of SE we first consider the case of $2$ arms and ask ourselves~--- what is the optimal way to privately discern whether the gap between the mean rewards of two arms is positive or negative? This motivates the study of private \emph{stopping rules} which take as input a stream of i.i.d observations from a distribution of support $[-R,R]$ and unknown mean $\mu$, and halt once they obtain a $(1\pm\epsStop)$-approximation of $\mu$ with confidence of at least $1-\deltaStop$. Note that due to the multiplicative nature of the required approximation, it is impossible to straight-forwardly use the Hoeffding or Bernstein bounds; rather a stopping rule must alter its halting condition with time.
\citep{domingo2002adaptive} proposed a stopping rule known as the Nonmonotonic Adaptive Sampling (NAS) algorithm that relies on the Hoeffding's inequality to maintain a confidence interval at each timestep. They showed a sample complexity bound of $O\left(\frac{R^2}{\epsStop^2 \mu^2} \left( \log(\frac{R}{\deltaStop\cdot \epsStop|\mu|}) \right)\right)$, later improved slightly by~\citep{mnih2008empirical} to $O\left(\frac{R^2}{\epsStop^2 \mu^2} \left( \log(\frac{1}{\deltaStop}) + \log\log(\frac{R}{\epsStop|\mu|}) \right)\right)$. The work of~\citep{dagum2000optimal} shows an essentially matching sample complexity lower-bound. Stopping Rules have also been applied to Reinforcement Learning and Racing algorithms (See \citet{sajed2018, mnih2008empirical}).

In this work we introduce a $\epsDP$-DP analogue of the NAS algorithm that is based on the \emph{sparse vector technique} (SVT), with added sample complexity of (roughly) $O(\frac{R\log(1/\deltaStop)}{\epsDP\epsStop|\mu|})$. Moreover, we show that this added sample complexity is optimal in the sense that any $\epsDP$-DP stopping rule has a matching sample complexity lower-bound. After we introduce preliminaries in Section~\ref{sec:preliminaries}, we present the private NAS in Section~\ref{sec:DP-NAS}.
We then turn our attention to the design of the  private SE algorithm. Note that straight-forwardly applying $K$ private stopping rules yields a suboptimal algorithm whose regret bound is proportional to $K^2$. Instead, we \emph{partition} the algorithm's arm-pulls into epochs, where epoch $e$ is set to eliminate all arms with suboptimality-gaps greater than $2^{-e}$. By design each epoch must be at least twice as long as the previous epoch, and so we can reset (compute empirical means from fresh reward samples) the algorithm in-between epochs while incurring only a constant-factor increase to the regret bound. Note that as a side benefit our algorithm also solves the private Best Arm Identification problem, with provably optimal cost. Details appear in Section~\ref{sec:DP-SE}.
We also assess the empirical performance of our algorithm in comparison to the DP-UCB baseline and show that the improvement in analysis (despite the use of large constants) is also empirically evident; details provided in Section~\ref{sec:experiments}.
Lastly, future directions for this work are discussed in Section~\ref{sec:future_work}.

\paragraph{Discussion.} Some may find the results of this work underwhelming~--- after all the improvement we put forth is solely over $\poly\log$-factors, and admittedly they are already subsumed by the non-private regret bound of the algorithm under many ``natural'' settings of parameters. Our reply to these is two-fold. First, our experiments (see Section~\ref{sec:experiments}) show a significantly improved performance empirically, which is only due to the different algorithmic approach. Second, as the designers of privacy-preserving learning algorithms it is our ``moral duty'' to quantify the \emph{added} cost of privacy on top of the already existing cost, and push this added cost to its absolute lowest. 

We would also like to emphasize a more philosophical point arising from this work. Both the UCB-algorithm and the SE-algorithm are provably optimal for the MAB problem in the non-private setting, and are therefore equivalent. But the UCB-algorithm makes in each timestep an input-dependent choice (which arm to pull); whereas the SE-algorithm input-dependent choices are reflected only in $K-1$ special timesteps in which it declares ``eliminate arm $a$'' (in any other timestep it chooses the next viable arm). In that sense, the SE-algorithm is \emph{simpler} than the UCB-algorithm, making it the less costly to privatize between the two. In other words, {differential privacy} gives quantitative reasoning for preferring one algorithm to another because ``simpler is better.'' While not a full-fledged theory (yet), we believe this narrative is of importance to anyone who designs differentially private data-analysis algorithms.

\section{Preliminaries}
\label{sec:preliminaries}

\paragraph{Stopping Rules.} In the \emph{stopping rule} problem, the input consists of a stream of i.i.d samples $\{X_t\}_{t\geq 1}$ drawn from a distribution over an a-priori known support $[-R,R]$ and with unknown mean $\mu$. Given $\epsStop,\deltaStop \in (0,1)$, the goal of the stopping rule is to halt after seeing as few samples as possible while releasing a $(1\pm\epsStop)$-approximation of $\mu$ at halting time. Namely, a \emph{$(\epsStop,\deltaStop)$-stopping rule} halts at some time $t^*$ and releases $\hat{\mu}$ such that $\Pr [|\hat{\mu} - \mu| > \epsStop |\mu| ] <\deltaStop$. (It should be clear that the halting time $t^*$ increases as $|\mu|$ decreases.) During any timestep $t$, we denote $\Xst \stackrel{\rm def}{=}\sum_{i=1}^t X_i$ and $\Xavg \stackrel{\rm def}{=}\Xst / t$. 

\paragraph{Stochastic MAB and its optimal bounds.} The formal description of the stochastic  MAB problem was provided in the introduction. Formally, the bound maintained by the UCB-algorithm for each arm $a$ at a given timestep $t$ is $B_{a,t} \stackrel{\rm def}{=} \muAvg_a + \sqrt{\nicefrac{2\log(t)}{t_a}}$ with $\muAvg_a$ denoting the empirical mean reward from pulling arm $a$ and $t_a$ denoting the number of times $a$ was pulled thus far.  We use $a^*$ to denote the leading arm, namely, an arm of highest mean reward: $\mu_{a^*} = \max_{a=1}^K\{ \mu_a\}$. Given any arm $a$ we denote the mean-gap as $\Delta_a \stackrel{\rm def}{=}\mu_{a^*}-\mu_a$, with $\Delta_{a^*}=0$ by definition. Additionally we denote the horizon with $T$ - the number of rounds that a MAB algorithm will be run for. An algorithm that chooses arm $a_t$ at timestep $t$ incurs an \emph{expected regret} or \emph{pseudo-regret} of $\sum_{t} \Delta_{a_t}$. It is well-known~\citep{lai1985asymptotically} that any consistent\footnote{A regret minimization algorithm is called consistent if its regret is sub-polynomial, namely in $o(n^p)$ for any $p>0$.} regret-minimization algorithm must incur a pseudo-regret of $\Omega(\sum_{a\neq a^*} \frac {\log(T)} {\Delta_a})$; and indeed the UCB-algorithm meets this bound and has pseudo-regret of $O(\sum_{a\neq a^*} \frac {\log(T)} {\Delta_a})$. However, the minimax regret bound of the UCB-algorithm is $O(\sqrt{KT\log(T)})$, obtained by an adversary that knows $T$ and sets all suboptimal arms' gaps to $\sqrt{\nicefrac{K\log(T)}{T}}$, whereas the minimax lower-bound of any algorithm is slightly smaller: $\Omega(\sqrt{KT})$~\citep{Auer2002}.

\paragraph{Differential Privacy.}
In this work, we preserve  \emph{event-level} privacy under continuous observation~\citep{Dwork2010}. Formally, we say two streams are neighbours if they differ on a single entry in a single timestep $t$, and are identical on any other timestep. An algorithm $\calM$ is $\epsDP$-differentially private if for any two neighboring streams $S$ and $S'$ and for any set ${\cal O}$ of decisions made from timestep $1$ through $T$, it holds that $\Pr[\calM(S) \in {\cal O}] \leq e^\epsDP \Pr[\calM(S') \in {\cal O}]$. Note that much like its input, the output $\calM(S)$ is also released in a stream-like fashion, and the requirement should hold for all decisions made by $\calM$ in all timesteps.

In this work, we use two mechanisms that are known to be $\epsDP$-DP. The first is the Laplace mechanism~\citep{DworkCalibratingNoiseSensitivity2006}. Given a function $f$ that takes as input a stream $S$ and releases an output in $\RR^d$, we denote its global sensitivity as $GS(f) = \max_{S,S'} ||f(S) - f(S')||_1$; and the Laplace mechanism adds a random (independent) sample from $Lap(GS(f)/\epsDP)$ to each coordinate of $f(S)$. The other mechanism we use is the \emph{sparse-vector technique} (SVT), that takes in addition to $S$ a sequence of queries $\{q_i\}_{i}$ (each query has a global sensitivity $\leq GS$), and halts with the very first query whose value exceeds a given threshold. The SVT works by adding a random noise sampled i.i.d from $Lap(3GS/\epsDP)$ to both to the threshold and to each of the query-values. See~\citep{dwork2014algorithmic} for more details. 

\paragraph{Concentration bounds.}
A Laplace r.v.~$X\sim Lap(\lambda)$ is sampled from a distribution with ${\sf PDF}(x)\propto e^{-\nicefrac{|x|}{\lambda}}$. It is known that $\E[X]=0$, $\mathrm{Var}[X]=2\lambda^2$ and that for any $\tau>0$ it holds that $\Pr[ |X|>\tau ] = e^{-\nicefrac{\tau}{\lambda}}$.

Throughout this work we also rely on the Hoeffding inequality~\citep{hoeffding1963probability}. Given a collection $\{X_t\}_{t=1}^T$ of i.i.d random variables that take value in a finite interval of length $R$ with mean $\mu$, it holds that  $\Pr \left[ |\Xavg - \mu| \geq \alpha \right] \leq 2\exp \left( -\nicefrac{2 \alpha^2t}{R^2} \right)$.

\paragraph{Additional Notation and Facts.} Throughout this work $\log(x)$ denotes the logarithm base $e$ of $x$. Given two distributions $\calP$ and $\calQ$, we denote their \emph{total-variation} distance as $d_{\rm TV}(\calP,\calQ) = \sup\limits_{S}\left( \big| \Pr_{X\sim\calP}[X\in S] - \Pr_{X\sim\calQ}[X\in S] \big| \right)$.
We emphasize we made no effort to minimize constants throughout this work.
We also rely on the following folklore fact. For completeness, its proof is shown in Appendix Section \ref{apx_sex:proofs}.
 
\DeclareRobustCommand{\factLogLogSolution}{Fix any $a>1$ and any $0<b<\frac 1 {16}$. Then for any $e< x < \nicefrac{\log(a\log (1/b))}{b}$ it holds that $\cfrac{\log \left(a \log(x) \right) }{x} > b$, and for any $x > \nicefrac{2\log(a\log(1/b))}{b}$ it holds that $\cfrac{\log \left(a \log(x) \right) }{x}  < b$.}

\begin{fact} \label{fact:loglog_solution}
	\factLogLogSolution
\end{fact}



\DeclareRobustCommand{\factLogSolution}{Fix any $a>1$ and any $0<b<\frac 1 {16}$. Then for any $e < x < \nicefrac{\log(a/b))}{b}$ it holds that $\cfrac{\log \left(a \cdot x \right) }{x} > b$, and for any $x > \nicefrac{2\log(a/b))}{b}$ it holds that $\cfrac{\log \left(a \cdot x \right) }{x}  < b$.}


\section{Differentially Private Stopping Rule}
\label{sec:DP-NAS}


In this section, we derive a differentially private stopping rule algorithm, DP-NAS, which is based on the non-private NAS (Nonmonotonic Adaptive Sampling). The non-private NAS is rather simple. Given $\deltaStop$, denote $h_t$ as confidence interval derived by the Hoeffding bound with confidence $1-\nicefrac\deltaStop {2t^2}$ for $t$ iid random samples bounded in magnitude by $R$; thus, w.p. $\geq 1- \deltaStop$ it holds that $\forall t,  |\Xavg - \mu| \leq h_t$. The NAS algorithm halts at the first $t$ for which $|\Xavg| \geq h_t\left( \frac 1 \alpha + 1\right)$. Indeed, such a stopping rule assures that  $|\Xavg - \mu| \leq h_t \leq \epsStop(|\Xavg| - h_t) \leq \epsStop |\mu|$, the last inequality follows from $\bigg| |\Xavg| -  |\mu| \bigg| \leq |\Xavg - \mu| \leq h_t $.



In order to make NAS differentially private we use the sparse vector technique, since the algorithm is basically asking a series of threshold queries: $q_t \stackrel{\rm def}= |\Xavg| - h_t\left( \frac 1 \alpha + 1\right) \stackrel{?}\geq 0$. Recall that the sparse-vector technique adds random noise both to the threshold and to the answer of each query, and so we must adjust the na\"ive threshold of $0$ to some $c_t$ in order to make sure that $\Xavg$ is sufficiently close to $\mu$. Lastly, since our goal is to provide a private approximation of the distribution mean, we also apply the Laplace mechanism to $\Xavg$ to assert the output is differentially private. Details appear in Algorithm~\ref{alg:privateNAS}.

\begin{algorithm}
	\caption{DP-NAS \label{alg:privateNAS}}
	\begin{algorithmic}[1]
		\STATE Set $\sigma_1 \gets \nicefrac{12R}\epsDP$, $\sigma_2 \gets \nicefrac{12R} \epsDP$, $\sigma_3 \gets \nicefrac{4R}\epsDP$.
		\STATE Sample $B \sim Lap(\sigma_1)$.
		\STATE Initialize $t \gets 0$. 
		\REPEAT
		\STATE $t \gets t + 1$ 
		\STATE $A_t \sim Lap(\sigma_2)$
		\STATE Get a new sample $X_t$ and update the mean $\Xavg$.
		\STATE $h_t \gets R \sqrt{\frac{2}{t}\log(\frac{16t^2}{\deltaStop})}$
		\STATE $c_t \gets \sigma_1 \log(\nicefrac 4\deltaStop) + \sigma_2 \log(\nicefrac {8t^2}\deltaStop) + \frac{\sigma_3}{\epsStop} \log(\nicefrac 4 \deltaStop)$
		\UNTIL {$|\Xavg| \geq h_t(1+\frac{1}{\epsStop}) + \frac{ c_t + B + A_t}{t}$ }
		
		\STATE Sample $L \sim Lap(\sigma_3)$. 
		\STATE \textbf{return } $\Xavg + \frac{L}{t}$
	\end{algorithmic}
\end{algorithm}

\newcommand{\bvar}{R \sqrt{\cfrac{\log (8N/\deltaStop)}{2} }}
\newcommand{\bvarSqr}{\cfrac{R^2 \log (8N/\deltaStop)}{2}}
\newcommand{\cvara}{\sigma_3 \log(4/\deltaStop)}
\newcommand{\cvarb}{\cfrac{R}{\epsDP}\left( 4 \epsStop \log(4/\deltaStop) + 8\epsStop \log(4N/\deltaStop) + 2 \log(4/\deltaStop) \right)}

\DeclareRobustCommand{\ThmDPNASstoppingRule}
{
	Algorithm \ref{alg:privateNAS} is a $\epsDP$-DP $(\epsStop, \deltaStop)$-stopping rule.
}

\begin{theorem} \label{thm:DP_NAS}
	\ThmDPNASstoppingRule
\end{theorem}

\begin{proof}
	First, we argue that Algorithm~\ref{alg:privateNAS} is $\epsDP$-differentially private. This follows immediately from the fact that the algorithm is a combination of the sparse-vector technique with the Laplace mechanism. The first part of the algorithm halts when $|\sum_{i=1}^t X_i| - h_t \cdot t (\tfrac 1 \epsStop + 1) - c_t \geq A_t+B$. Indeed, this is the sparse-vector mechanism for a sum-query of sensitivity of no more than $2R$. It follows that sampling both the threshold-noise $B$ and the query noise $A_t$ from $Lap(3\cdot \frac 2 \epsDP \cdot 2R)$ suffices to maintain $\tfrac \epsDP 2$-DP. Similarly, adding a sample from $Lap(\tfrac 2 {t\epsDP} \cdot 2R)$ suffices to release the mean with $\tfrac \epsDP 2$-DP at the very last step of the algorithm.

	Since $\sum_{t\geq 1}t^{-2} \leq 2$, under the assumption that all $\{X_t\}$ are i.i.d samples from a distribution of mean $\mu$, the Hoeffding-bound and union-bound give that
	$\Pr[\exists t,~~ |\Xavg - \mu| > h_t ] \leq \nicefrac\deltaStop 4$. Standard concentration bound on the Laplace distribution give that $\Pr[|B| > \sigma_1 \log(\nicefrac 4 \deltaStop)]\leq \nicefrac \deltaStop 4$, $\Pr[ \exists t,~~|A_t| > \sigma_2 \log(\nicefrac {8t^2}\deltaStop)] \leq \nicefrac \deltaStop 4$, and $\Pr [ |L| > \sigma_3 \log(\nicefrac 4 \deltaStop)] \leq \nicefrac \deltaStop 4$. It follows that w.p. $\geq 1-\deltaStop$ none of these events happen, and so $\forall t,~ c_t\geq |B|  + |A_t| + |L|/\epsStop$.
	
	It follows that at the time we halt we have that 
	\begin{align*}
	|\Xavg - \mu| & \stackrel{\rm Hoeffding}\leq h_t 
	\cr & \leq \epsStop (|\Xavg| - h_t) -\frac \epsStop {t}(c_t + A_t + B) 
	\cr & \stackrel{(\ast)}\leq \epsStop  |\mu| -\frac \epsStop {t}(c_t + A_t + B) 
	\leq \epsStop |\mu| - \tfrac{|L|}{t}
	\end{align*}
	where $(\ast)$ is due to $\bigg| |\Xavg| - |\mu| \bigg| \leq |\Xavg - \mu| \leq h_t$. Therefore, we have that $|\Xavg +\frac L t - \mu| \leq |\Xavg-\mu| + \frac{|L|}{t} \leq \epsStop|\mu|$.
\end{proof}

Rather than analyzing the utility of Algorithm~\ref{alg:privateNAS}, namely, the high-probability bounds on its stopping time, we now turn our attention to a slight modification of the algorithm and analyze the revised algorithm's utility. The modification we introduce, albeit technical and non-instrumental in the utility bounds, plays a conceptual role in the description of later algorithms. We introduce Algorithm~\ref{alg:exp_privateNAS} where we exponentially reduce the number of SVT queries using standard doubling technique. Instead of querying the magnitude of the average at each timestep, we query it at exponentially growing intervals, thus paying no more than a constant factor in the utility guarantees while still reducing the number of SVT queries dramatically.

\newcommand{\nutilde}{\tilde{\nu}}
\newcommand{\Ctilde}{\tilde{C}}
\begin{algorithm}
	\caption{DP exponential NAS}
	\label{alg:exp_privateNAS}
	\begin{algorithmic}[1]
		\STATE Set $\sigma_1 \gets \nicefrac {12R} \epsDP$,~$\sigma_2 \gets \nicefrac {12R} \epsDP$,~$\sigma_3 \gets \nicefrac {4R}\epsDP$.
		\STATE Sample $B \sim Lap(\sigma_1)$
		\STATE Initialize $k\gets 0$ and $t \gets 0$.
		\REPEAT
		\STATE $k \gets k + 1$
		\REPEAT
		\STATE $t \gets t + 1$ 
		\STATE \text{Sample} $X_t$ and update $\Xavg$.
		\UNTIL{$t = 2^k$}
		\STATE $A_t \sim Lap(\sigma_2)$
		\STATE $c_t \gets \sigma_1 \log(\nicefrac 4 \deltaStop) + \sigma_2 \log(\nicefrac{8k^2} \deltaStop) + \frac{\sigma_3}{\epsStop} \log(\nicefrac 4 \deltaStop)$
		\STATE $h_t \gets R \sqrt{\frac{2}{t}\log(\frac{16k^2}{\deltaStop})}$
		\UNTIL{$|\Xavg| \geq h_t(1+\frac{1}{\epsStop}) + \frac{c_t + B + A_t }{t}$}
		\STATE $L \sim Lap(\sigma_3)$ 
		\STATE \textbf{return } $\Xavg + \frac{L}{t}$
	\end{algorithmic}
\end{algorithm}

\DeclareRobustCommand{\ColDPexpNASstoppingRule}
{
	Algorithm \ref{alg:exp_privateNAS} is a $\epsDP$-DP $(\epsStop, \deltaStop)$-stopping rule.
}
\begin{corollary} \label{col:DP_exp_NAS}
	\ColDPexpNASstoppingRule
\end{corollary}

\begin{proof}
	The only difference between Algorithms~\ref{alg:privateNAS} and~\ref{alg:exp_privateNAS} lies in checking the halting condition at exponentially increasing time-intervals, namely during times $t=2^k$ for $k\in\mathbb{N}$. The privacy analysis remains the same as in the proof of Theorem~\ref{thm:DP_NAS}, and the algorithm correctness analysis is modified by considering only the timesteps during which we checking for the halting condition. Formally, we denote $\Ep$ as the event where (i) $\forall k,~ |\Xavgk - \mu| \leq h_{2^k}$, (ii) $|B| \leq \sigma_1 \log(\nicefrac 4 \deltaStop)$, (iii) $\forall k,~~|A_{2^k}| \leq \sigma_2 \log(\nicefrac {8k^2}\deltaStop)$, and (iv) $|L| \leq \sigma_3 \log(\nicefrac 4 \deltaStop)$. Analogous to the proof of Theorem~\ref{thm:DP_NAS} we bound $\Pr[\Ep]\geq 1-\deltaStop$ and the result follows.
\end{proof}

\DeclareRobustCommand{\ThmDPexpNAScomplexity}
{
	Fix $\deltaStop \leq 0.08$ and $\mu\neq 0$. Let $\{X_t\}_t$ be an ensemble of i.i.d samples from any distribution over the range $[-R,R]$ and with mean $\mu$. Denote $t_0 \stackrel{\rm def}=\cfrac{R^2\log((\nicefrac{1} \deltaStop)\cdot \log(\frac{R}{\epsStop |\mu|}))}{\epsStop^2 \mu^2}$, $t_1 \stackrel{\rm def}= \cfrac{R\log((\nicefrac 1 \deltaStop) \cdot \log (\frac{R}{\epsStop |\mu|}))}{\epsDP |\mu|}$, $t_2 \stackrel{\rm def}= \cfrac{R \log(\nicefrac 1 \deltaStop)}{\epsDP \epsStop |\mu|}$. Then with probability at least $1 - \deltaStop$, Algorithm~\ref{alg:exp_privateNAS} halts by timestep $t_U = 2000(t_0+t_1+t_2)$.
}
\begin{theorem} \label{thm:DP_exp_NAS_complexity}
	\ThmDPexpNAScomplexity
\end{theorem}


\begin{proof}
	Recall the event $\Ep$ from the proof of Corollary~\ref{col:DP_exp_NAS} and its four conditions. We assume $\Ep$ holds and so the algorithm releases a $(1\pm\epsStop)$-approximation of $\mu$. To prove the claim, we show that under $\Ep$, at time $t_U$ it must hold that $|\Xavg| \geq h_t(1+\frac{1}{\epsStop}) + \frac{c_t + B + A_t }{t}$.
	
	Under $\Ep$ we have that $|\Xavg| \geq |\mu|- h_t$ and $\frac{c_t + B + A_t }{t}\leq \frac{2\sigma_1} t \log(\nicefrac 4 \deltaStop) + \frac{2\sigma_2} t \log(\nicefrac{8k^2} \deltaStop) + \frac{\sigma_3}{\epsStop t} \log(\nicefrac 4 \deltaStop)$; and so it suffices to show that $|\mu| \geq h_t(2+\frac 1 \alpha) + \frac{24R\log(\nicefrac 4 \deltaStop)} {\epsDP t}  + \frac{24R\log(\nicefrac{8k^2} \deltaStop)} {\epsDP t} + \frac{4R\log(\nicefrac 4 \deltaStop)}{\epsStop \epsDP t} $. In fact, since $\alpha<1$ we show something slightly stronger: that at time $t_U$ we have $|\mu| \geq \frac{3h_t} \alpha + \frac{48R\log(\nicefrac {8k^2} \deltaStop)} {\epsDP t}  + \frac{4R\log(\nicefrac 4 \deltaStop)}{\epsStop \epsDP t}$. This however is an immediate corollary of the following three facts.
	\vspace{-4mm}
	\begin{enumerate}
		\parskip=0pt
		\item For any $t\geq 1000t_0$ we have $\frac{\log(\nicefrac{4\log_2(t)}{\deltaStop})}{t} \leq \left(\frac{\epsStop|\mu|}{2\cdot 3\cdot 3 \cdot R}\right)^2 $, implying  $\frac{|\mu|}3 \geq \frac{3h_t} \alpha$.
		\item For any $t\geq 1000t_1$ we have $\frac{\log(\nicefrac{4\log_2(t)}{\deltaStop})}{t}\leq \frac{\epsDP|\mu|}{3\cdot 2\cdot 48 \cdot R} $, implying $\frac{|\mu|}3 \geq \frac{2\cdot 48R\log(\nicefrac {4k} \deltaStop)} {\epsDP t} \geq \frac{48R\log(\nicefrac {8k^2} \deltaStop)} {\epsDP t}$.
		\item For any $t\geq 48t_2$ we have $\frac{|\mu|}3 \geq \frac{4R\log(\nicefrac 4 \deltaStop)}{\epsStop \epsDP t}$.
	\end{enumerate}
	\vspace{-3mm}
	where the first two rely on Fact~\ref{fact:loglog_solution}. It follows therefore that at time $1000(t_0+t_1+t_2)$ all three conditions hold and so, due to the exponentially growth of the intervals, by time $t_u = 2000(t_0+t_1+t_2)$ we reach some $t$ which is a power of $2$, on which we pose a query for the SVT mechanism and halt. 
\end{proof}

\subsection{Private Stopping Rule Lower bounds}
\label{subsec:DPStoppingRuleLB}

We turn our attention to proving the (near) optimality of Algorithm~\ref{alg:exp_privateNAS}. 
A non-private lower bound was proven in~\citep{dagum2000optimal}, who showed no stopping rule algorithm can achieve a sample complexity better than $\Omega \left( \frac{\max\{\sigma^2, R \epsStop |\mu|\}}{\epsStop^2 \mu^2}  \log(\nicefrac 1 \deltaStop) \right)$ (with $\sigma^2$ denoting the variance of the underlying distribution). In this section, we prove a lower bound on the additional sample complexity that \emph{any} $\epsDP$-DP stopping rule algorithm must incur. 
We summarize our result below:

\DeclareRobustCommand{\ThmPrivacyLowerBound}
{Any $\epsDP$-differentially private $(\epsStop,\deltaStop)$-stopping rule whose input consists of a steam of i.i.d samples from a distribution over support $[-R,R]$ and with mean $\mu \neq0$,  must have a sample complexity of $\Omega \left( \nicefrac{R\log (\nicefrac 1\deltaStop)}{\epsDP \epsStop |\mu|}  \right)$.
}
\begin{theorem} \label{thm:stoppingrule_lowerbound_privacy}
    \ThmPrivacyLowerBound
\end{theorem}

\begin{proof}
	Fix $\epsDP,\epsStop, \deltaStop>0$ such that $\epsStop < 1$ and $\deltaStop < \nicefrac 1 4$, and fix $R$ and $\mu>0$. We define two distributions $\calP, \calQ$ over a support consisting of two discrete points: $\{-R,R\}$. Setting $\Pr_{\calP}[R] = \frac{1}2 + \frac{\mu}{2R}$ we have that $\E_{X\sim\calP}[X]=\mu$. Set $\mu'$ as any number infinitesimally below the threshold of $\frac{1-\epsStop}{1+\epsStop}\mu$, so that we have $(1+\epsStop)\mu' < (1-\epsStop)\mu$; we set the parameters of $\calQ$ s.t. $\Pr_{\calQ}[R] = \frac 1 2 + \frac {\mu'}{2R}$ so $\E_{X\sim\calQ}[X]=\mu'$.
	By definition, the total variation distance $\dtv(\calP,\calQ) = \frac {\mu-\mu'} {2R} = \frac {2\epsStop\mu}{2R(1+\epsStop)} < \frac {\epsStop\mu}{R}$.
	
	Let $\calM$ be any $\epsDP$-differentially private $(\epsStop,\deltaStop)$-stopping rule. Denote $n= \frac{R\log(\nicefrac{1} {\deltaStop})}{12\alpha\mu\epsDP}$. Let $\Ep$ be the event ``after seeing at most $n$ samples, $\calM$ halts and outputs a number in the interval $\big[(1-\alpha)\mu, (1+\alpha)\mu\big]$.'' We now apply the following, very elegant, lemma from~\citep{karwa2017finite}, stating that the group privacy loss of a differentially privacy mechanism taking as input $n$ i.i.d samples either from a distributions ${\cal D}$ or from a distribution ${\cal D}'$ scales effectively as $O(\epsDP n\cdot \dtv({\cal D},{\cal D}'))$.
	\begin{lemma}[Lemma~6.1 from~\citep{karwa2017finite}]
		\label{lem:KawraVadhanGroupInputPrivacy}
		Let $\calM$ be any $\epsDP$-differentially private mechanism, fix a natural $n$ and fix two distributions ${\cal D}$ and ${\cal D}'$, and let $\bar S$ and $\bar S'$ denote an ensemble of $n$ i.i.d samples taken from ${\cal D}$ and ${\cal D}'$ resp. Then for any possible set of outputs $O$ it holds that $\Pr[\calM(S)\in O] \leq e^{6\epsDP n\cdot  \dtv({\cal D},{\cal D}')} \Pr[\calM(S')\in O]$.
	\end{lemma}
	And so, applying $\calM$ over $n$ i.i.d samples taken from $\calQ$, we must have that $\Pr_{\calM, S\sim\calQ^n}[\Ep] \leq \beta$, since $(1-\epsStop)\mu>(1+\epsStop)\mu'$. Applying Lemma~\ref{lem:KawraVadhanGroupInputPrivacy} to our setting, we get 
	\begin{align*}
	\Pr_{\calM, S\sim\calP^n}[\Ep] &\leq e^{6\epsDP n\cdot  \dtv({\cal P},{\cal Q})} \Pr_{\calM, S\sim\calQ^n}[\Ep]
	\cr & \leq \beta\cdot \exp({6\epsDP n \cdot \frac {\epsStop\mu}{R} }) 
	\cr & = \beta \cdot \exp( \frac  {6\epsDP\epsStop\mu}{R} \cdot \frac{R\log(\nicefrac{1} {\deltaStop})}{12\epsDP\alpha\mu}) = \frac \beta {\sqrt{\beta}} < \frac 1 2
	\end{align*}
	since $\beta < \nicefrac 1 4$. Since, by definition, we have that the probability of the event $\Ep'$ ``after seeing at most $n$ samples, $\calM$ halts and outputs a number \emph{outside} the interval $\big[(1-\alpha)\mu, (1+\alpha)\mu\big]$'' over $n$ i.i.d samples from $\calP$ is at most $\beta$, then it must be that $\calM$ halts after seeing strictly more than $n$ samples w.p. $> 1 - (\nicefrac 1 2 +\beta) > \nicefrac 1 4$. 
\end{proof}

Combining the non-private lower bound of~\citep{dagum2000optimal} and the bound of Theorem~\ref{thm:stoppingrule_lowerbound_privacy}, we immediately infer the overall sample complexity bound, which follows from the fact that the variance of the distribution $\calP$ used in the proof of Theorem~\ref{thm:stoppingrule_lowerbound_privacy} has variance of $\Theta(R^2)$.

\DeclareRobustCommand{\ColPrivateStoppingRuleLowerBound}
{
	There exists a distribution $\calP$ for which any $\epsDP$-differentially private $(\epsStop,\deltaStop)$-stopping rule algorithm has a sample complexity of $\Omega \left( \frac{R^2\log(\nicefrac 1 \deltaStop)}{\epsStop^2 \mu^2}  + \frac{R\log(\nicefrac 1\deltaStop)}{\epsDP \epsStop |\mu|}\right)$.
}
\begin{corollary} \label{col:private_stoppingrule_lowerbound}
    \ColPrivateStoppingRuleLowerBound
\end{corollary}
\vspace{-2mm}
\paragraph{Discussion.} How optimal is Algorithm \ref{alg:exp_privateNAS}? The sample complexity bound in Theorem \ref{thm:DP_exp_NAS_complexity} can be interpreted as the sum of the non-private and private parts. The non-private part is $\Omega \bigg(\cfrac{R^2}{\epsStop^2 \mu^2} \big(\log(1/ \deltaStop)$ $+ \log \log \frac{R}{\epsStop |\mu|} \big) \bigg)$ and the private part is $ \Omega \bigg( \cfrac{R}{\epsDP |\mu|} \big(\log(1/ \deltaStop) + \log \log \frac{R}{\epsStop |\mu|} \big)$ $+ \cfrac{R \log(1/\deltaStop)}{\epsDP \epsStop |\mu|} \bigg)$.
If we add in the assumption that $\log(\frac{R}{\epsStop |\mu|}) \leq \nicefrac 1 \deltaStop$ we get that the upper-bound of Theorem~\ref{thm:DP_exp_NAS_complexity} matches the lower-bound in Corollary~\ref{col:private_stoppingrule_lowerbound}.

How benign is this assumption? Much like in~\citep{mnih2008empirical}, we too believe it is a very mild assumption. Specifically, in the next section, where we deal with finite sequences of length $T$, we set $\deltaStop$ as proportional to $\nicefrac 1 T$. Since over finite-length sequence we can only retrieve an approximation of $\mu$ if $\frac {|\mu|}{R} \gg \frac 1 T$, requiring $\frac{R}{|\mu|} < 2^{T}$ is trivial. However, we cannot completely disregard the possibility of using a private stopping rule in a setting where, for example, both $\epsStop,\deltaStop$ are constants whereas $\frac {|\mu|}{R}$ is a sub-constant. In such a setting, $\log(\frac{R}{\epsStop |\mu|})$ may dominate $\nicefrac 1 \deltaStop$, and there it might be possible to improve on the performance of Algorithm~\ref{alg:exp_privateNAS} (or tighten the bound).
\newcommand{\Kpos}{K_2^+}
\newcommand{\Kneg}{K_2^-}
\newlength\myindent
\setlength\myindent{2em}
\newcommand\bindent{%
  \begingroup
  \setlength{\itemindent}{\myindent}
  \addtolength{\algorithmicindent}{\myindent}
}
\newcommand\eindent{\endgroup}

\newcommand{\mubar}{\bar{\mu}}
\newcommand{\mutilde}{\widetilde{\mu}}

\section{An Optimal Private MAB Algorithm}
\label{sec:DP-SE}

In this section, our goal is to devise an optimal $\epsDP$-differentially private algorithm for the stochastic $K$-arms bandit problem, in a setting where all rewards are between $[0,1]$. We denote the mean reward of each arm as $\mu_a$, the best arm as $a^*$, and for any $a\neq a^*$ we refer to the gap $\Delta_a = \mu_{a^*}-\mu_a$. We seek in the optimal algorithm in the sense that it should meet both the non-private instance-dependent bound of~\citep{lai1985asymptotically} and the lower bound of \citep{shariff2018differentially}; namely an algorithm with an instance-dependent pseudo-regret bound of $O \left( \frac{K \log(T)}{\epsDP} + \sum_{a \neq a^*} \frac{\log(T)}{\Delta_a}  \right)$. The algorithm we devise is a differentially private version of the Successive Elimination (SE) algorithm~\citep{even2002pac}.
SE initializes by setting all $K$ arms as viable options, and iteratively pulls all viable arms maintaining the same confidence interval around the empirical average of each viable arm's reward. Once some viable arm's upper confidence bound is strictly smaller than the lower confidence bound of some other viable arm, the arm with the lower empirical reward is eliminated and is no longer considered viable.  It is worth while to note that the classical UCB algorithm and the SE algorithm have the same asymptotic pseudo-regret. 
To design the differentially private analouge of SE, we use our results from the previous section regarding stopping rules. After all, in the special case where we have $K=2$ arms, we can straight-forwardly use the private stopping-rule to assess the mean of the difference between the arms up to a constant $\epsStop$ (say $\epsStop = 0.5$). The question lies in applying this algorithm in the $K>2$ case.

Here are a few failed first-attempts. The most straight-forward ideas is to apply $\binom K 2$ stopping rules / SVTs for all pairs of arms; but since a reward of a single pull of any single arm plays a role in $K-1$ SVT instantiations, it follows we would have to scale down the privacy-loss of each SVT to $\Theta(\nicefrac \epsDP K)$ resulting in an added regret scaled up by a factor of $K$. In an attempt to reduce the number of SVT-instantiations, we might consider asking for each arm whether \emph{there exists} an arm with a significantly greater reward, yet it still holds that the reward from a single pull of the \emph{leading} arm $a^*$ plays a role in $K$ SVT-instantiations. Next, consider merging all queries into a single SVT, posing in each round $K$ queries (one per arm) and halting once we find that a certain arm is suboptimal; but this results in a single SVT that may halt $K-1$ times, causing us yet again to scale $\epsDP$ by a factor of $K$. 

In order to avoid scaling down $\epsDP$ by a factor of $K$, our solution leverages on the combination of parallel decomposition and geometrically increasing intervals. Namely we partition the arm pulls of the algorithm into \emph{epochs} of geometrically increasing lengths, where in epoch $e$ we eliminate \emph{all} arms of optimality-gap $\geq 2^{-e}$. In fact, it turns out we needn't apply the SVT at the end of each epoch\footnote{We thank the anonymous referee for this elegant observation.} but rather just test for a noticeably underperforming arm using a private histogram. The key point is that at the beginning of each new epoch we nullify all counters and start the mean-reward estimation completely anew (over the remaining set of viable arms) --- and so a single reward plays a role in only one epoch, allowing for $\epsDP$-DP mean-estimation in each epoch (rather than $\epsDP/K$). Yet due to the fact that the epochs are of exponentially growing lengths the total number of pulls for any suboptimal arm is proportional to the length of the epoch in which it eliminated, resulting in only a constant factor increase to the regret.
The full-fledged details appear in Algorithm~\ref{alg:DP_SE}.

\begin{algorithm}[ht!]
\caption{DP Successive Elimination}
\label{alg:DP_SE}
\begin{algorithmic}[1]
\INPUT $K$ arms, confidence $\deltaStop$, privacy-loss $\epsDP$.
\STATE Let $S \gets \{1, \ldots, K \}$. 
\STATE Initialize: $t\gets 0$, $epoch\gets 0$.
\REPEAT 
    \STATE Increment $epoch \gets epoch+1$.
    \STATE Set $r\gets 0$
    \STATE Zero all means: $\forall i \in S$ set $\mubar_i \gets 0$
    \STATE Set $\Delta_e \gets 2^{-epoch}$
    \STATE Set $R_e \gets \max\left(\frac{32 \log(\nicefrac{8 |S| epoch^2}{\beta})}{\Delta_e^2} , \frac{8 \log(\nicefrac{4 |S| epoch^2}{\beta})}{\epsDP \Delta_e} \right) + 1$
    \WHILE{$r < R_e$}
        \STATE Increment $r \gets r + 1$. 
        \STATE \textbf{foreach} $i\in S$
            \STATE \qquad Increment $t\gets t+1$
            \STATE \qquad Sample reward of arm $i$, update mean $\mubar_i$.
    \ENDWHILE    
    
    \STATE Set $h_{e} \gets \sqrt{\frac{{\log \left( \nicefrac{8 |S| \cdot epoch^2   } {\deltaStop} \right)}}{2R_e}}$
    \STATE Set $c_{e} \gets \frac{\log \left(\nicefrac{4|S| \cdot epoch^2}{ \deltaStop}\right)}{R_e \epsDP}$
    
    \STATE \textbf{foreach} $i \in S$ set $\mutilde_i \gets \mubar_i + \Lap(\nicefrac{1}{\epsDP r})$
    \STATE Let $\mutilde_{\max} = \max_{i \in S} \mutilde_i$
    \STATE Remove all arm $j$ from $S$ such that:
    \STATE \hspace{15pt} $\mutilde_{\max} - \mutilde_j > 2h_{e} + 2c_{e}$
\UNTIL{$|S| = 1$}
\STATE Pull the arm in $S$ in all remaining rounds.
\end{algorithmic}
\end{algorithm}

\newcommand{\cut}[1]{}
\cut{
\begin{algorithm}[ht!]
\caption{DP Successive Elimination}
\label{alg:DP_SE}
\begin{algorithmic}[1]
\INPUT $K$ arms, confidence $\deltaStop$, privacy-loss $\epsDP$.
\STATE Let $S \gets \{1,..,K \}$. 
\STATE Initialize: $t\gets 0$, $epoch\gets 0$.
\REPEAT 
    \STATE Increment $epoch \gets epoch+1$.
    \STATE Set $r\gets 0$, $\ell\gets 0$.
    \STATE Zero all means: $\forall i$ set $\mubar_i = 0$
    \STATE Sample $B \sim Lap(\nicefrac 6 \epsDP)$
    \REPEAT 
        \STATE Increment $\ell \gets \ell + 1$. 
        \REPEAT
            \STATE Increment $r \gets r + 1$
            \STATE \textbf{foreach} $i\in S$
            \STATE \qquad Increment $t\gets t+1$
            \STATE \qquad Sample reward of arm $i$, update mean $\mubar_i$.
        \UNTIL{$r \geq 2^\ell$}    
        \STATE Sample $A_r \sim Lap(\nicefrac 6  \epsDP)$
        \STATE Set $h_r \gets \sqrt{\frac{{\log \left( \frac{16K |S|\ell^2   } {\deltaStop} \right)}}{2r}}$, and set\\ $c_{r} \gets \frac{6\log \left(\nicefrac {4K} \deltaStop \right)}\epsDP + \frac{6 \log \left(\nicefrac{8K\ell^2} \deltaStop \right)}{\epsDP} + \frac{16\log \left(\nicefrac{4K|S|} \deltaStop \right)}{\epsDP}$
    \UNTIL{$\max_{i,j \in S} \left( \mubar_i - \mubar_j \right) > 10h_r + \frac{A_r + B + c_{r}}{r}$} \\
    \STATE \textbf{foreach} $i \in S$ set $\mutilde_i \gets \mubar_i + Lap(2/\epsDP R_e)$
    \STATE Let $\mutilde_{\max} = \max_{i \in S} \mutilde_i$
    \STATE Remove all arm $j$ from $S$ such that:
    \STATE \hspace{15pt} $\mutilde_{\max} - \mutilde_j > 2h_r + \frac{4\log \left( 4K|S| / \deltaStop \right)}{\epsDP r}$
\UNTIL{$|S| = 1$}
\STATE Pull the arm in $S$ in all remaining timesteps.
\end{algorithmic}
\end{algorithm}


}
\DeclareRobustCommand{\ThmDPSEprivacy}
{
	Algorithm~\ref{alg:DP_SE} is $\epsDP$-differentially private.
}

\begin{theorem} \label{thm:DP_SE_privacy}
    \ThmDPSEprivacy
\end{theorem}
\vspace{-5mm}
\begin{proof}
     Consider two streams of arm-rewards that differ on the reward of a single arm in a single timestep. This timestep plays a role in a single epoch $e$. Moreover, let $a$ be the arm whose reward differs between the two neighboring streams. Since the reward of each arm is bounded by $[0,1]$ it follows that the difference of the mean of arm $a$ between the two neighboring streams is $\leq 1/R_e$. Thus, adding noise of $\Lap(\nicefrac{1}{\epsDP R_e})$ to $\mu_a$ guarantees $\epsDP$-DP.
\end{proof}

To argue about the optimality of Algorithm~\ref{alg:DP_SE}, we require the following lemma, a key step in the following theorem that bounds the pseudo-regret of the algorithm.

\DeclareRobustCommand{\LemDPSEregretBound}
{
	Fix any instance of the $K$-MAB problem, and denote $a^*$ as its optimal arm (of highest mean), and the gaps between the mean of arm $a^*$ and any suboptimal arm $a\neq a^*$ as $\Delta_a$. Fix any horizon $T$. Then w.p. $\geq 1 - \deltaStop$ it holds that Algorithm~\ref{alg:DP_SE} pulls each suboptimal arm $a\neq a^*$ for a number of timesteps upper bounded by
	\[ \min\{ T, ~~O\left(  \left( \log(\nicefrac K  \deltaStop) + \log \log (\nicefrac 1 \Delta_{a}) \right) \left( \frac{1}{\Delta^2_{a}} + \frac{1}{\epsDP \Delta_{a}} \right) \right)   \}  \] 
}
\begin{lemma} \label{lem:DP_SE_numpulls}
    \LemDPSEregretBound
\end{lemma}
\DeclareRobustCommand{\pfMainLemmaOldDPSE}{
To bound the number of pulls of arm $a$ by $T$ is trivial; to provide the bound that depends on the gap $\Delta_a$ we bound the number of epochs in all rounds where arm $a$ is still viable. First we introduce some notations for convenience. We sort the arms in terms of their true means in a descending order: $\mu_1 \geq \mu_2 \geq ... \geq \mu_K$, where $\mu_a$ is the mean of the $a$-th arm. Hence, their corresponding suboptimality gaps are sorted in an ascending order: $\Delta_2 \leq ... \leq \Delta_K$, where $\Delta_a = \mu_1 - \mu_a$, and we also denote $\Delta_{a,{a'}} = \mu_a - \mu_{a'}$. We denote by $\muAvg_a$ the empirical average of each arm $a$, and denote the empirical gap by $\DeltaAvg_{a}=\muAvg_1 - \muAvg_a$, and similarly denote $\DeltaAvg_{a,{a'}}=\muAvg_a - \muAvg_{a'}$. Lastly, just like in Algorithm~\ref{alg:DP_SE}, we denote the private estimation of an arm's average by $\muPriv_a$ (the empirical average with added Laplace noise), and analogously denote $\DeltaPriv_a = \muPriv_1 - \muPriv_a$, $\DeltaPriv_{a,a'} = \muPriv_a-\muPriv_{a'}$. We refer to a sequence of pulls of all viable arms (arms in $S$) made by algorithm as a \emph{round}, indexed by $r$. Just like in the proof of Theorem~\ref{thm:DP_exp_NAS_complexity}, since we have at most $K$ epochs and in each epoch $|S|$ viable arms, and since $\sum_{\ell \geq 1}\frac 1 {2\ell^2}\leq 1$, then: (i) The Hoeffding bound gives that in all epochs and in all rounds where we query the SVT and for all viable arms we have $|\mu_a - \muAvg_a|\leq h_r$ w.p. $\geq 1-\frac \deltaStop 4$; (ii) Laplace concentration bounds give that in all epochs and in all rounds where we query the SVT and for each of the $|S|$ viable arms in an epoch, it must hold that $|B| + |A_r| \leq \frac{6\log \left(\nicefrac {4K} \deltaStop \right)+6 \log \left(\nicefrac{8Kl^2} \deltaStop \right)}\epsDP$ (under the same notation introduced in Algorithm~\ref{alg:DP_SE}) w.p. $\geq 1- \frac{\deltaStop}{2}$; and (iii) Laplace concentration bounds give that in all epochs and for each of the $|S|$ viable arms in an epoch we have $|\muAvg_a - \muPriv_a| \leq \frac{2\log(\nicefrac{4K|S|}{\deltaStop})}\epsDP$ w.p. $\geq 1 - \frac \deltaStop 4$. We thus continue assuming all three bounds hold. In particular at the end of each epoch, for all the $|S|$ viable arms in the respective epoch we get that $|\muPriv_a - \mu_a| \leq |\mu_a - \muAvg_a| + |\muAvg_a - \muPriv_a| \leq h_r + \frac{2\log(\nicefrac{4K|S|}{\deltaStop})}\epsDP$; and so it follows that for any pair of arms $a, a'$ we have $|\DeltaPriv_{a,a'} - \Delta_{a,{a'}}| \leq 2h_r +\frac{4\log(\nicefrac{4K|S|}{\deltaStop})}\epsDP$.

Fix an epoch $e$, denote $j_e = \argmax_{i \in S} \Delta_i$~--- the viable arm with the largest gap in this epoch, and denote its gap as $\Delta_e$. Since Algorithm~\ref{alg:DP_SE} applies in each epoch the private stopping rule detailed in Algorithm~\ref{alg:exp_privateNAS} with $\alpha = \nicefrac 1 4$, we can use the bound given in Theorem~\ref{thm:DP_exp_NAS_complexity} and deduce that the epoch terminates within $r_e \leq 40000\left( \log(\nicefrac K  \deltaStop) + \log \log (\nicefrac 1 {\Delta_{e}}) \right) \left( \frac{1}{\Delta^2_{e}} + \frac{1}{\epsDP \Delta_{e}} \right)$ rounds. We show that Algorithm~\ref{alg:DP_SE} eliminates arm $j_e$ as well as any arm $a \in S$ for which $\Delta_a \geq \nicefrac{\Delta_e}{2}$.

Let $a^1$ and $a^2$ denote the pair of arms whose large gap in empirical means causes the SVT to halt. Namely, the arms such that $\DeltaAvg_{a^1,a^2} > 10h_r + \frac{A_r + B + c_r} r > 10h_r + \frac{16\log(\nicefrac{K|S|}{\deltaStop})}r$. Since $|\Delta_{a^1,a^2} - \DeltaAvg_{a^1,a^2}|\leq 2h_r$ it follows that $\Delta_{e} \geq \Delta_{a^1, a^2} >8h_r + \frac{16\log(\nicefrac{K|S|}{\deltaStop})}r$. Now, consider any arm $a\in S$ such that $\Delta_a \geq \nicefrac {\Delta_e}2 > 4h_r + \frac{8\log(\nicefrac{K|S|}{\deltaStop})}r$. By the above discussion we have that for the arm with the highest private mean estimation $\muPriv_{\max}$ it holds that $\muPriv_{\max}-\muPriv_a> \muPriv_1 - \muPriv_a > \Delta_a - 2h_r - \frac{4\log(\nicefrac{4K|S|}{\deltaStop})}\epsDP > 2h_r + \frac{4\log(\nicefrac{4K|S|}{\deltaStop})}\epsDP$, and so the arm $a$ is eliminated by the algorithm. Also note that by the same bound, we have that $\muPriv_{\max} - \muPriv_1 \leq 0 + 2h_r +\frac{4\log(\nicefrac{4K|S|}{\deltaStop})}\epsDP$ so arm 1 (the leading arm) is never eliminated.

We leverage on the above to infer a bound on the number of pulls made on any suboptimal arm $a$. Consider any suboptimal arm $a$ and let $e$ denote the last epoch in which this arm was viable (the last epoch where $a\in S$), and note that it could be that $e$ is the last epoch of the algorithm and arm $a$ is never eliminated. By definition, $\Delta_a \leq \Delta_e$, and so the number of pulls of arm $a$ in epoch $e$ is at most $r_e \leq \left( \frac{40000}{\Delta^2_{a}} + \frac{40000}{\epsDP \Delta_{a}} \right)\left( \log(\nicefrac K  \deltaStop) + \log \log (\nicefrac 1 {\Delta_{a}}) \right)$. Moreover, arm $a$ was pulled during epochs $1,2,... e-1$ as well, but by the above argument we have that the largest gap in epoch $e-1$ had to be at least $2\Delta_e\geq 2\Delta_a$, in epoch $e-2$~--- at least $4\Delta_e\geq 4\Delta_a$, and so on until epoch $1$ where the gap was at least $2^{e-1}\Delta_a$. Thus the total number of pulls of arm $a$ is at most $\sum\limits_{m=1}^e r_m \leq \left( \log(\nicefrac {K\log (\nicefrac 1 {\Delta_{a}})}  \deltaStop)  \right)\sum\limits_{m=0}^{e-1} \left( \frac{40000}{2^{2m}\Delta^2_{a}} + \frac{40000}{2^m\epsDP \Delta_{a}} \right) \leq$\break $\left( \log(\nicefrac {K\log (\nicefrac 1 {\Delta_{a}})}  \deltaStop)  \right) \left( \frac{80000}{\Delta^2_{a}} + \frac{80000}{\epsDP \Delta_{a}} \right)$, where the last inequality follows from a sum of a geometric series.
}
\DeclareRobustCommand{\pfMainLemmaDPSE}{
    The bound of $T$ is trivial so we focus on proving the latter bound. Given an epoch $e$ we denote by $\Ep_e$ the event where for \emph{all arms} $a \in S$ it holds that both (i) $|\mu_a - \mubar_a|\leq h_e$ and (ii) $|\mubar_a - \mutilde_a| \leq c_e$; and also denote $\Ep = \bigcup\limits_{e\geq 1} \Ep_e$. The Hoeffding bound, concentration of the Laplace distribution and the union bound over all arms in $S$ give that $\Pr[\Ep_e] \geq 1 - \left( \frac{\deltaStop}{4 e^2} +\frac{\deltaStop}{4 e^{2}}\right)$, thus $\Pr[\Ep] \geq 1 - \frac \deltaStop 2\left(\sum_{e\geq 1} e^{-2}\right)  \geq 1 - \deltaStop$. The remainder of the proof continues under the assumption the $\Ep$ holds, and so, for any epoch $e$ and any viable arm $a$ in this epoch we have $|\mutilde_a-\mu_a| \leq h_e + c_e$. As a result for any epoch $e$ and any two arms $a^1, a^2$ we have that $|(\mutilde_{a^1}-\mutilde_{a^2})-(\mu_{a^1}-\mu_{a^2})| \leq 2h_e+2c_e$.
    
    Next, we argue that under $\Ep$ the optimal arm $a^*$ is never eliminated. Indeed, for any epoch $e$, we denote the arm $a_e = \argmax_{a \in S} \mutilde_a$ and it is simple enough to see that $\mutilde_{a_e}-\mutilde_{a^*} \leq 0 + 2h_e+2c_e$, so the algorithm doesn't eliminate $a^*$.
    
    Next, we argue that, under $\Ep$, in any epoch $e$ we eliminate all viable arms with suboptimality gap $\geq 2^{-e} = \Delta_e$. Fix an epoch $e$ and a viable arm $a$ with suboptimality gap $\Delta_a \geq \Delta_e$. 
    Note that we have set parameter $R_e$ so that
    \begin{align*}
        h_e &= \sqrt{\frac{{\log \left( \nicefrac{8 |S| \cdot e^2   } {\deltaStop} \right)}}{2R_e}} < \sqrt{\frac{{\log \left( \nicefrac{8 |S| \cdot e^2   } {\deltaStop} \right)}}{2\cdot \frac{32 \log(\nicefrac{8 |S| e^2}{\beta})}{\Delta_e^2}}} = \frac{\Delta_e}8
        \cr 
        c_e &= \frac{\log \left(\nicefrac{4|S| \cdot e^2}{ \deltaStop}\right)}{R_e \epsDP} < \frac{\log \left(\nicefrac{4|S| \cdot e^2}{ \deltaStop}\right)}{\epsDP\cdot \frac{8 \log(\nicefrac{4 |S| e^2}{\beta})}{\epsDP \Delta_e}} = \frac{\Delta_e}8
    \end{align*}
    Therefore, since arm $a^*$ remains viable, we have that $\mutilde_{\max}-\mutilde_a \geq \mutilde_{a^*} - \mutilde_{a} \geq \Delta_a - (2h_e +2c_e) > \Delta_e(1-\tfrac 2 8-  \tfrac 2 8) \geq \frac{\Delta_e} 2 > 2h_e+2c_e$, guaranteeing that arm $a$ is removed from $S$.
    
    Lastly, fix a suboptimal arm $a$ and let $e(a)$ be the first epoch such that $\Delta_a \geq \Delta_{e(a)}$, implying $\Delta_{e(a)} \leq \Delta_a < \Delta_{e(a)-1}=2\Delta_e$. Using the immediate observation that for any epoch $e$ we have $R_e \leq R_{e+1}/ 2$, we have that the total number of pulls of arm $a$ is
    \begin{align*}
    \sum_{e\leq e(a)} R_e \leq \sum_{e\leq e(a)} 2^{e-e(a)} R_{e(a)} \leq R_{e(a)} \sum_{i\geq 0}2^{-i} \leq 2\left( \frac{32 \log(\nicefrac{8 |S|\cdot e(a)^2}{\beta})}{\Delta_e^2} + \frac{8 \log(\nicefrac{4 |S|\cdot e(a)^2}{\beta})}{\epsDP \Delta_e} \right)
    \end{align*}
    The bounds $\Delta_e > \Delta_a/2$, $|S|\leq K$, $e(a) < \log_2(2/\Delta_a)$ and $K \geq 2$ allow us to conclude and infer that under $\Ep$ the total number of pulls of arm $a$ is at most
    $\log (K\log (\nicefrac{2}{\Delta_a}) / \deltaStop)  \left( \frac{1024}{\Delta_a^2} + \frac{96}{\epsDP \Delta_a} \right)$.
}

\begin{proof}[Proof of Lemma~\ref{lem:DP_SE_numpulls}]
\pfMainLemmaDPSE
\end{proof}

\DeclareRobustCommand{\ColDPSEinstanceDependentRegretBound}
{
	Under the same notation as in Lemma~\ref{lem:DP_SE_numpulls} and for sufficiently large $T$, the expected regret of Algorithm~\ref{alg:DP_SE} is at most
	$O \left( \left(\sum\limits_{a\neq a^*}  \frac{\log(T)}{\Delta_{a}}\right) + \frac{K\log(T)}{\epsDP} \right)$. 
}

\DeclareRobustCommand{\pfThmRegretInstanceDependent}{
\begin{proof}
	In order to bound the expected regret based on the high-probability bound given in Lemma~\ref{lem:DP_SE_numpulls}, we must set $\beta = \nicefrac 1 T$. (Alternatively, we use the standard guess-and-double technique when the horizon $T$ is unknown. I.e. we start with a guess of $T$ and on time $\nicefrac T 2$ we multiply the guess $T\gets 2T$.) Thus, with probability at most $\tfrac 1 T$ we may pull a suboptimal on all timesteps incurring expect regret of at most $1 \cdot T \cdot \frac 1 T = 1$; and with probability $\geq 1 - \frac 1 T$, since each time we pull a suboptimal arm $a\neq a^*$ we incur an expected regret of $\Delta_a$, our overall expected regret when $T$ is sufficient large is proportional to at most 
	\begin{align*}
	& \sum_{a\neq a^*}  \left( \log(\nicefrac K {(1/T)}) + \log \log (\nicefrac 1 {\Delta_{a}}) \right) \left( \frac{\Delta_a}{\Delta_{a}^2} + \frac{\Delta_a}{\epsDP\Delta_a} \right) \\
	&~~~= \sum_{a\neq a^*} \left( \log(TK\cdot \log (\nicefrac 1 {\Delta_{a}}) \right) \left( \frac{1}{\Delta_{a}} + \frac{1}{\epsDP} \right)  \\
	&~~~\leq \left(\sum_{a\neq a^*} \frac{3\log(T)}{\Delta_{a}}\right) + \frac{3\log(T)(K-1)}{\epsDP}  
	\end{align*}
	where the last inequality follows from the trivial bounds $T \geq K$ and $T \geq 1/\Delta_a$.
\end{proof}
}

\begin{theorem} \label{thm:DP_SE_instanceDependent_regret}
	\ColDPSEinstanceDependentRegretBound
\end{theorem}

\pfThmRegretInstanceDependent

It is worth noting yet again that the expected regret of Algorithm~\ref{alg:DP_SE} meets both the (instance dependent) non-private lower bound~\citep{lai1985asymptotically} of $\Omega \left(\sum_{a\neq a^*}  \frac{\log(T)}{\Delta_{a}}\right)$ and the private lower bound~\citep{shariff2018differentially} of $\Omega\left(\nicefrac{K\log(T)}{\epsDP}\right)$.

\paragraph{Minimax Regret Bound.} The bound of Theorem~\ref{thm:DP_SE_instanceDependent_regret} is an instance-dependent bound, and so we turn our attention to the minimax regret bound of Algorithm~\ref{alg:DP_SE}~--- Given horizon bound $T$, how should an adversary set the gaps between the different arms as to maximize the expected regret of Algorithm~\ref{alg:DP_SE}? We next show that in any setting of the gaps, the following is an instance independent bound on the expected regret of Algorithm~\ref{alg:DP_SE}.

\DeclareRobustCommand{\ThmDPSEinstanceIndependentRegretBound}
{
	The pseudo regret of Algorithm~\ref{alg:DP_SE} is $O \big( \sqrt{T K \log(T)}$ $+ ~\nicefrac{K \log(T)}{\epsDP} \big)$.
}

\DeclareRobustCommand{\pfThemInstanceIndependentRegret}{
\begin{proof}
	Throughout the proof we assume Algorithm~\ref{alg:DP_SE} runs with a parameter $\deltaStop = 1/T$; and since any arm $a$ with $\Delta_a < 1/T$ yields a negligible expected regret bound of at most $1$, then we may assume $\Delta_a\geq 1/T$. Thus, the bound of Lemma~\ref{lem:DP_SE_numpulls} becomes $\min\left\{T,~~C\cdot \log(TK)(\frac{1}{\Delta_{a}^2}+\frac{1}{\epsDP\Delta_a})\right\}$ for some constant $C>0$. It follows that for any suboptimal arm $a$, the expected regret from pulling arm $a$ is therefore at most $\min\left\{\Delta_a T,~~2C\log(T)(\frac{1}{\Delta_{a}}+\frac{1}{\epsDP})\right\}$ (as $K\leq T$). 
	
	Denote by $\Delta^*$ the gap which equates the two possible regret bounds under which all arms are pulled $T/K$ times, namely $\Delta^*\frac T K = 2C\log(T)(\frac{1}{\Delta^*}+\frac{1}{\epsDP})$. While deriving $\Delta^*$ closed form is rather hairy, one can easily verify that $\Delta^* = \Theta(\max\{\sqrt {\nicefrac{K\log(T)}{T}},~~ \frac{K \log(T)}{\epsDP T}\})$.
	First, note that given $T$, in a setting where all suboptimal arms have a gap of precisely $\Delta^*$, then the cumulative expected regret bound is proportional to $O \left( \sqrt{T K \log(T)} + \nicefrac{K \log(T)}{\epsDP} \right)$. We show that regardless of how the different arm-gaps are set by an adversary, the expected regret of our algorithm is still proportional to the required bound. 
	
	Suppose an adversary sets a MAB instance, and again we rearrange arms such that arm $1$ is the leading arm and the gaps are increasing. We partition the set of suboptimal arms $2,3,.., K$ to two sets: $\{2,3,..,k'\}$ and $\{k'+1, k+2,..., K\}$ where $k'$ is the largest index of an arm with a gap $\leq \Delta^*$. Since this is a partition, one of the two sets contributes at least half of the expected regret. We thus break into cases.
	
	-- Each time we pull an arm from the former set, we incur an expected regret of at most $\Delta^*$. Since there are $T$ arm pulls overall, a crude bound on the expected regret obtained from pulling arms $\{2,..,k'\}$ is $\Delta^* T$. Therefore, if it is the case that the regret from pulling arms $\{2,3,..,k'\}$ is at least half of the expected regret, then the entire expected regret is at most $2\Delta^* T$.
	
	-- Based on the above discussion, the upper-bound on the expected regret due to pulling the arms in the set $\{k'+1, k'+2,..., K\}$ is at most 
	\begin{align*}
	2C\log(T)\hspace{-1mm}\sum_{a =k'+1}^{K} \hspace{-1mm} \left(\tfrac{1} {\Delta_a} + \tfrac 1 \epsDP\right) 
	&\leq 2C\log(T)\hspace{-1mm}\sum_{a =k'+1}^{K}\hspace{-1mm} \left(\tfrac 1 {\Delta^*} +\tfrac 1 \epsDP \right) 
	\cr&= (K-k')\Delta^* \frac T K \leq \Delta^* T
	\end{align*}
	Therefore, if it is the case that the regret from pulling arms $\{k'+1, k'+2,...,K\}$ is greater than half of the expected regret, then the entire expected regret is at most $2\Delta^* T$.
	
	In either case, it is simple to see that the expected regret is upper bounded by $O(\Delta^*T) = O \left( \sqrt{T K \log(T)} + \nicefrac{K \log(T)}{\epsDP} \right)$.
\end{proof}
}

\begin{theorem} \label{thm:DP_SE_instanceIndependent_regret}
(Instance Independent Bound) \ThmDPSEinstanceIndependentRegretBound
\end{theorem}

\pfThemInstanceIndependentRegret

Again, we comment on the optimality of the bound in Theorem~\ref{thm:DP_SE_instanceIndependent_regret}. The non-private minimax bound~\citep{auer2002nonstochastic} is known to be $\Omega(\sqrt{TK})$ and combining it with the private bound of $\Omega(K\log(T)/\epsDP)$ we see that the above minimax bound is just $\sqrt{\log(T)}$-factor away from being optimal.

\cut{

\begin{lemma}
	Fix any instance of the $K$-MAB problem, and denote $a^*$ as its optimal arm (of highest mean), and the gaps between the mean of arm $a^*$ and any suboptimal arm $a\neq a^*$ as $\Delta_a$. Fix any horizon $T$. Then w.p. $\geq 1 - \deltaStop$ it holds that Algorithm~\ref{alg:DP_SE1} pulls each suboptimal arm $a\neq a^*$ for a number of timesteps upper bounded by
	\[ \min\{ T, ~~O\left(  \left( \log(\nicefrac K  \deltaStop) + \log \log (\nicefrac 1 \Delta_{a}) \right) \left( \frac{1}{\Delta^2_{a}} + \frac{1}{\epsDP \Delta_{a}} \right) \right)   \}  \]
\end{lemma}

\begin{proof}
    
\end{proof}

\begin{theorem}
Algorithm~\ref{alg:DP_SE1} preserves $\epsDP$- Differential Privacy.
\end{theorem}

\begin{proof}
    
\end{proof}
}
\section{Empirical Evaluation}
\label{sec:experiments}

\paragraph{Goal.} In this section, we empirically compare the DP-UCB algorithm \citep{mishra2015nearly} and our DP-SE algorithm (Algorithm~\ref{alg:DP_SE}). Our goal is two fold. First, we would like to assert that indeed there \emph{exists} some setting of parameters under which our DP-SE algorithm outperforms (achieves smaller expected regret than) the DP-UCB baseline. Afterall, the improvement we introduce is over $\poly\log(T)$ factors and does incur an increase in the constants repressed by the big-$O$ notation. Hence, our primary goal is to verify that indeed the asymptotic improvement in performance is reflected in actual empirical performance. Second, assuming the former is answered on the affirmative, we would like to see under which region of parameters our DP-SE algorithm outperforms the DP-UCB baseline. 

\paragraph{Setting and Experiments.} By Default, we set $T = 5\times 10^7$, $\epsDP = 0.25$ and $K=5$. 
We assume $T$ is a-priori known to both algorithms and set $\beta = 1/T$.
We consider four instances, denoted by $C_1$, $C_2$, $C_3$, $C_4$, where in all the settings the reward of any arm is drawn from a Bernoulli distribution. In $C_1$ all suboptimal gaps are the same, and the arms' mean-rewards are $\{ 0.75, \underbrace{0.7,... 0.7}_{K-1} \}$; whereas in $C_2$ the suboptimal arms' gaps decrease linearly, where the largest mean is always $0.75$ and the smallest mean is always $0.25$ (so for $K=5$ the means are $\{ 0.75, 0.625, 0.5, 0.375, 0.25\}$) \footnote{Constraining the means within $[0.25,0.75]$ ensures the variance of the arms are similar (upto a constant of $\nicefrac 4 3$)}. We considered $C_3$ to compare the performances for the case that a larger fraction of arms have large suboptimality gaps, hence we chose to use a quadratic \emph{convex} function of the form: $\mu_i = a (i-K)^2 + c$ such that $\mu_1 = 0.75$, $\mu_K = 0.25$ and $a > 0$ (so for $K=5$ the means are $\{ 0.75, 0.53125, 0.375, 0.28125, 0.25\}$). $C_4$ was chosen to illustrate the performance for the case that a larger faction of arms have small suboptimality gaps, hence it suffices to use a quadratic \emph{concave} function: $\mu_i = a (i-1)^2 + c$ such that $\mu_1 = 0.75$, $\mu_K = 0.25$ and $a < 0$ (so for $K=5$ the means are $\{ 0.75, 0.71875, 0.625, 0.46875, 0.25\}$). 
Using $a^*$ to denote the optimal arm, we measure the algorithms' performances in terms of their pseudo regret, so upon pulling a suboptimal arm $a\neq a^*$ each algorithm incurs a cost $\Delta_a = \mu_{a^*}-\mu_a$. For each setting, $30$ runs of the algorithms were carried out and their average pseudo regrets are plotted.

Under all four settings we conduct two sets of experiments. First, we vary $\epsDP \in \{0.1, 0.25, 0.5, 1\}$, and the results in settings $C_1$, $C_2$, $C_3$, $C_4$, are given in the Figures ~\ref{fig:K=5|varyEps|setting1}, ~\ref{fig:K=5|varyEps|setting2}, ~\ref{fig:K=5|varyEps|setting3}, ~\ref{fig:K=5|varyEps|setting4} respectively. Then we vary $K \in \{3,5,10,20\}$, and the results under $\epsDP=0.5, 1$ in settings $C_1$, $C_2$, $C_3$, $C_4$ are given in Figures ~\ref{fig:varyK|eps0.5&1|setting1}, ~\ref{fig:varyK|eps0.5&1|setting2}, ~\ref{fig:varyK|eps0.5&1|setting3}, ~\ref{fig:varyK|eps0.5&1|setting4} respectively, while the results under $\epsDP=0.1, 0.25$ in settings $C_1$, $C_2$, $C_3$, $C_4$ are given in Figures ~\ref{fig:varyK|eps0.1&0.25|setting1}, ~\ref{fig:varyK|eps0.1&0.25|setting2}, ~\ref{fig:varyK|eps0.1&0.25|setting3}, ~\ref{fig:varyK|eps0.1&0.25|setting4} respectively

\paragraph{Results and discussion.} 
The results conclusively show that DP-SE outperforms DP-UCB. Subject to the caveat that our experiments are proof-of-concept only and we did not conduct a thorough investigation of the entire hyper-parameter space, we \emph{could not find even a single setting where DP-UCB is even comparable to our DP-SE}. I.e. in \emph{all} settings we tested, we outperform DP-UCB by at least 5 times. We also comment as to the difference in the shape of the two pseudo-regret curves --- while the DP-UCB curve is smooth (attesting to the fact it pulls suboptimal arms even for fairly large values of $T$), the DP-SE is piece-wise linear (exhibiting the fact that at some point it eliminates all suboptimal arms).

\cut{
\paragraph{Goal.} In this section, we empirically compare the DP-UCB algorithm \citep{mishra2015nearly} and DP-SE algorithm (Algorithm~\ref{alg:DP_SE}). Our goal is two fold. First, we would like to assert that indeed there \emph{exists} some setting of parameters under which our DP-SE algorithm outperforms (achieves smaller expected regret than) the DP-UCB baseline. Afterall, the improvement we introduce is over $\poly\log(T)$ factors and does incur an increase in the constants repressed by the big-$O$ notation. Hence, our primary goal is to verify that indeed the asymptotic improvement in performance is reflected in actual empirical performance. Second, assuming the former is answered on the affirmative, we would like to empirically assess the region of parameters under which our DP-SE algorithm outperforms the DP-UCB baseline, or provide ``guidelines'' as to when should one prefer the DP-SE algorithm to the DP-UCB algorithm. 

In addition, we also experiment with a variant of our algorithm. Recall, Algorithm~\ref{alg:DP_SE} sets the SVT mechanism to halt when the largest empirical reward is greater than $10$ times the Hoeffding bound (see line 18 of Algorithm~\ref{alg:DP_SE}), in order to have the worst-case guarantee that all arms of substantial gap from the leading arm are removed. We thus consider a modification of Algorithm~\ref{alg:DP_SE} where the halting condition is
\begin{align*}
    \max_{i,j \in S} \left( \mubar_i - \mubar_j \right) > 2h_r + \frac{A_r + B + c_{r}}{r}
\end{align*}
for $c_r = \frac{6\log \left(\nicefrac {4K} \deltaStop \right)}\epsDP + \frac{6 \log \left(\nicefrac{8Kl^2} \deltaStop \right)}{\epsDP} + \frac{4\log \left(\nicefrac{4K|S|} \deltaStop \right)}{\epsDP}$ and $h_r$ denoting the Hoeffding bound after $r$ arm-pulls. Such a condition assures  w.h.p. that the worst viable arm is eliminated in each epoch $e$, yet doesn't guarantee \emph{all} arms of noticeable gaps are removed. Moreover, such a halting condition ``evens the playing field'' as the bounds in the DP-UCB algorithm also depend solely on $2h_r$. We refer to this as the ``modified DP-SE'' algorithm in our experiments.

\paragraph{Setting and Experiments.} By Default, we set $T = 5\times 10^7$, $\epsDP = 0.25$ and $K=5$. Recall, the DP-UCB's suboptimality is reflected in the $\poly\log(T)/\epsDP$  factor in its regret bound. Thus, it stands to reason we focus our empirical study on a setting where $T$ is fairly large and $\epsDP$ is fairly small. We assume $T$ is a-priori known to the algorithms and set $\beta = 1/T$.
We consider four instances, denoted by $C_1$, $C_2$, $C_3$, $C_4$, where in all the settings the reward of any arm is drawn from a Bernoulli distribution. In $C_1$ all suboptimal gaps are the same, and the arms' mean-rewards are $\{ 0.75, \underbrace{0.7,... 0.7}_{K-1} \}$; whereas in $C_2$ the suboptimal arms' gaps decrease linearly, where the largest mean is always $0.75$ and the smallest mean is always $0.25$ (so for $K=5$ the means are $\{ 0.75, 0.625, 0.5, 0.375, 0.25\}$) \footnote{Constraining the means within $[0.25,0.75]$ ensures the variance of the arms are similar (upto a constant of $\nicefrac 4 3$)}. We considered $C_3$ to compare the performances for the case that a larger fraction of arms have large suboptimality gaps, hence we chose to use a quadratic \emph{convex} function of the form: $\mu_i = a (i-K)^2 + c$ such that $\mu_1 = 0.75$, $\mu_K = 0.25$ and $a > 0$ (so for $K=5$ the means are $\{ 0.75, 0.53125, 0.375, 0.28125, 0.25\}$). $C_4$ was chosen to illustrate the performance for the case that a larger faction of arms have small suboptimality gaps, hence it suffices to use a quadratic \emph{concave} function: $\mu_i = a (i-1)^2 + c$ such that $\mu_1 = 0.75$, $\mu_K = 0.25$ and $a < 0$ (so for $K=5$ the means are $\{ 0.75, 0.71875, 0.625, 0.46875, 0.25\}$). Let $a^*$ be the optimal arm and we measure the performance in terms of their pseudo regret, so upon pulling a suboptimal arm $a\neq a^*$ each algorithm incurs a cost $\Delta_a = \mu_{a^*}-\mu_a$. For each setting, $30$ runs of the algorithms were carried out and their average pseudo regrets are plotted.

Under all four settings we conduct two sets of experiments. First, we vary $\epsDP \in \{0.1, 0.25, 0.5, 1\}$, and the results in setting $C_1$, $C_2$, $C_3$, $C_4$, are given in the Figures ~\ref{fig:K=5|varyEps|setting1}, ~\ref{fig:K=5|varyEps|setting2}, ~\ref{fig:K=5|varyEps|setting3}, ~\ref{fig:K=5|varyEps|setting4} respectively. Then we vary $K \in \{3,5,10,20\}$, and the results under $\epsDP=0.25, 1$ in setting $C_1$, $C_2$, $C_3$, $C_4$ are given in Figures ~\ref{fig:varyK|eps0.5&1|setting1}, ~\ref{fig:varyK|eps0.5&1|setting2}, ~\ref{fig:varyK|eps0.5&1|setting3}, ~\ref{fig:varyK|eps0.5&1|setting4} respectively.

\paragraph{Results and discussion.} 
A few observations are immediately clear. First, in setting $C_1$, where all gaps are the same and quite small, the DP-SE algorithm outperforms the DP-UCB baseline when either $\epsDP$ is small ($ \leq 0.25$ in our experiments) or for large values of $K$ (see Figure \ref{fig:K=5|varyEps|setting1} and \ref{fig:varyK|eps0.5&1|setting1}). In other words, the conditions for which DP-UCB is better than DP-SE are when $\epsDP$ is fairly large, the number of arms is moderate, \emph{and} all arms have identical and fairly small gaps. We would like to point out that in some of the plots, there are overlappings between the curves (we put a (*) in their captions). The setting $C_4$ further adds mounting evidence that naive DP-SE can be outperformed by DP-UCB for small suboptimality gaps and large values of $\epsDP$ (see Figure \ref{fig:K=5|varyEps|setting4} and \ref{fig:varyK|eps0.5&1|setting4}), only that slightly larger $\epsDP$ and smaller $K$ can be forgiving unlike in $C_1$ due to the presence of some large suboptimality gaps. In all other cases~--- and especially note the consistency throughout all experiments in setting $C_2$ and $C_3$~--- \emph{our naive DP-SE outperforms the DP-UCB baseline}. Moreover, in \emph{all} settings the modified DP-SE is the best algorithm of all three, even though we cannot prove that the modified DP-SE algorithm eliminates \emph{all} suboptimal arms of comparable gap from the leading arm. We believe our experiments unequivocally show that the asymptotic improvement in the analysis of the DP-SE over DP-UCB is evident in actual, empiric performance.

We postulate that the reason for the improved performance across setting $C_2$ is the fact that under DP-UCB arms of large gaps remain effectively viable (i.e. are pulled relatively frequently) for a longer period of time than under DP-SE, which  eliminates noticeably suboptimal arms early on. In other words, the large gaps play to the advantage of DP-SE. Moreover, comparing the curves for DP-SE and the modified DP-SE in Figure~\ref{fig:K=5|varyEps|setting2}, we see that standard DP-SE eliminates all suboptimal arms not long after the modified DP-SE eliminates all suboptimal arms. This implies that in setting $C_2$, the number of arm pulls required to create a noticeable gap between arms is mostly due to the privacy-dependent gap of $O(\frac{\log(T)}{\epsDP})$ rather than the Hoeffding bound. Since in $C_3$ there are more large suboptimality gaps than smaller suboptimality gaps, similar reasoning should apply under $C_3$ too.

In contrast, in setting $C_1$ (see Figure~\ref{fig:K=5|varyEps|setting1}) we see drastic performance difference between the standard DP-SE and the modified DP-SE. Here the gap between arms is small enough s.t.~the key component in the decision to halt is the Hoeffding bound (unless $\epsDP$ is quite small). Indeed, the modified DP-SE algorithm, which sets the dependency on the Hoeffding bound to be 5-times smaller than in the standard DP-SE algorithm, also happens to eliminate all suboptimal arms in (roughly) $\nicefrac 1 5$ of the time it takes the standard DP-SE algorithm to eliminate all arms. We would like to also comment that one thing that plays to the potential advantage of the modified DP-SE algorithm is the use of exponentially growing intervals. It is likely that the round $r$ which is \emph{also a power of $2$} under which the modified DP-SE halts is large enough to allow some slackness, that helps the algorithm to overcome the random noise and assert that all arms of noticeable suboptimality gap at round $r$ are indeed eliminated.

We conclude by repeating the high-level message. Unless $\epsDP$ is large or there are many arms with small suboptimality gaps, the added cost of privacy places a noticeable role in the accumulated pseudo-regret, and so our DP-SE algorithm outperforms the DP-UCB baseline.
}

\newgeometry{margin=0.8in}
\newpage
\twocolumn


\begin{figure}[!]
    \begin{subfigure}[h]{0.45\textwidth}
    	\includegraphics[width=0.97\textwidth]{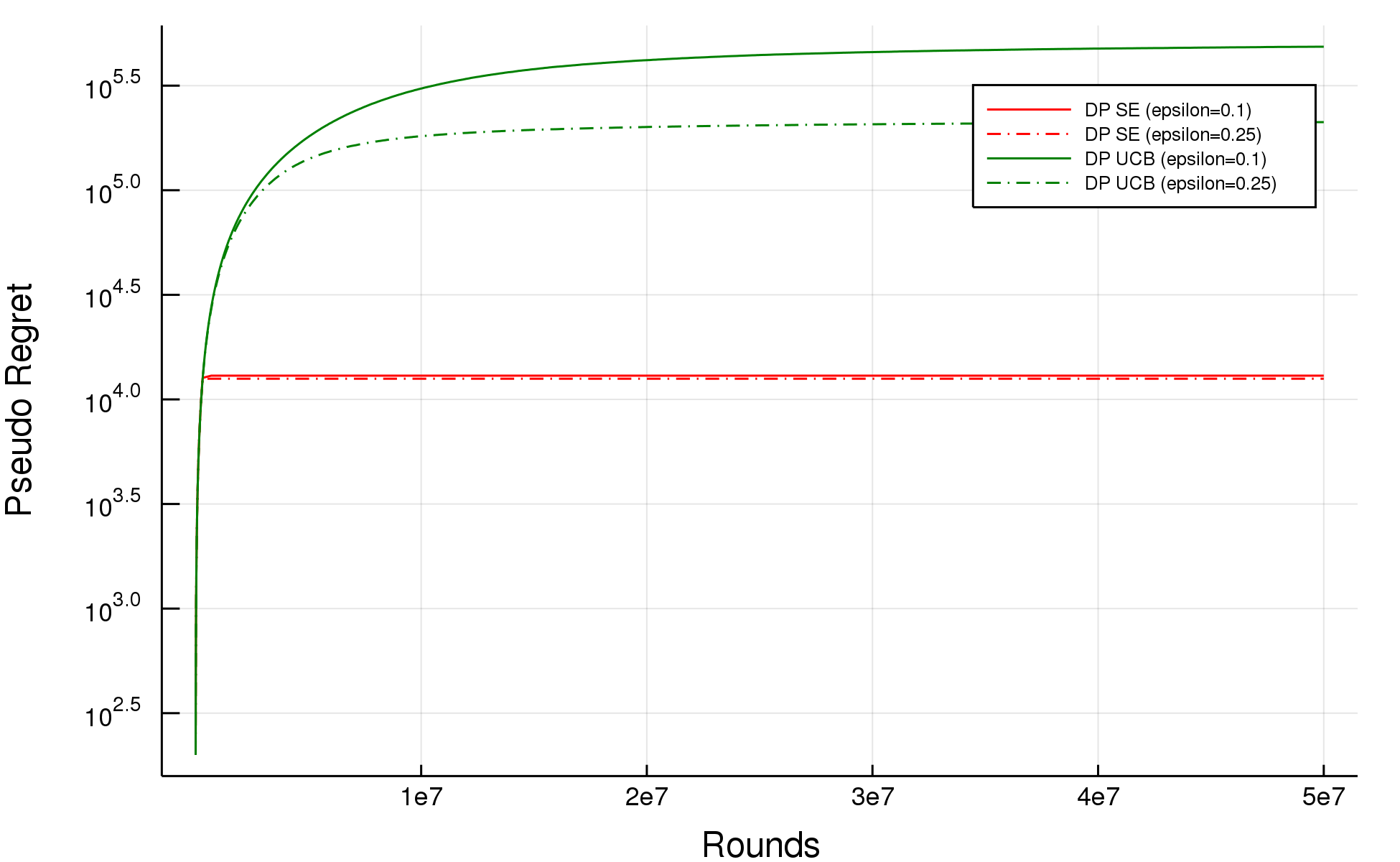}
    	\caption{$\epsDP=0.1$ and $0.25$}
    \end{subfigure} 
    
    \begin{subfigure}[h]{0.45\textwidth}
    	\includegraphics[width=0.97\textwidth]{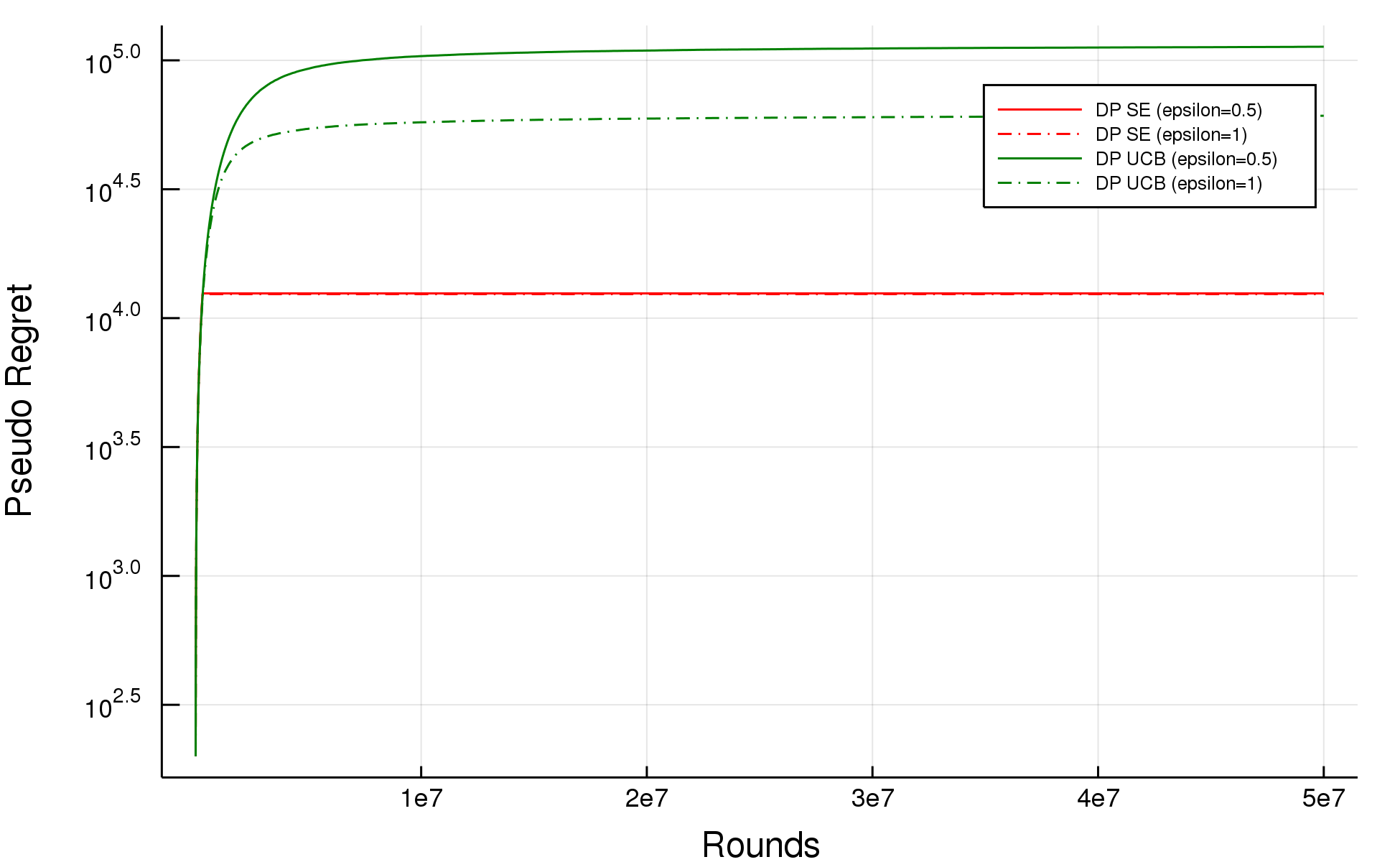}
    	\caption{$\epsDP=0.5$ and $1$}
    \end{subfigure}
    
    \caption{\label{fig:K=5|varyEps|setting1} Under $C_1$ with $K=5, T=5 \times 10^7$}
\end{figure}

\begin{figure}[!]
    \begin{subfigure}[h]{0.45\textwidth}
    	\includegraphics[width=0.97\textwidth]{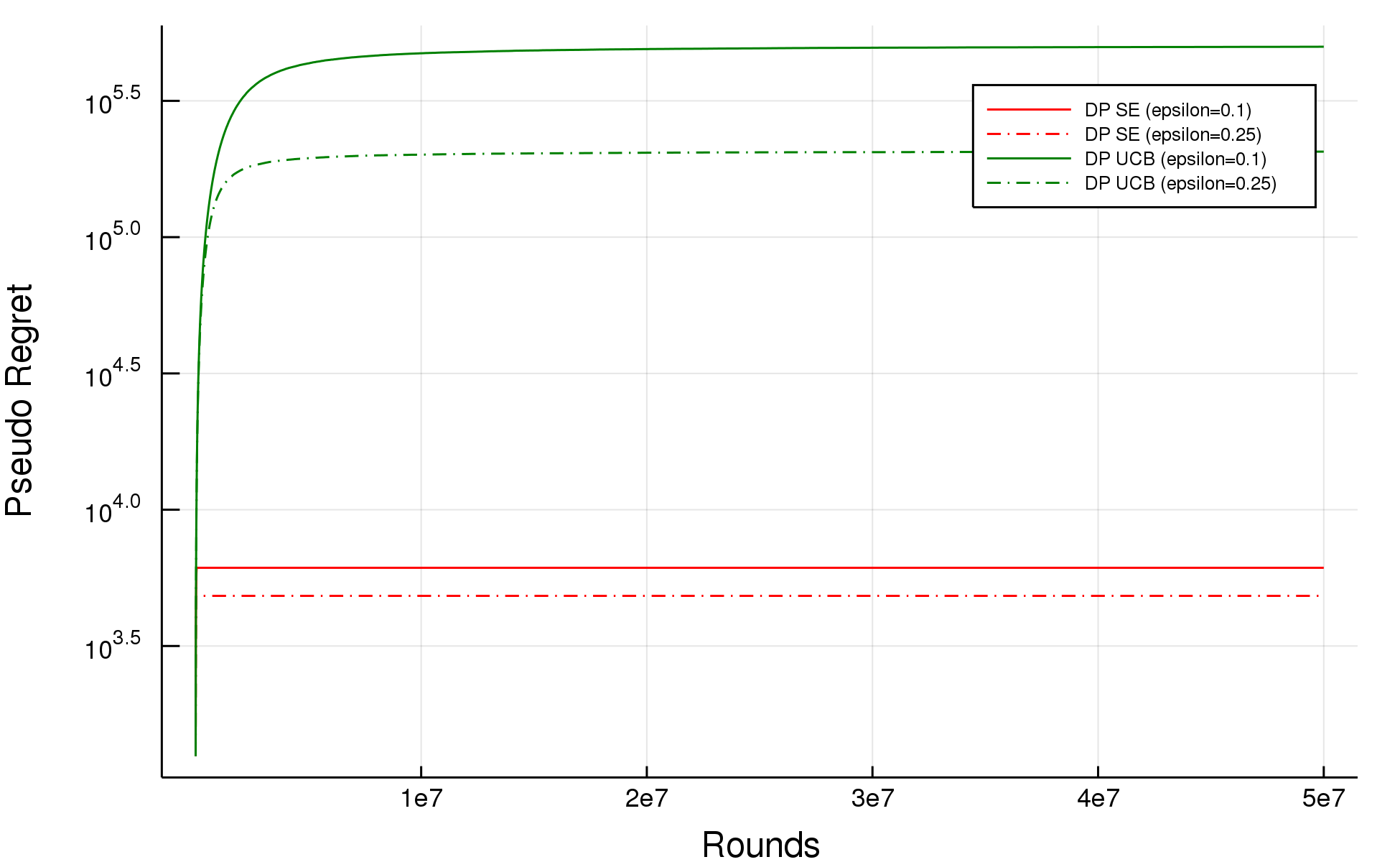}
    	\caption{$\epsDP=0.1$ and $0.25$}
    \end{subfigure}
    
    \begin{subfigure}[h]{0.45\textwidth}
    	\includegraphics[width=0.97\textwidth]{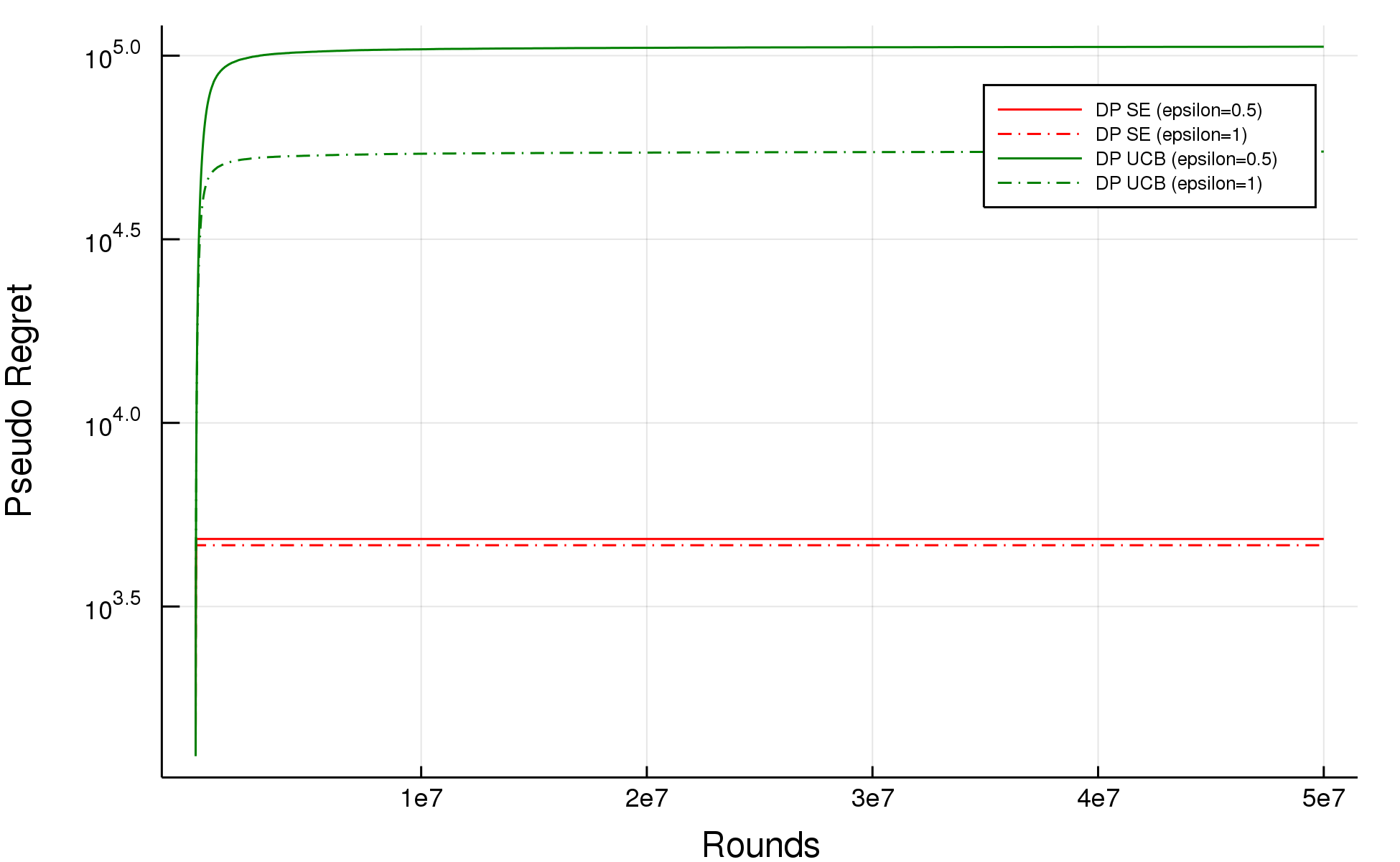}
    	\caption{$\epsDP=0.5$ and $1$}
    \end{subfigure}
    
    \caption{\label{fig:K=5|varyEps|setting2} Under $C_2$ with $K=5, T=5 \times 10^7$}
\end{figure}

\begin{figure}[!]
    \begin{subfigure}[h]{0.45\textwidth}
    	\includegraphics[width=0.97\textwidth]{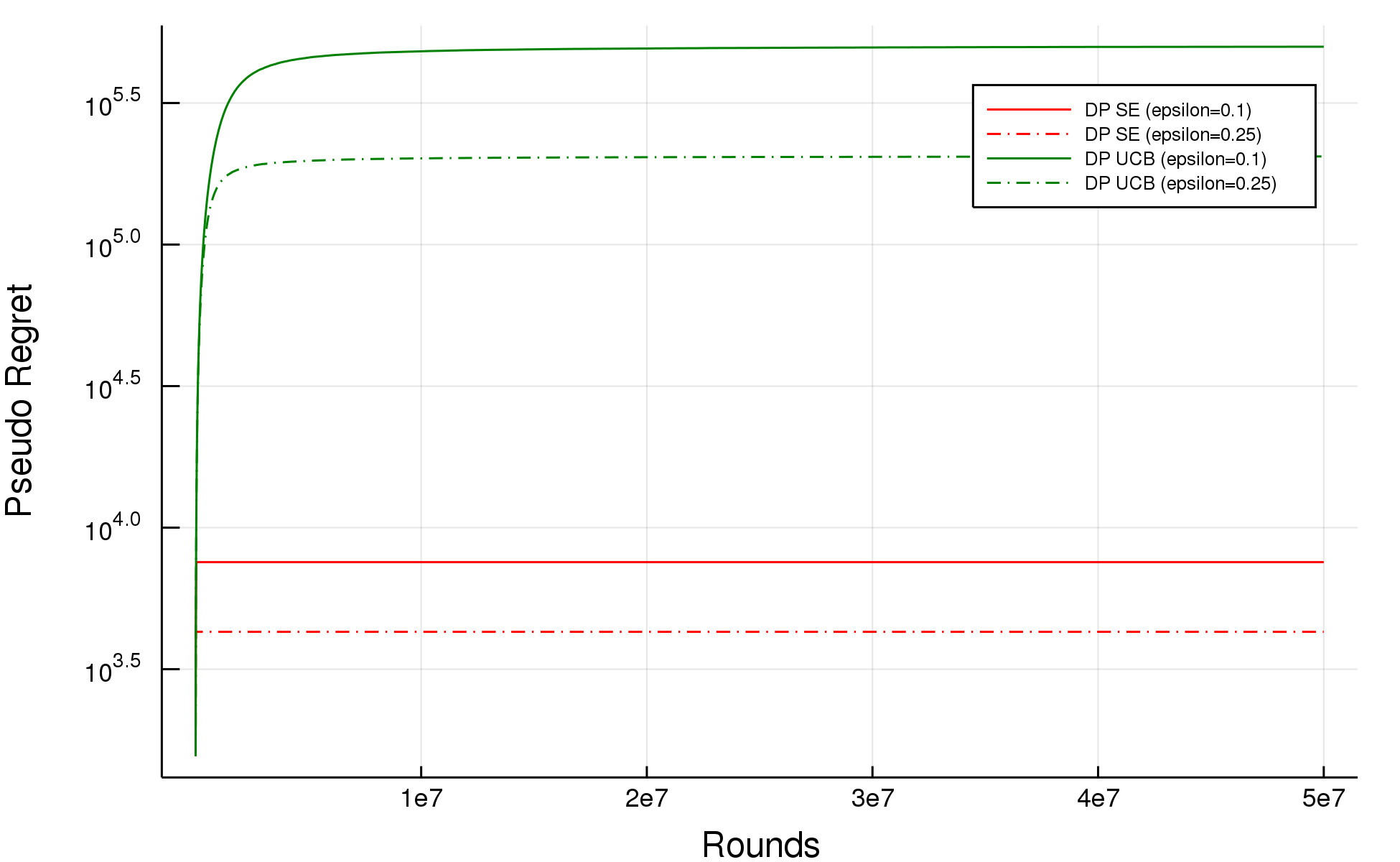}
    	\caption{$\epsDP=0.1$ and $0.25$}
    \end{subfigure} 
    
    \begin{subfigure}[h]{0.45\textwidth}
    	\includegraphics[width=0.97\textwidth]{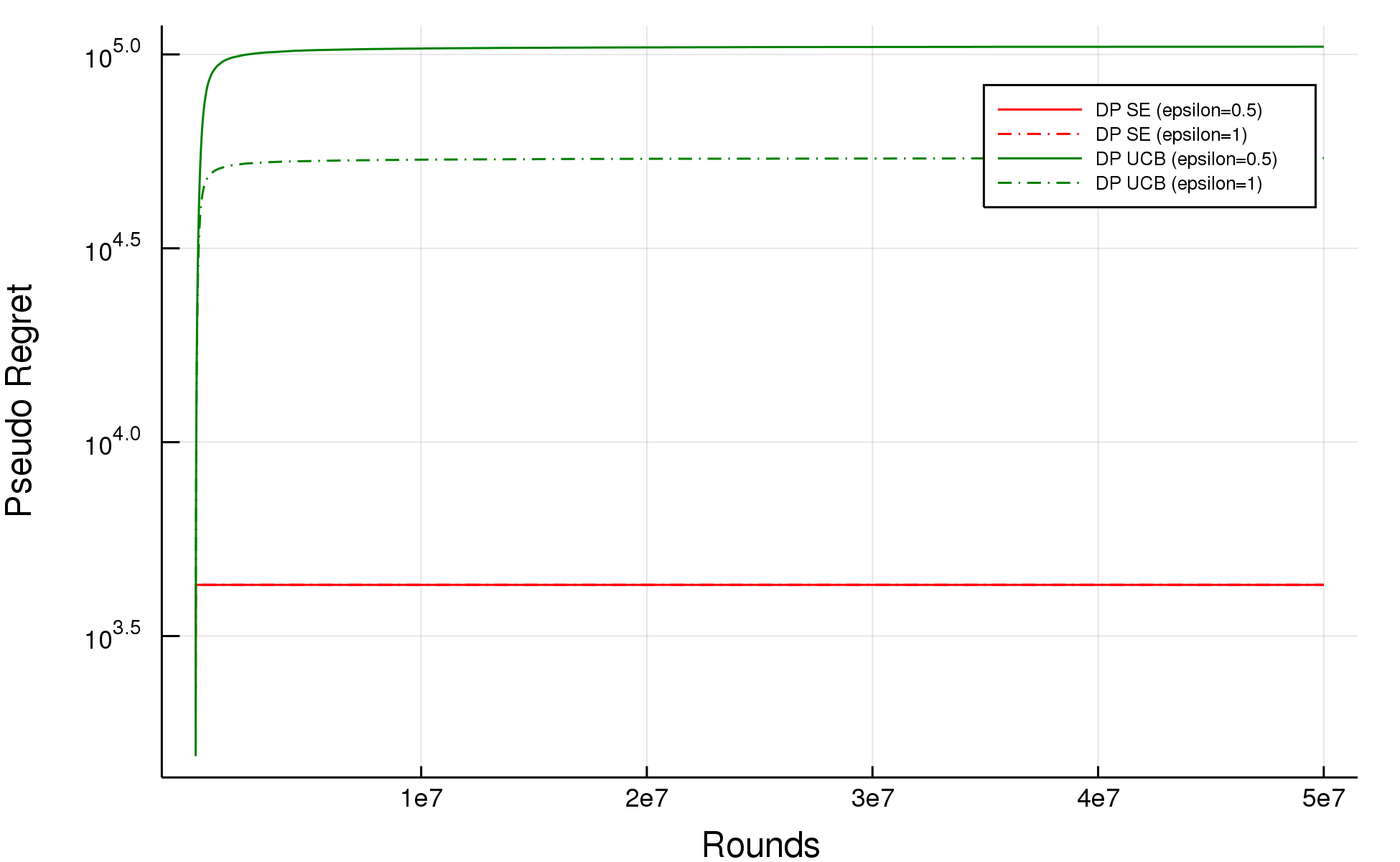}
    	\caption{$\epsDP=0.5$ and $1$}
    \end{subfigure}
    
    \caption{\label{fig:K=5|varyEps|setting3} Under $C_3$ with $K=5, T=5 \times 10^7$}
\end{figure}
    
\begin{figure}[!]
    \begin{subfigure}[h]{0.45\textwidth}
    	\includegraphics[width=0.97\textwidth]{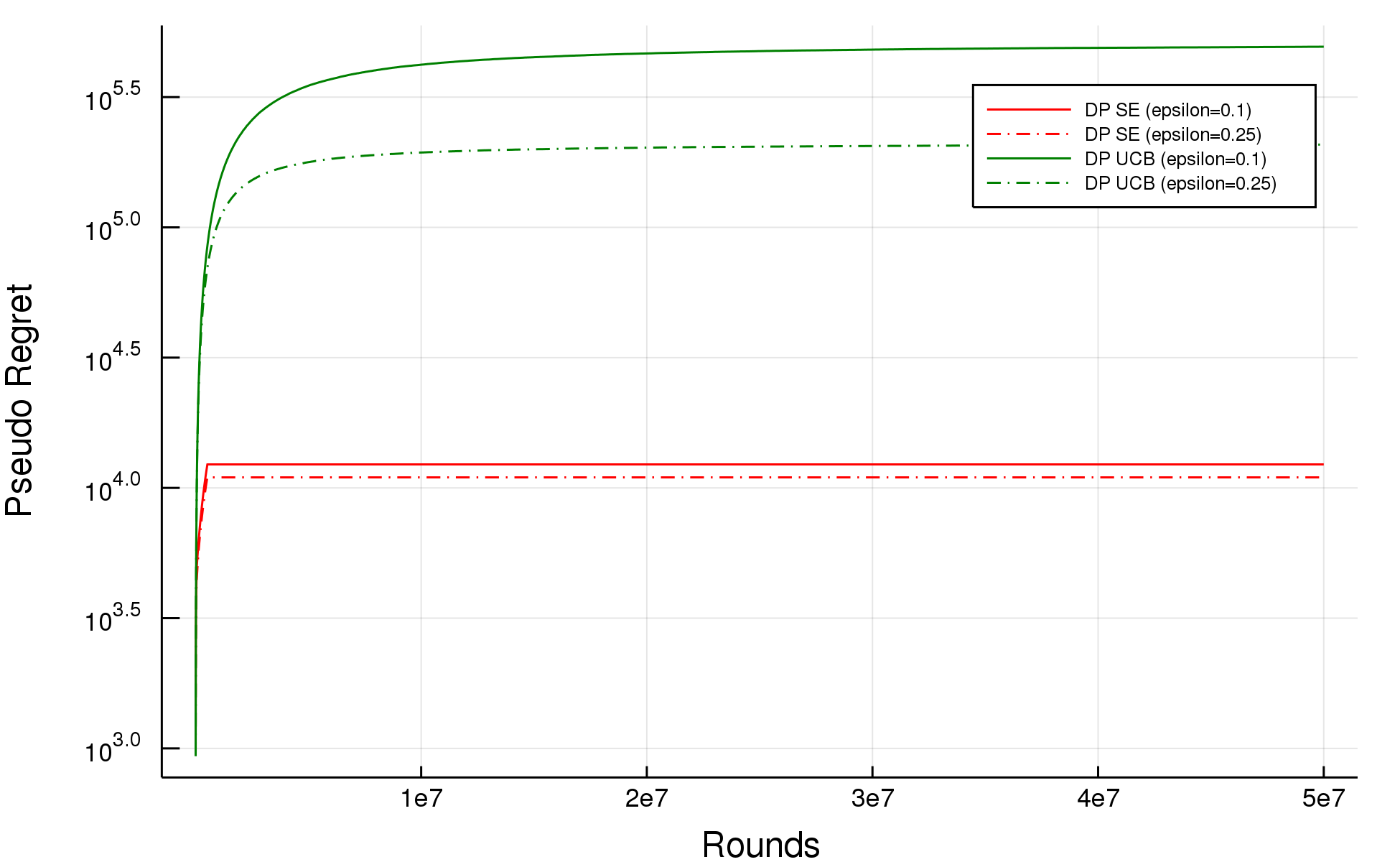}
    	\caption{$\epsDP=0.1$ and $0.25$}
    \end{subfigure}
    
    \begin{subfigure}[h]{0.45\textwidth}
    	\includegraphics[width=0.97\textwidth]{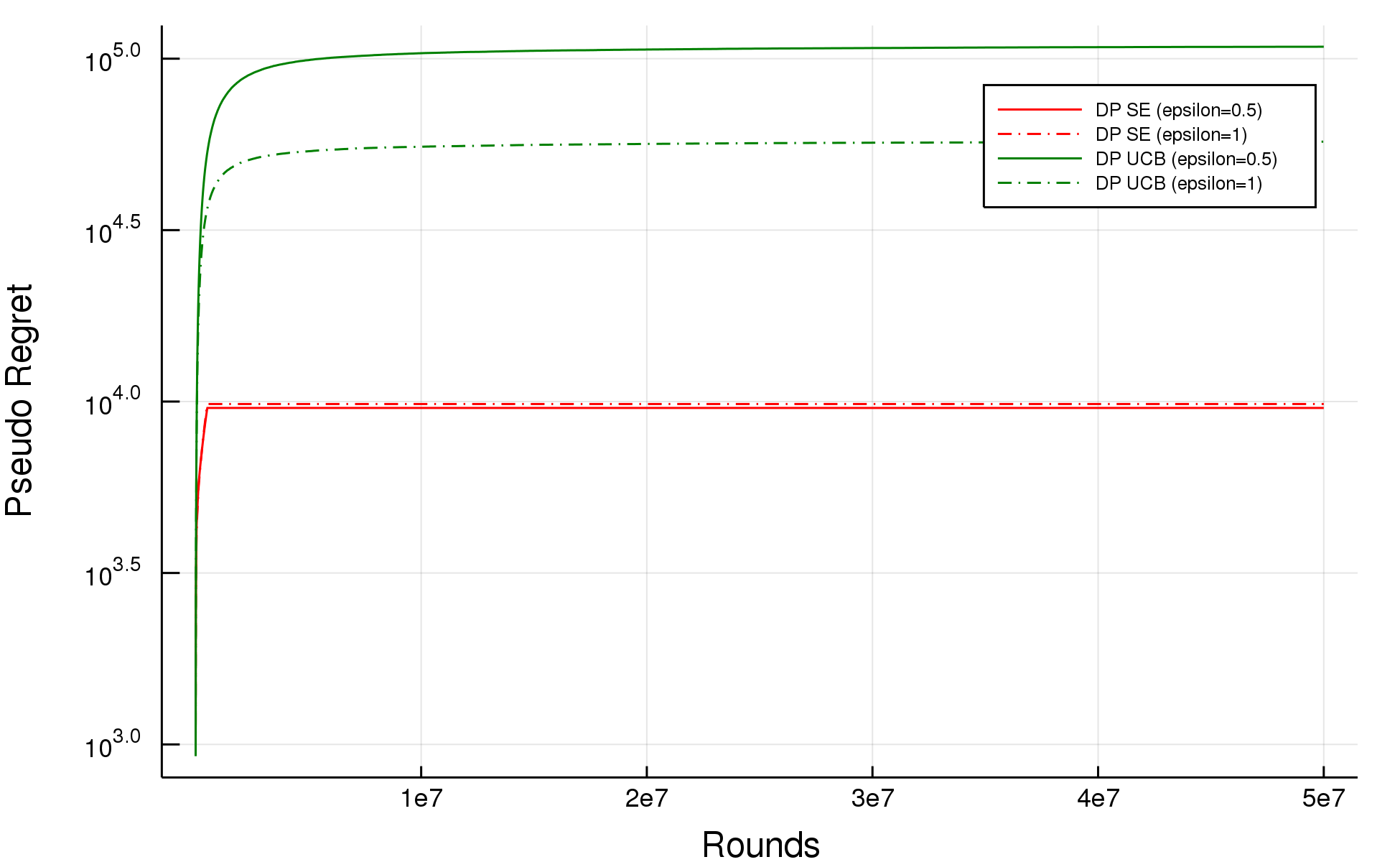}
    	\caption{$\epsDP=0.5$ and $1$}
    \end{subfigure}
    
    \caption{\label{fig:K=5|varyEps|setting4} Under $C_4$ with $K=5, T=5 \times 10^7$}
\end{figure}


\begin{figure}[h!]
    \begin{center}
    \begin{subfigure}[h]{0.5\textwidth}
    	\includegraphics[width=0.95\textwidth]{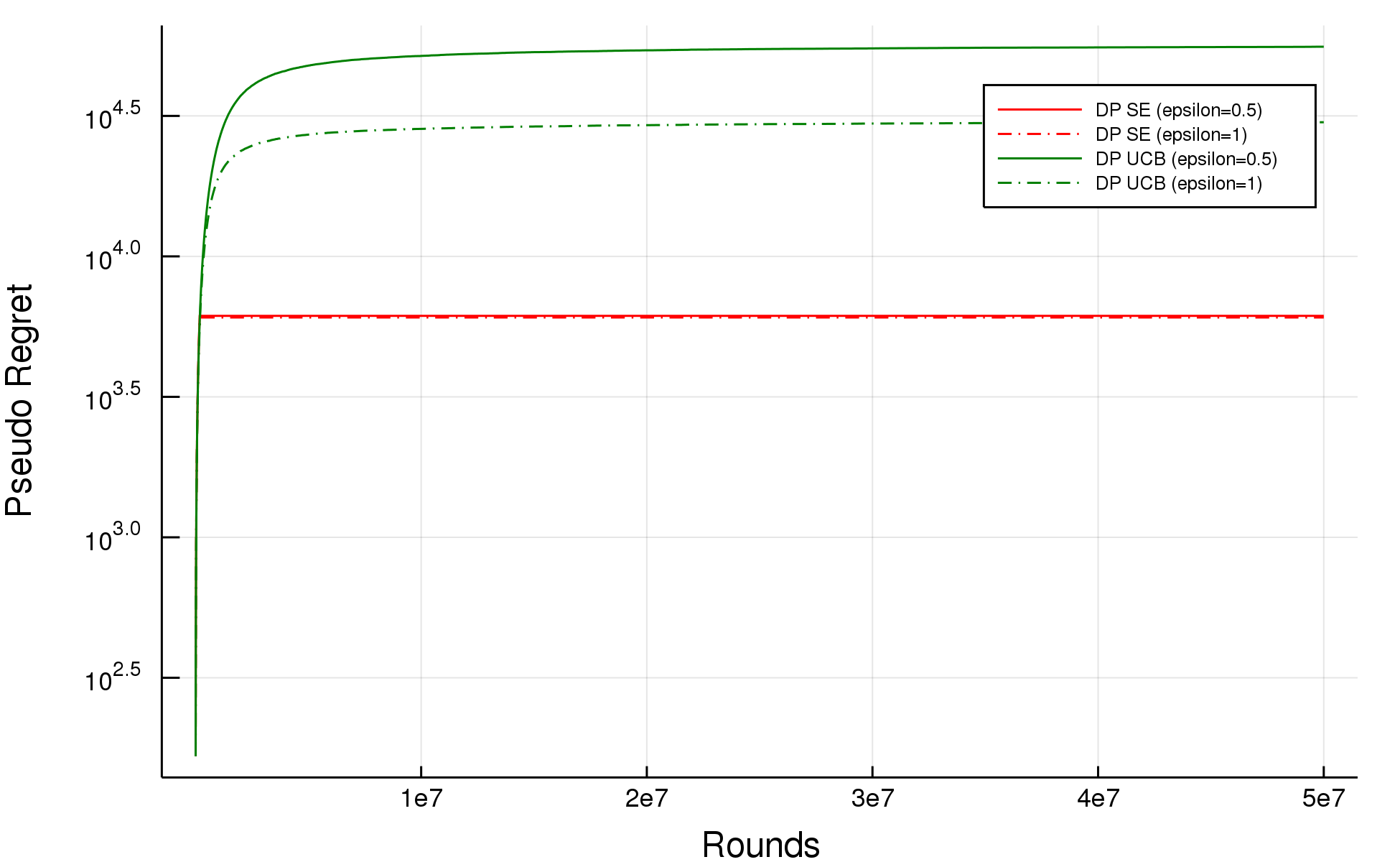}
    	\caption{$K=3$}
    \end{subfigure} 
    
    \begin{subfigure}[h]{0.5\textwidth}
    	\includegraphics[width=0.95\textwidth]{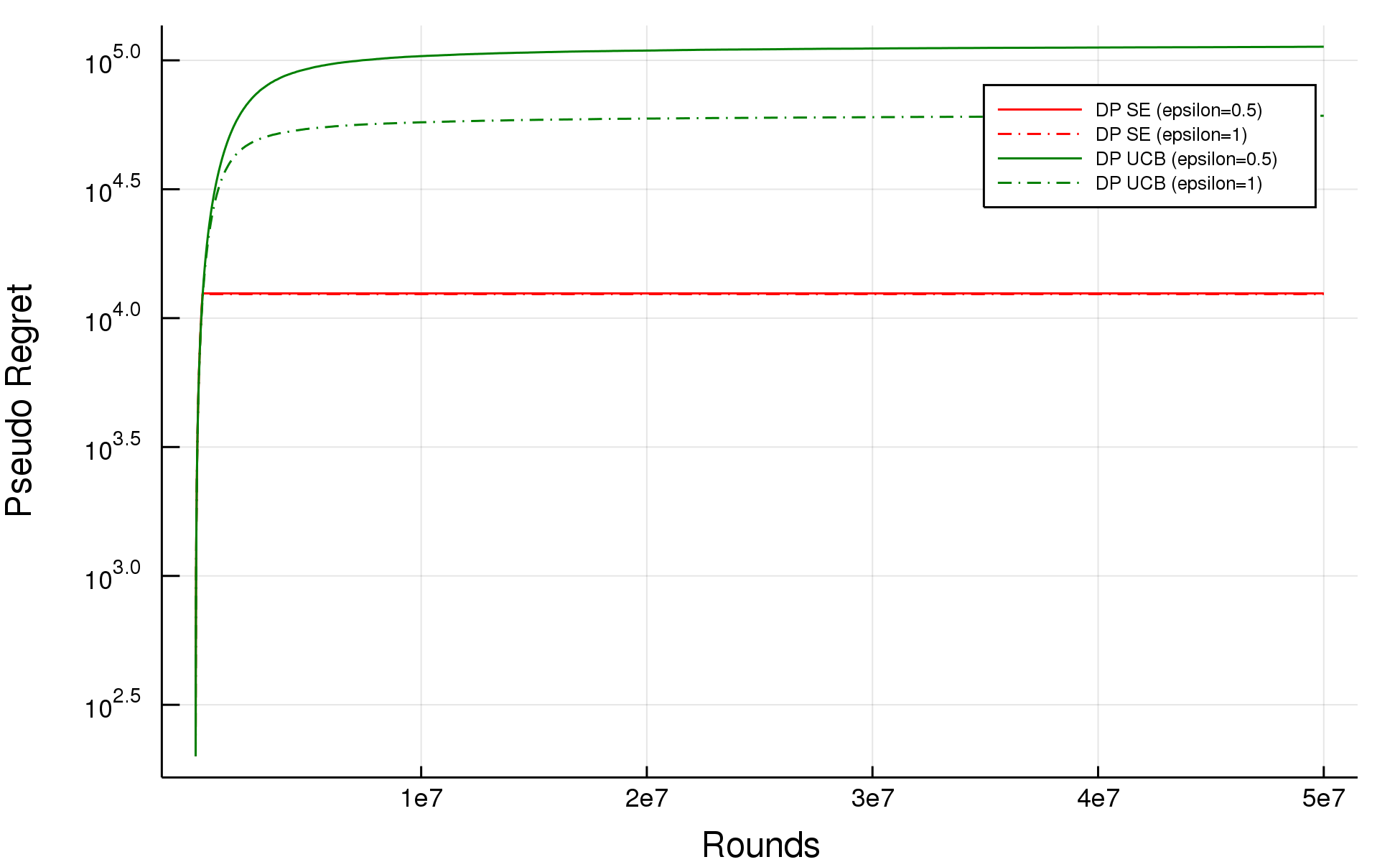}
    	\caption{$K=5$}
    \end{subfigure}
    
    \begin{subfigure}[h]{0.5\textwidth}
    	\includegraphics[width=0.95\textwidth]{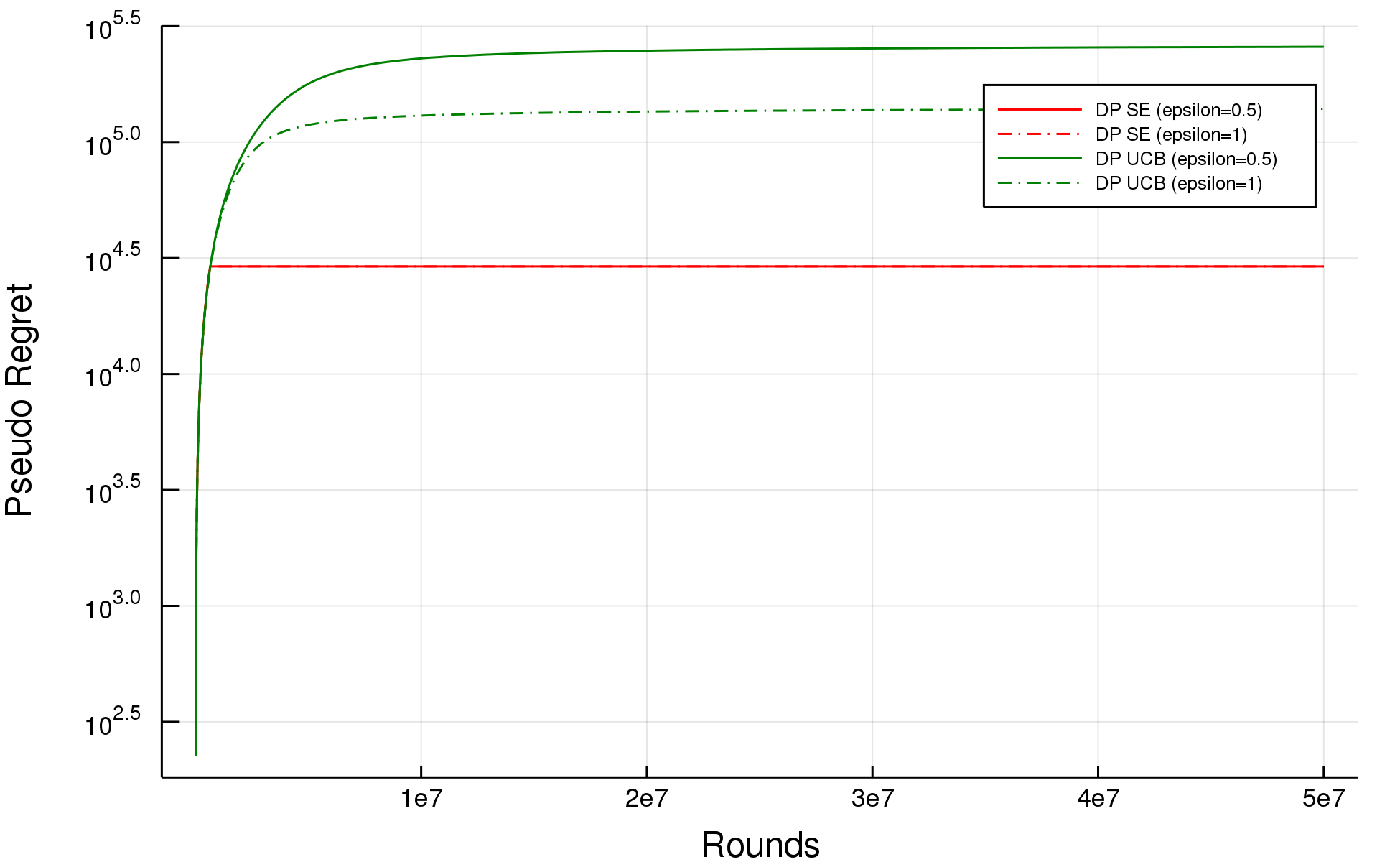}
    	\caption{$K=10$}
    \end{subfigure}
    
    \begin{subfigure}[h]{0.5\textwidth}
    	\includegraphics[width=0.95\textwidth]{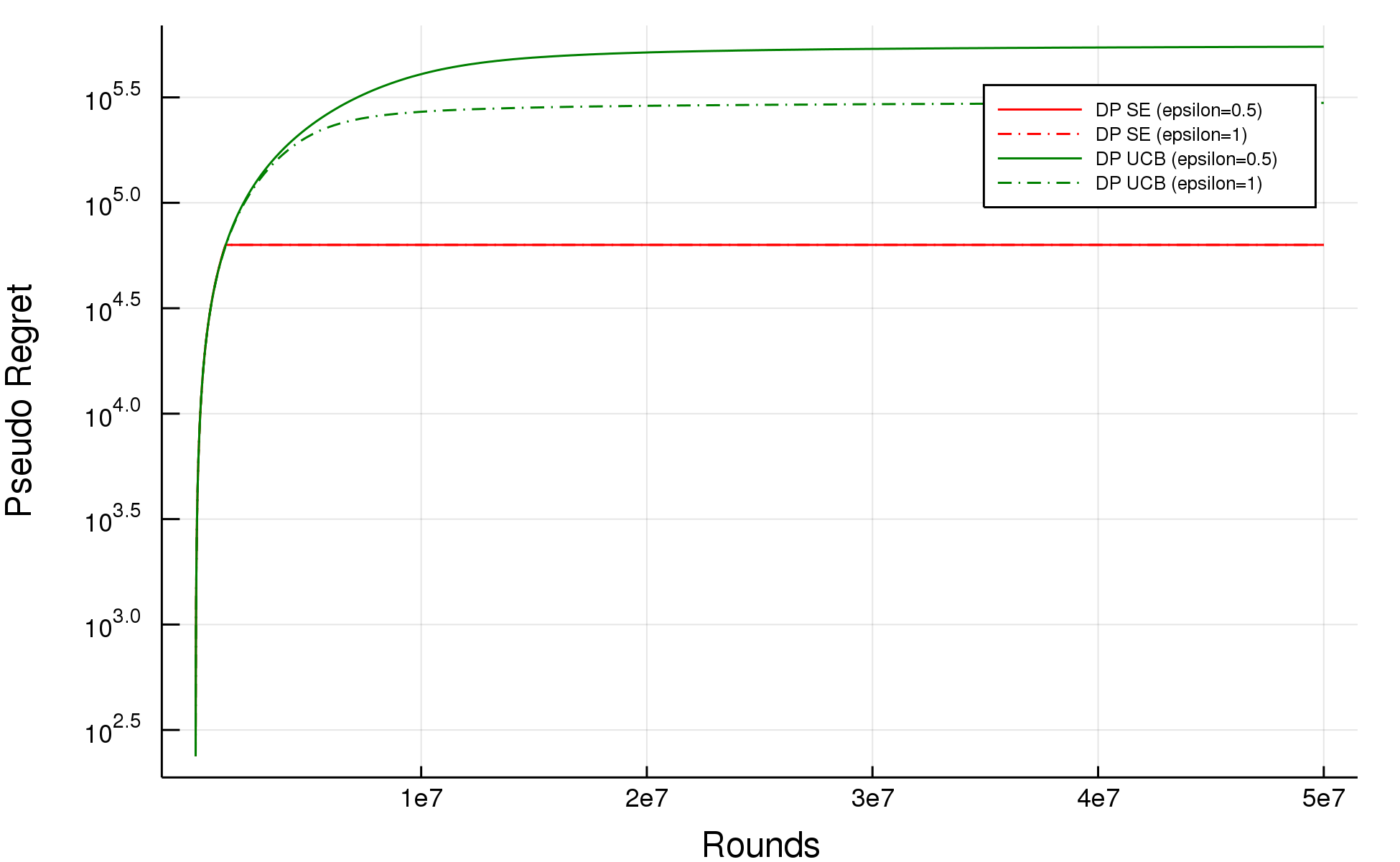}
    	\caption{$K=20$}
    \end{subfigure}
    
    \caption{\label{fig:varyK|eps0.5&1|setting1} Under $C_1$ with $\epsDP \in \{0.5,1\}, T=5 \times 10^7$}
    \end{center}
\end{figure}

\begin{figure}[h!]
    \begin{center}
    \begin{subfigure}[h]{0.5\textwidth}
    	\includegraphics[width=0.95\textwidth]{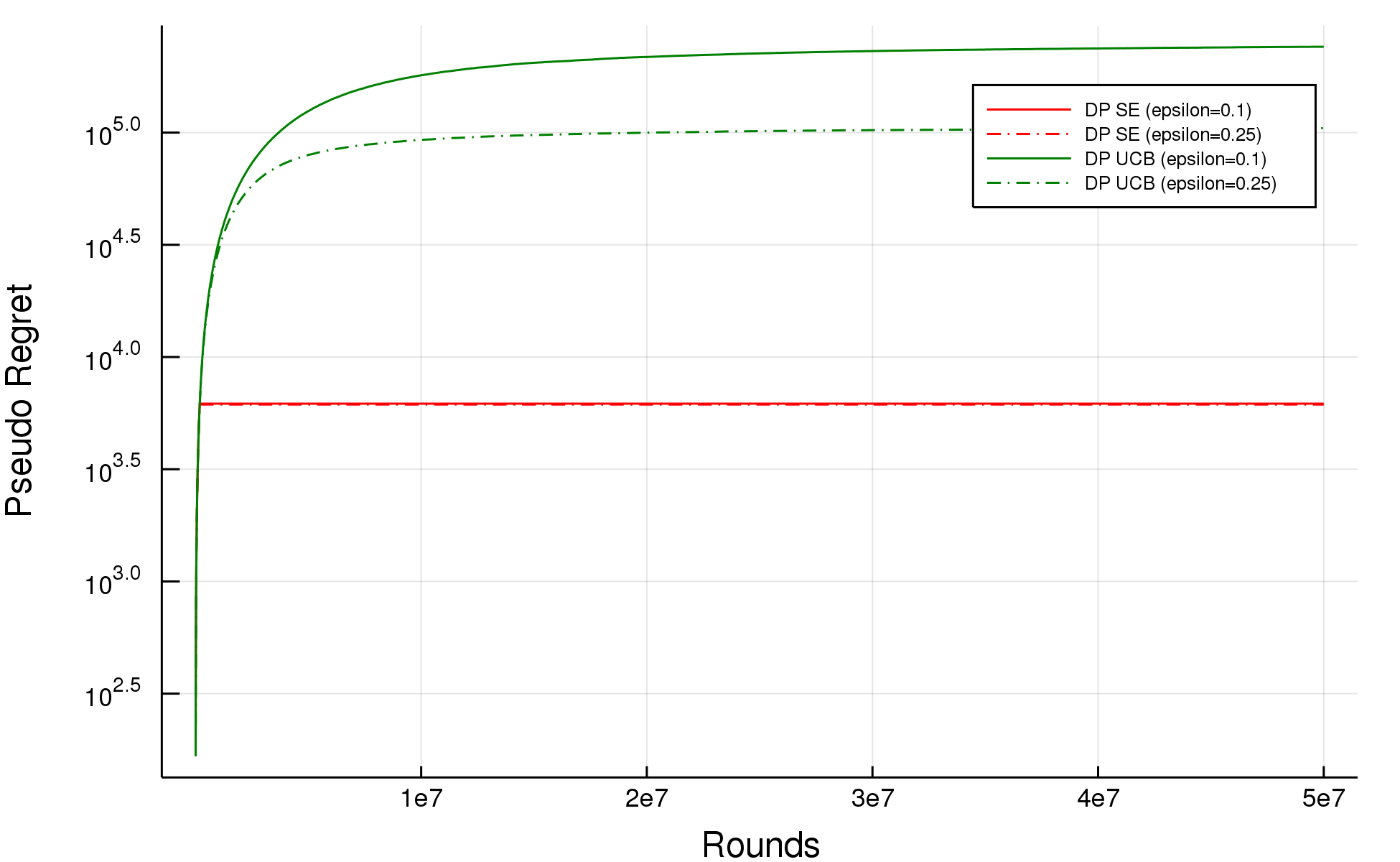}
    	\caption{$K=3$}
    \end{subfigure} 
    
    \begin{subfigure}[h]{0.5\textwidth}
    	\includegraphics[width=0.95\textwidth]{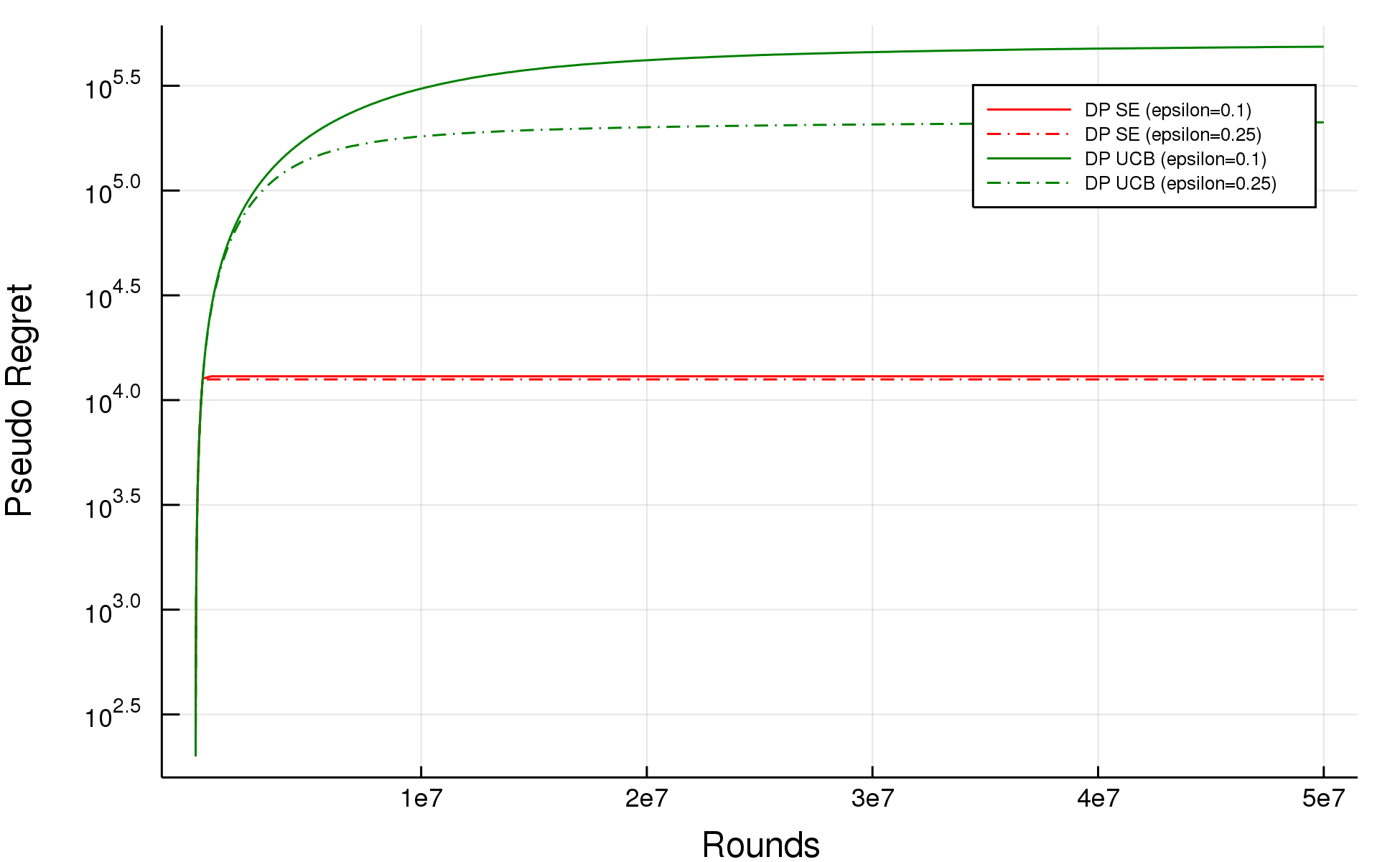}
    	\caption{$K=5$}
    \end{subfigure}
    
    \begin{subfigure}[h]{0.5\textwidth}
    	\includegraphics[width=0.95\textwidth]{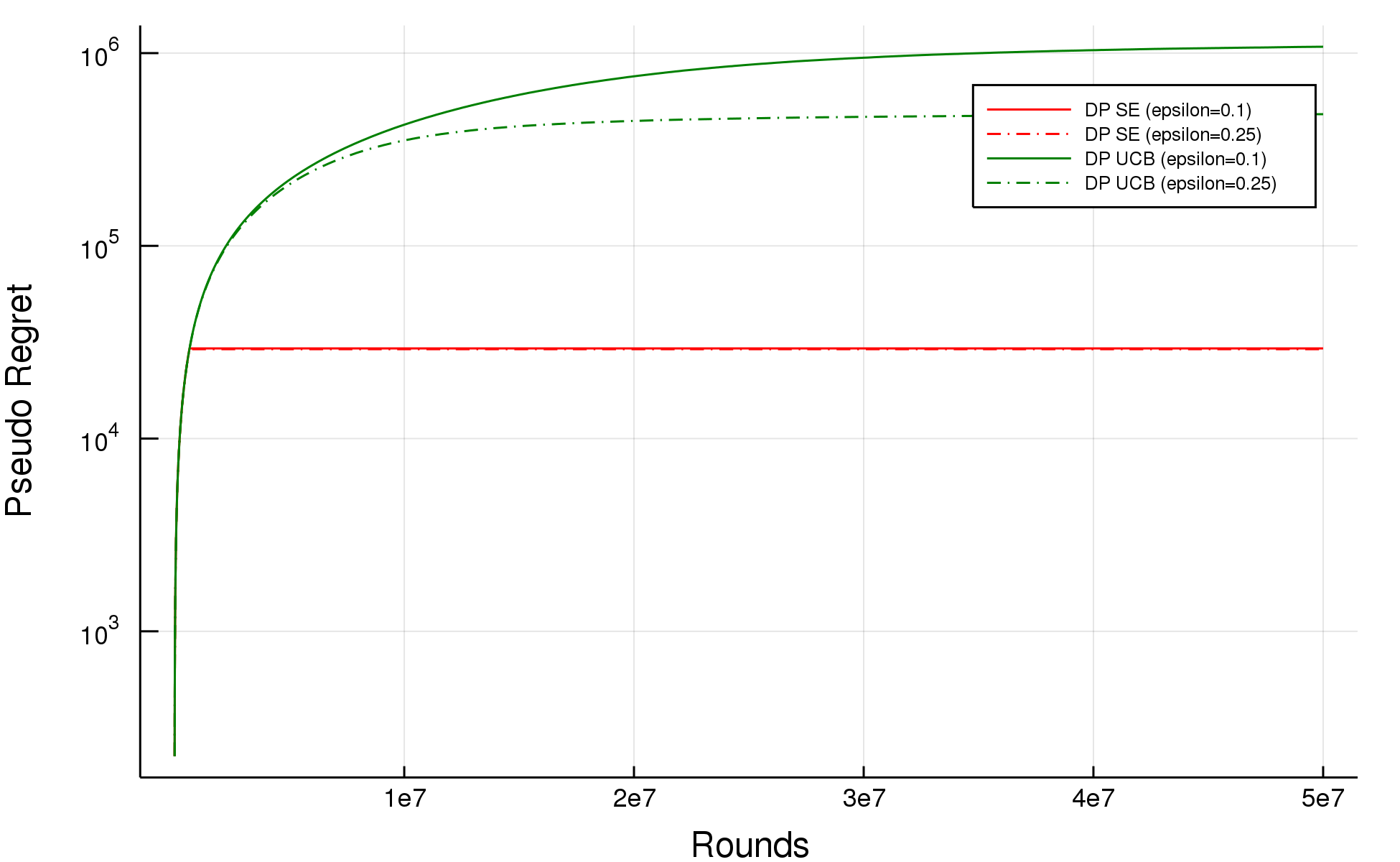}
    	\caption{$K=10$}
    \end{subfigure}
    
    \begin{subfigure}[h]{0.5\textwidth}
    	\includegraphics[width=0.95\textwidth]{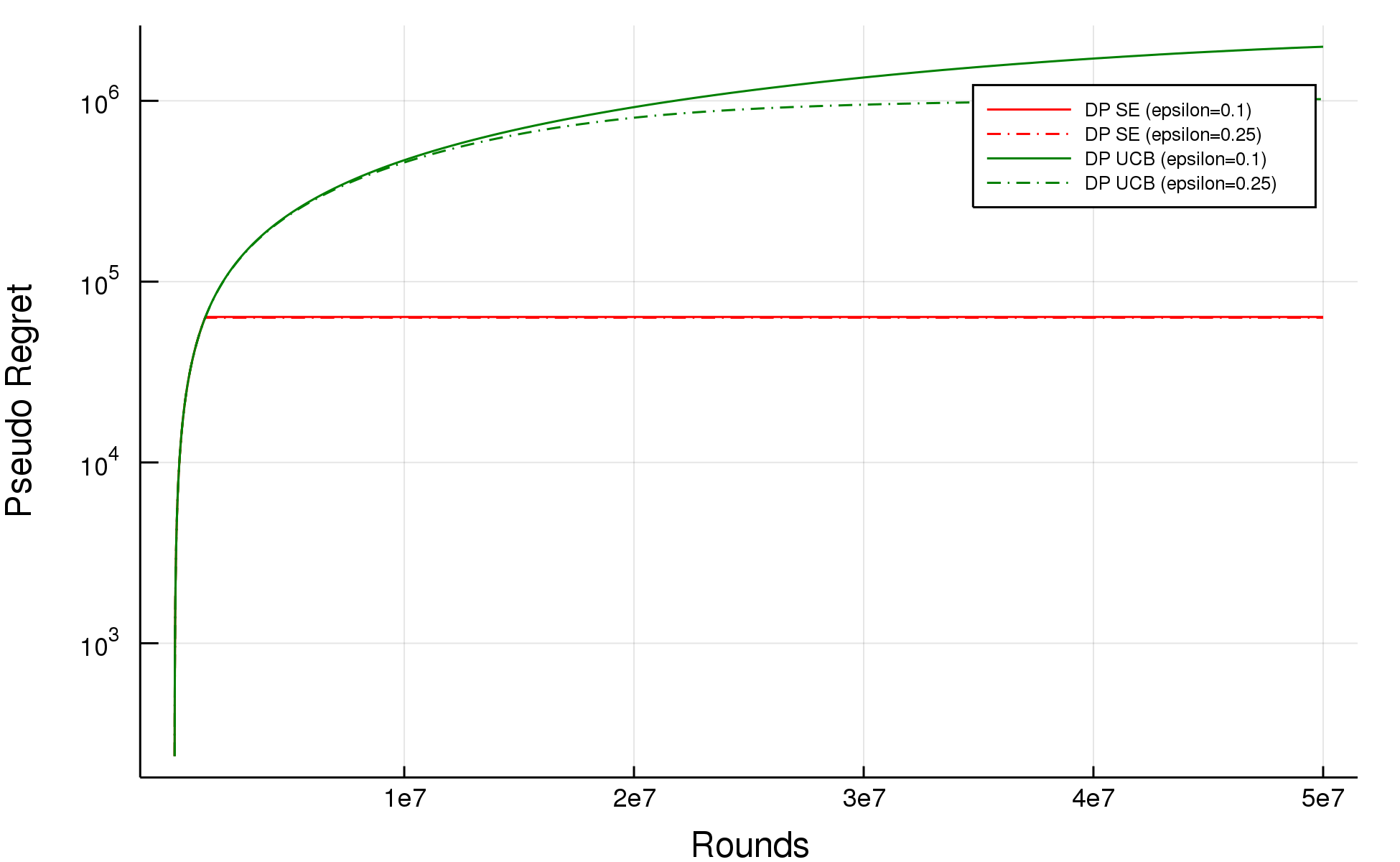}
    	\caption{$K=20$}
    \end{subfigure}
    
    \caption{\label{fig:varyK|eps0.1&0.25|setting1} Under $C_1$ with $\epsDP \in \{0.1,0.25\}, T=5 \times 10^7$}
    \end{center}
\end{figure}

\begin{figure}[h!]
    \begin{center}
    \begin{subfigure}[h]{0.5\textwidth}
    	\includegraphics[width=0.95\textwidth]{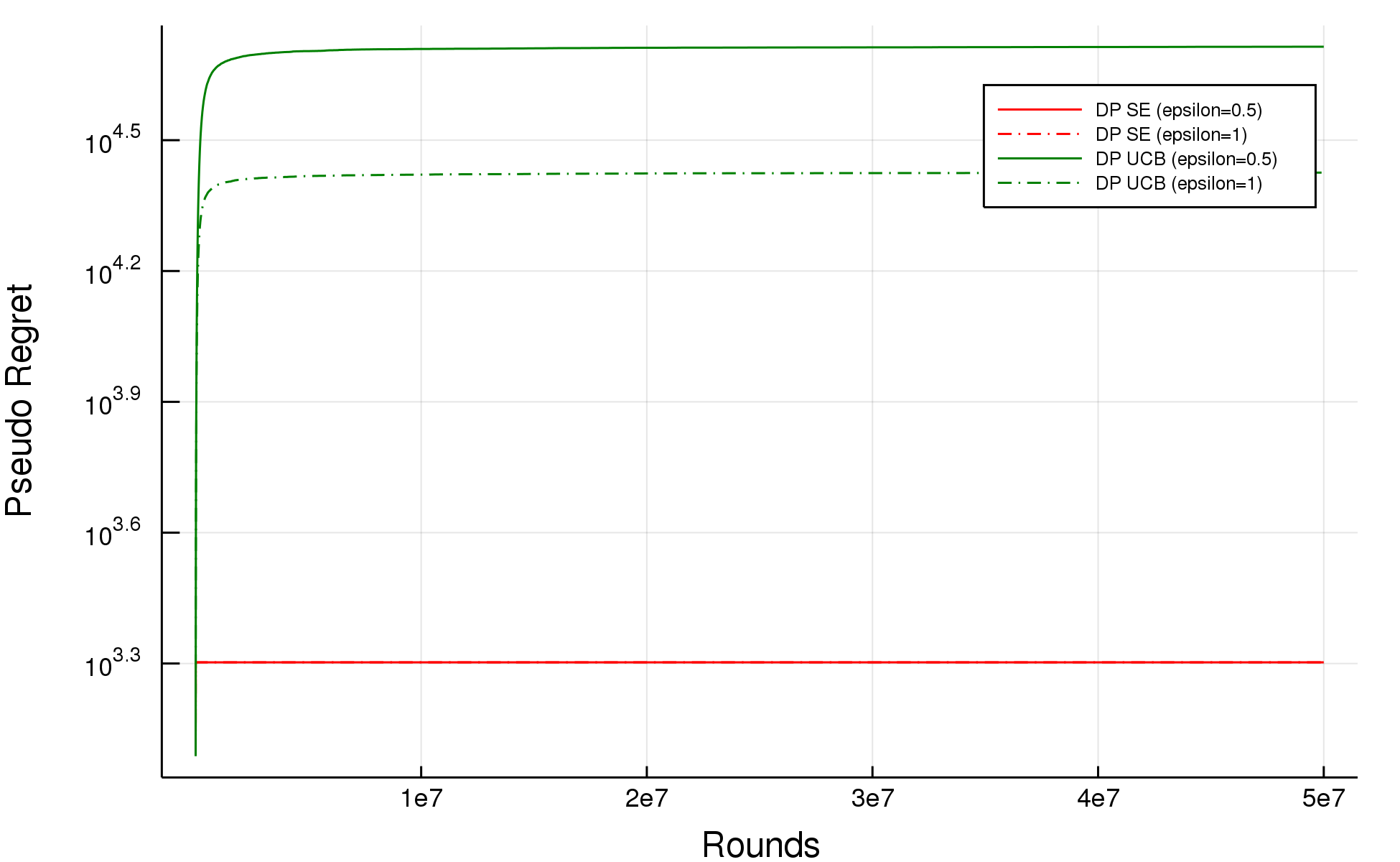}
    	\caption{$K=3$}
    \end{subfigure} 
    
    \begin{subfigure}[h]{0.5\textwidth}
    	\includegraphics[width=0.95\textwidth]{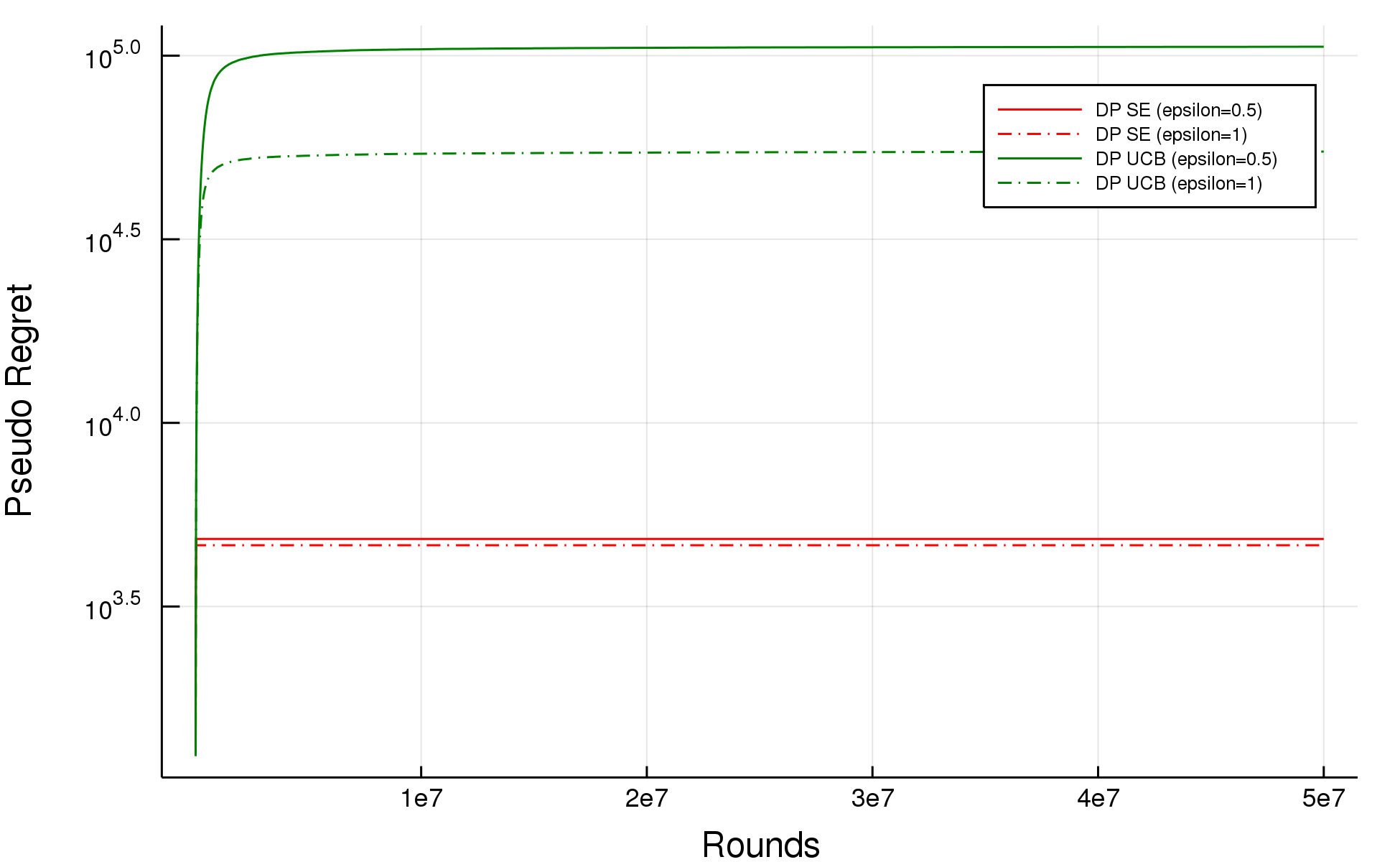}
    	\caption{$K=5$}
    \end{subfigure}
    
    \begin{subfigure}[h]{0.5\textwidth}
    	\includegraphics[width=0.95\textwidth]{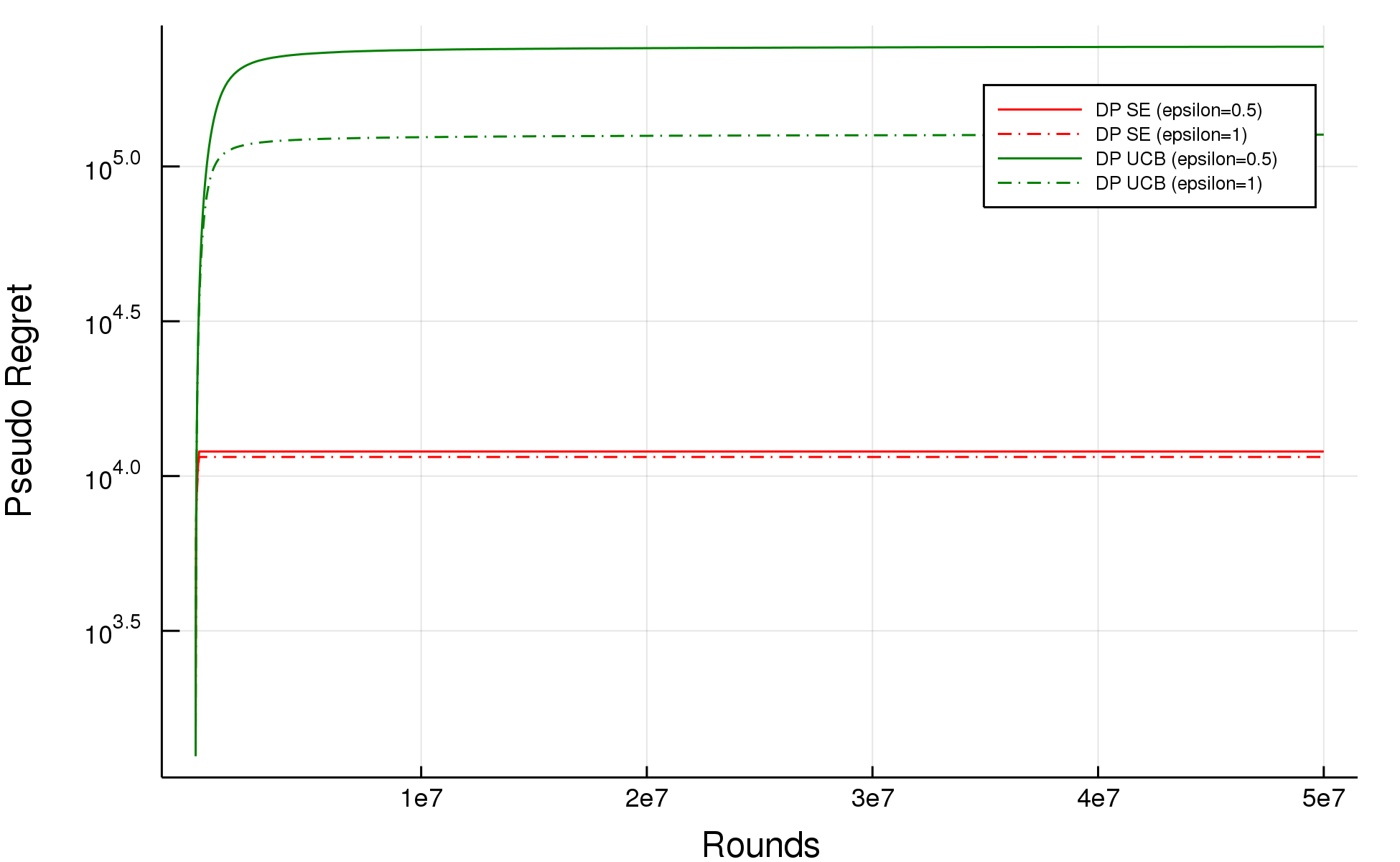}
    	\caption{$K=10$}
    \end{subfigure}
    
    \begin{subfigure}[h]{0.5\textwidth}
    	\includegraphics[width=0.95\textwidth]{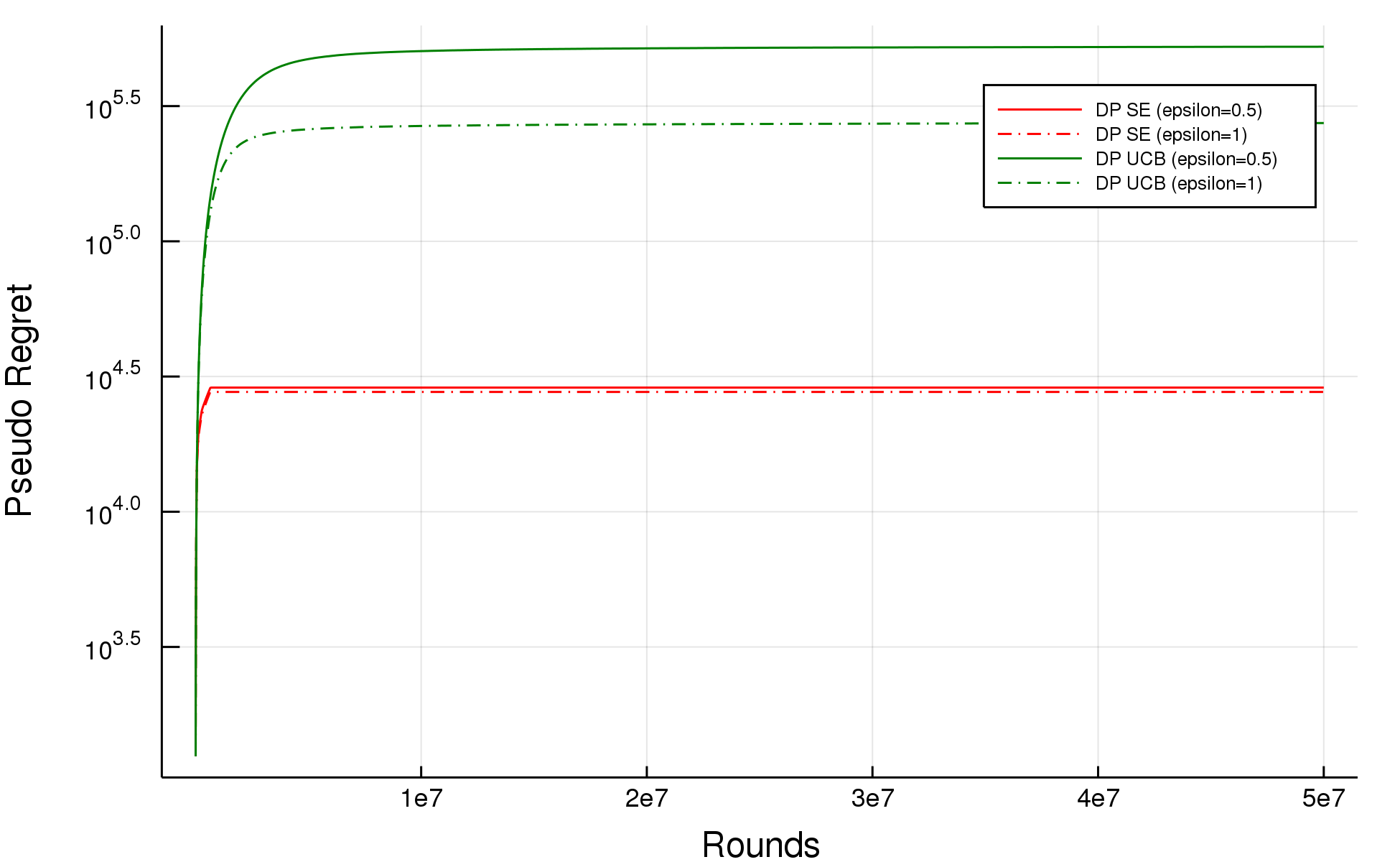}
    	\caption{$K=20$}
    \end{subfigure}
    
    \caption{\label{fig:varyK|eps0.5&1|setting2} Under $C_2$ with $\epsDP \in \{0.5,1\}, T=5 \times 10^7$}
    \end{center}
\end{figure}

\begin{figure}[h!]
    \begin{center}
    \begin{subfigure}[h]{0.5\textwidth}
    	\includegraphics[width=0.95\textwidth]{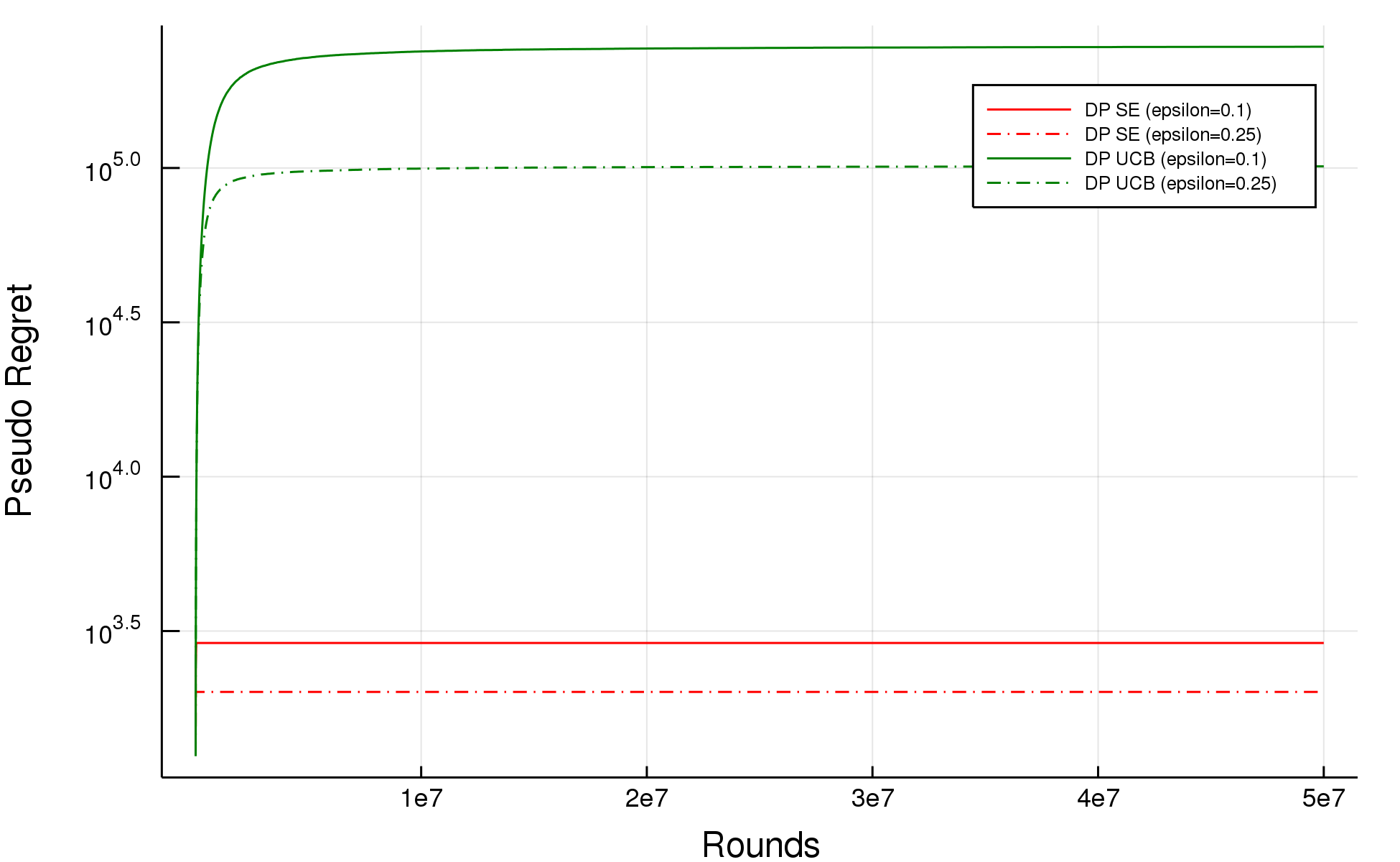}
    	\caption{$K=3$}
    \end{subfigure} 
    
    \begin{subfigure}[h]{0.5\textwidth}
    	\includegraphics[width=0.95\textwidth]{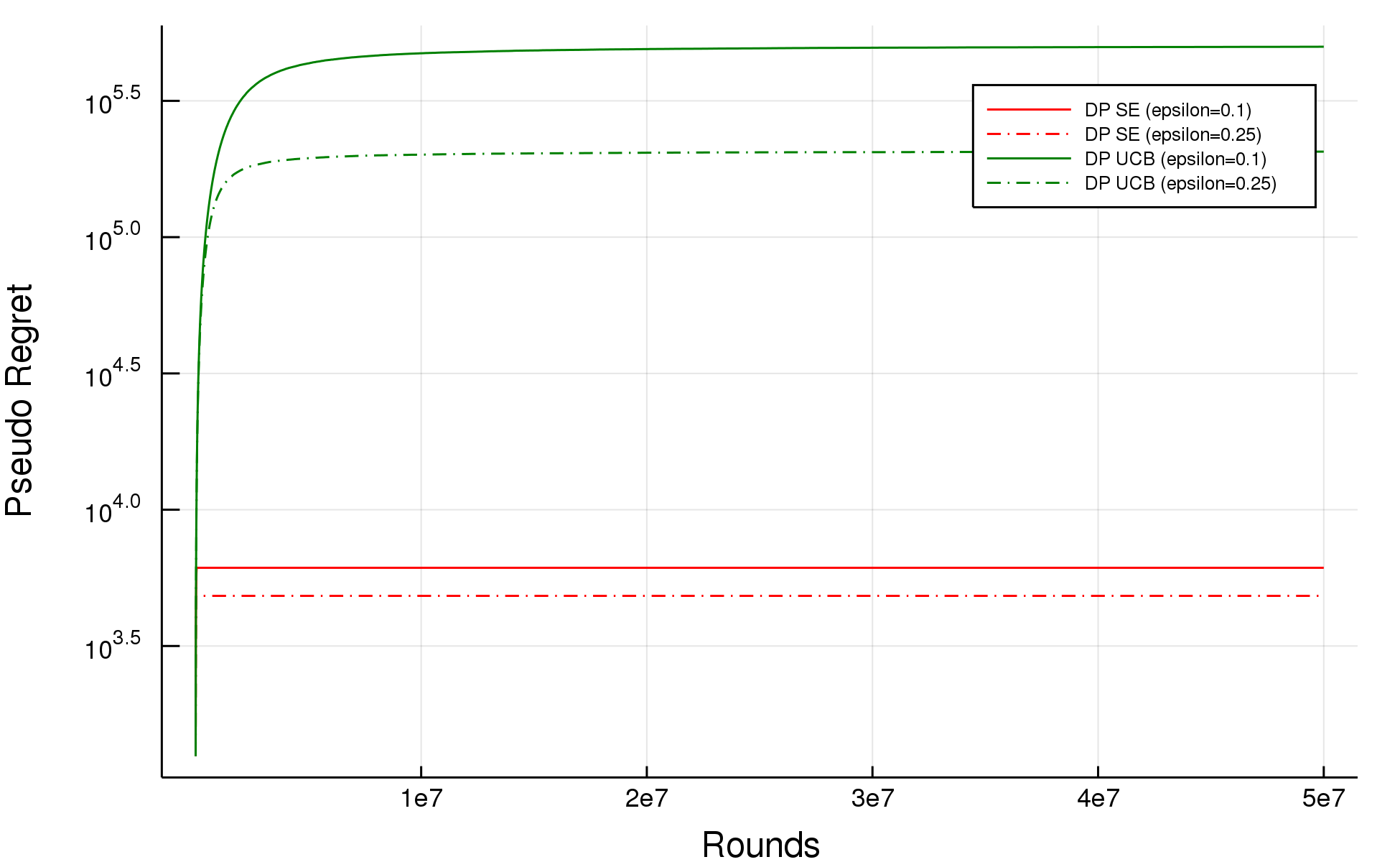}
    	\caption{$K=5$}
    \end{subfigure}
    
    \begin{subfigure}[h]{0.5\textwidth}
    	\includegraphics[width=0.95\textwidth]{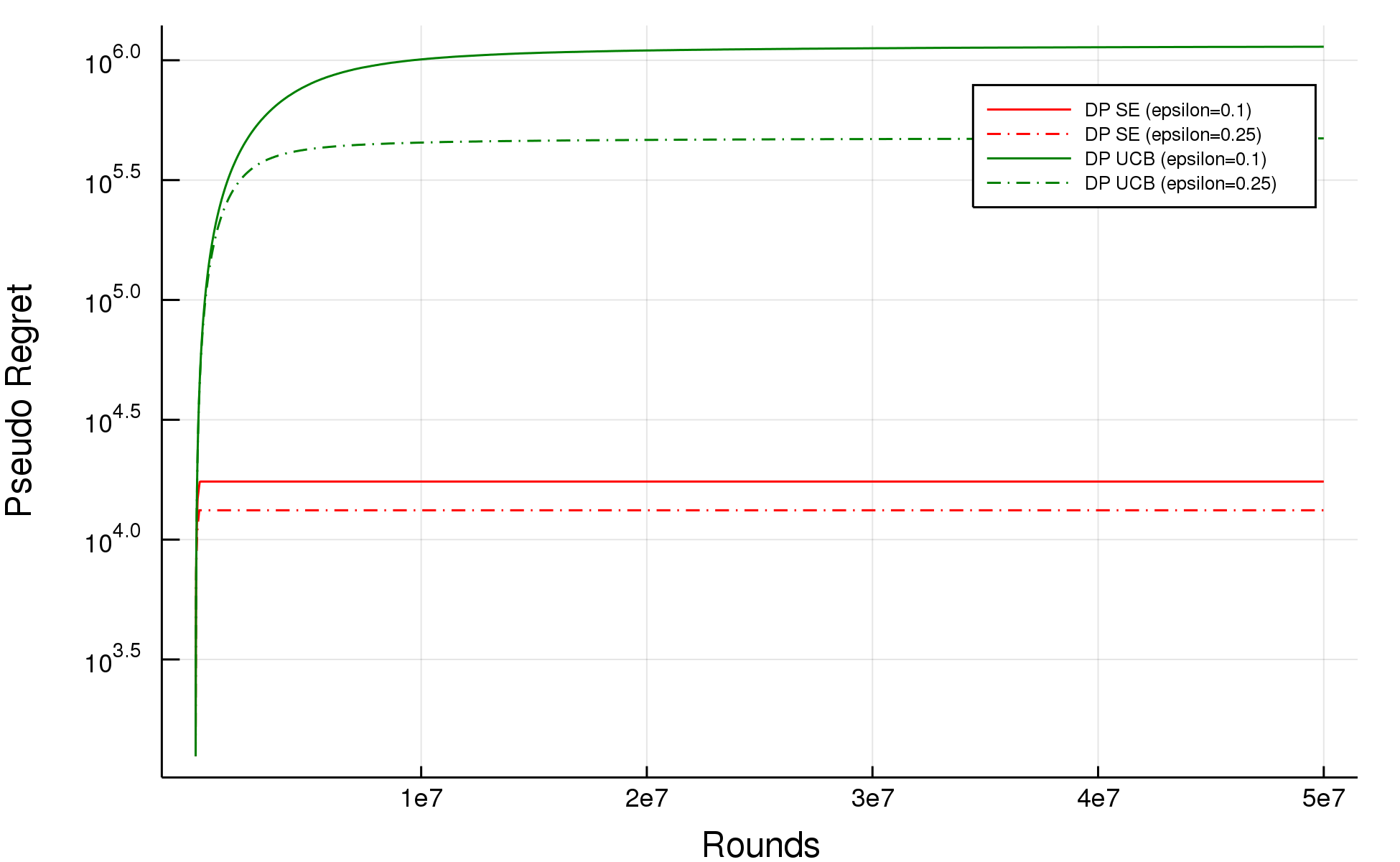}
    	\caption{$K=10$}
    \end{subfigure}
    
    \begin{subfigure}[h]{0.5\textwidth}
    	\includegraphics[width=0.95\textwidth]{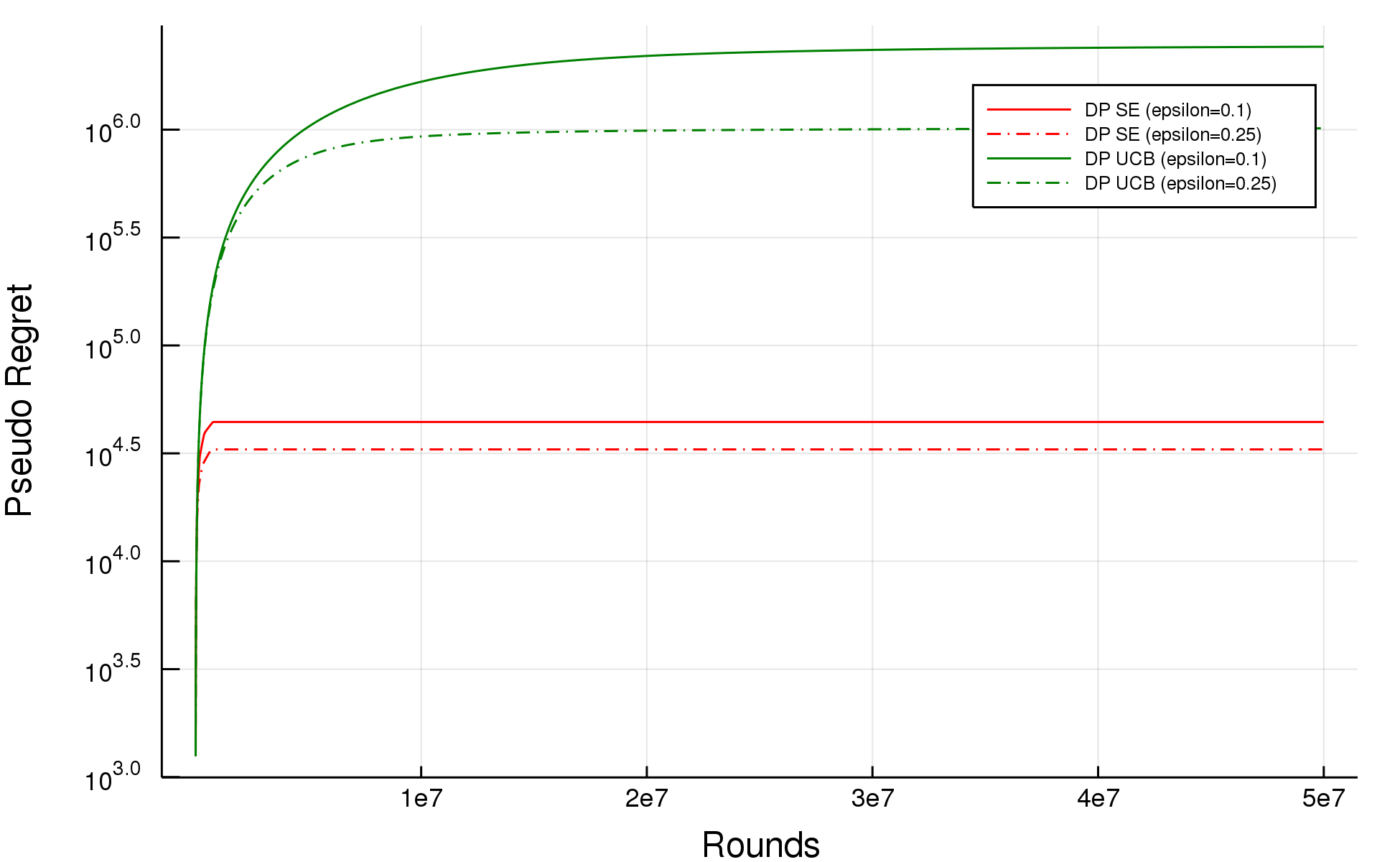}
    	\caption{$K=20$}
    \end{subfigure}
    
    \caption{\label{fig:varyK|eps0.1&0.25|setting2} Under $C_2$ with $\epsDP \in \{0.1,0.25\}, T=5 \times 10^7$}
    \end{center}
\end{figure}

\begin{figure}[h!]
    \begin{center}
    \begin{subfigure}[h]{0.5\textwidth}
    	\includegraphics[width=0.95\textwidth]{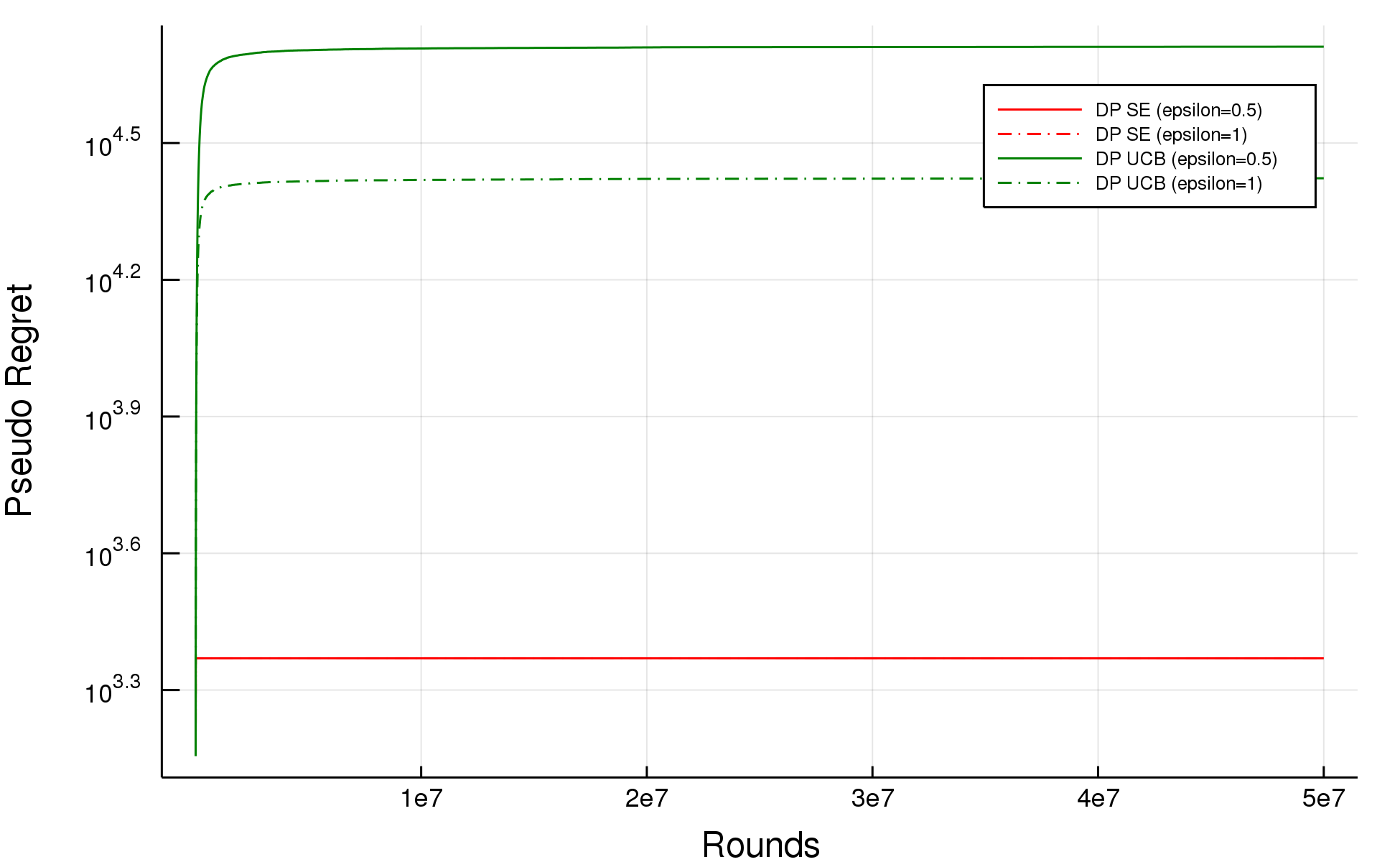}
    	\caption{$K=3$}
    \end{subfigure} 
    
    \begin{subfigure}[h]{0.5\textwidth}
    	\includegraphics[width=0.95\textwidth]{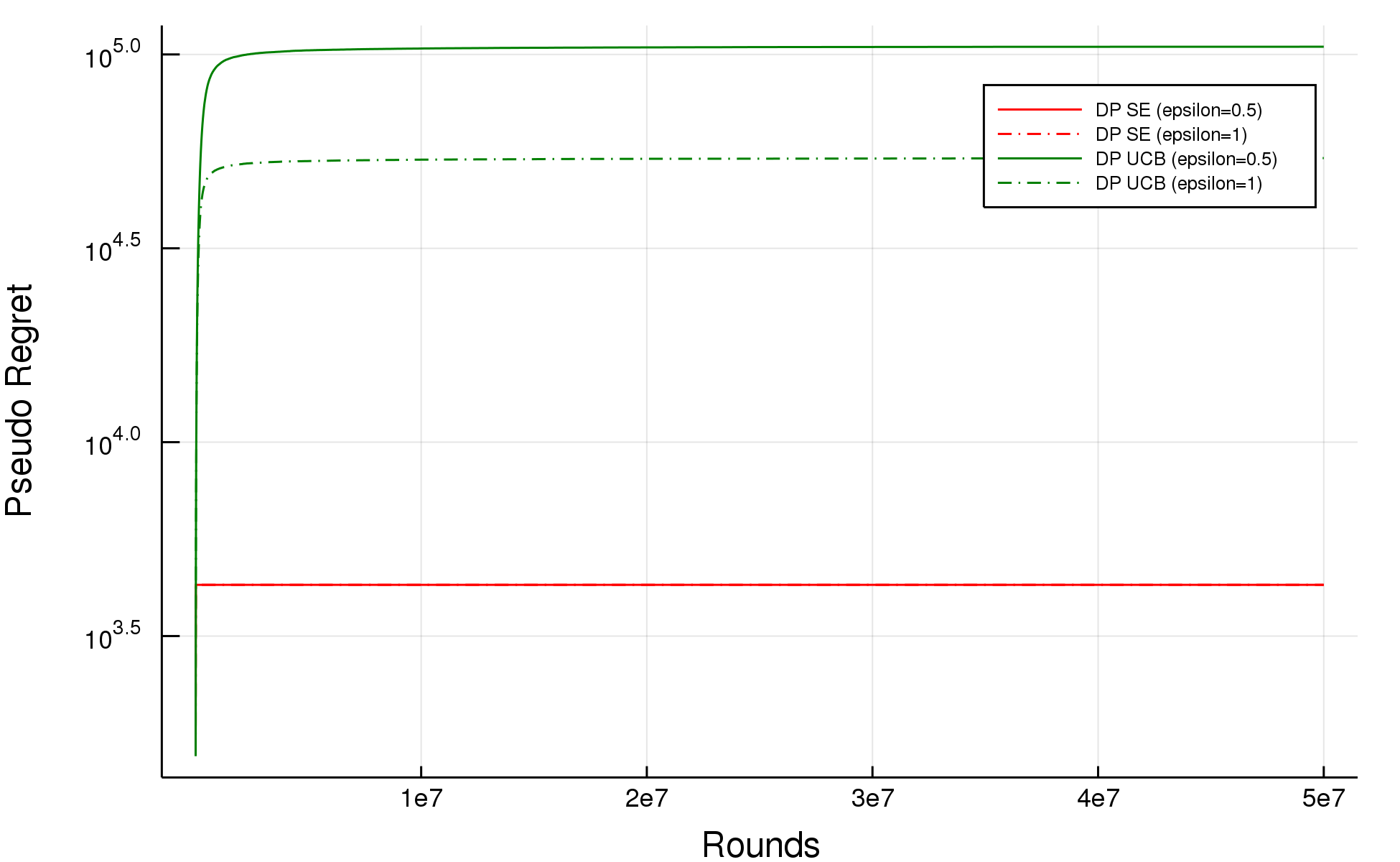}
    	\caption{$K=5$}
    \end{subfigure}
    
    \begin{subfigure}[h]{0.5\textwidth}
    	\includegraphics[width=0.95\textwidth]{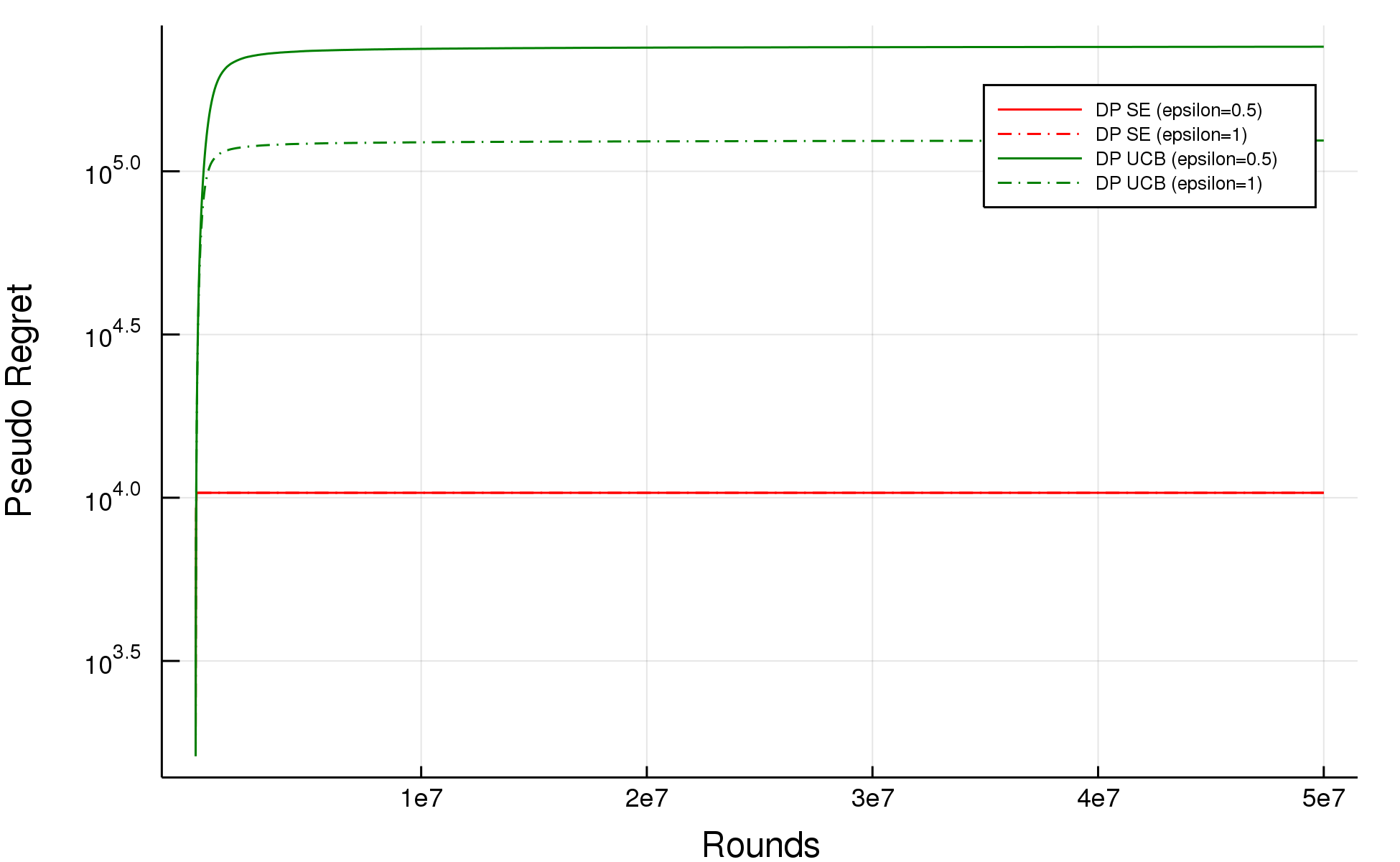}
    	\caption{$K=10$}
    \end{subfigure}
    
    \begin{subfigure}[h]{0.5\textwidth}
    	\includegraphics[width=0.95\textwidth]{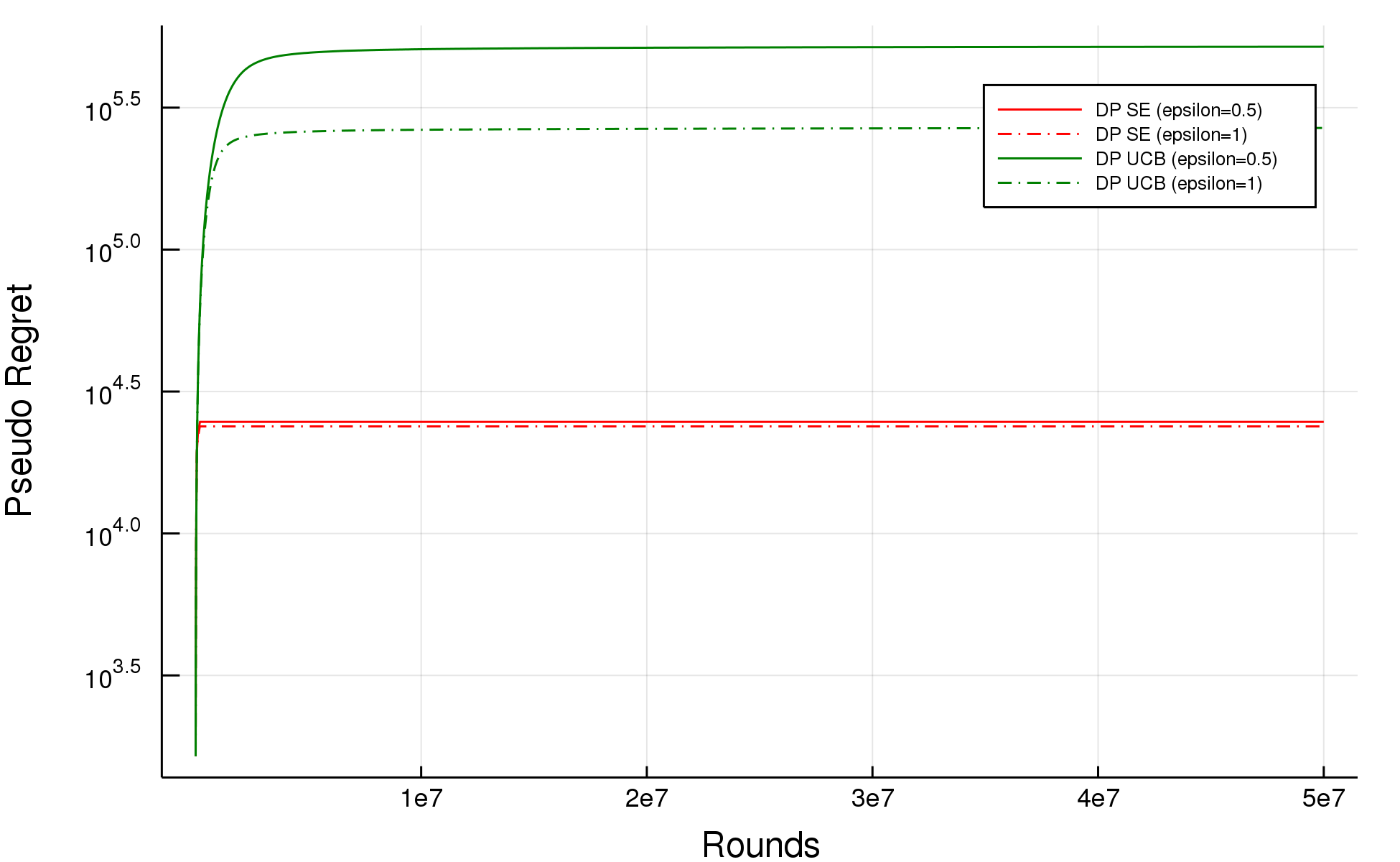}
    	\caption{$K=20$}
    \end{subfigure}
    
    \caption{\label{fig:varyK|eps0.5&1|setting3} Under $C_3$ with $\epsDP \in \{0.5,1\}, T=5 \times 10^7$}
    \end{center}
\end{figure}

\begin{figure}[h!]
    \begin{center}
    \begin{subfigure}[h]{0.5\textwidth}
    	\includegraphics[width=0.95\textwidth]{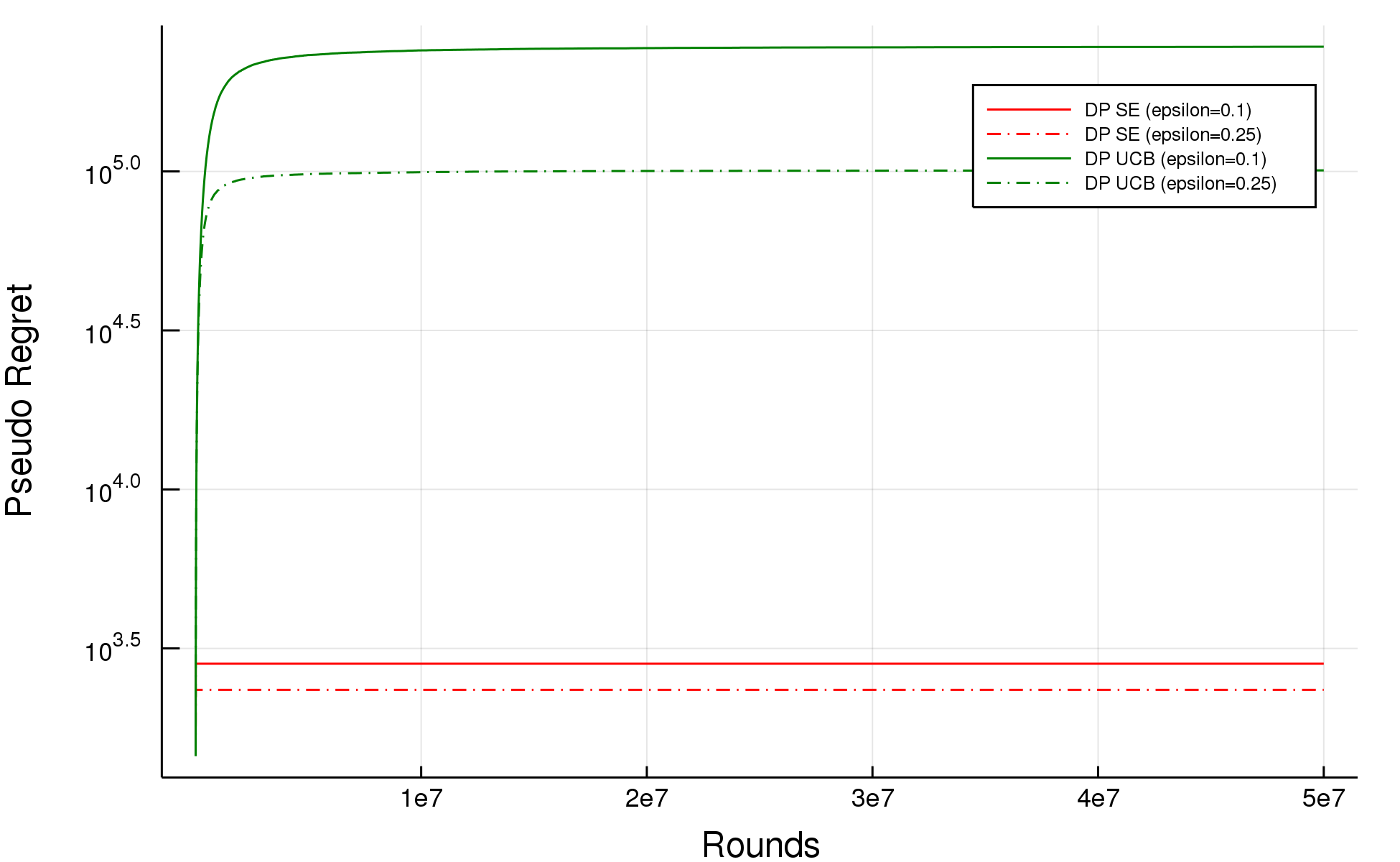}
    	\caption{$K=3$}
    \end{subfigure} 
    
    \begin{subfigure}[h]{0.5\textwidth}
    	\includegraphics[width=0.95\textwidth]{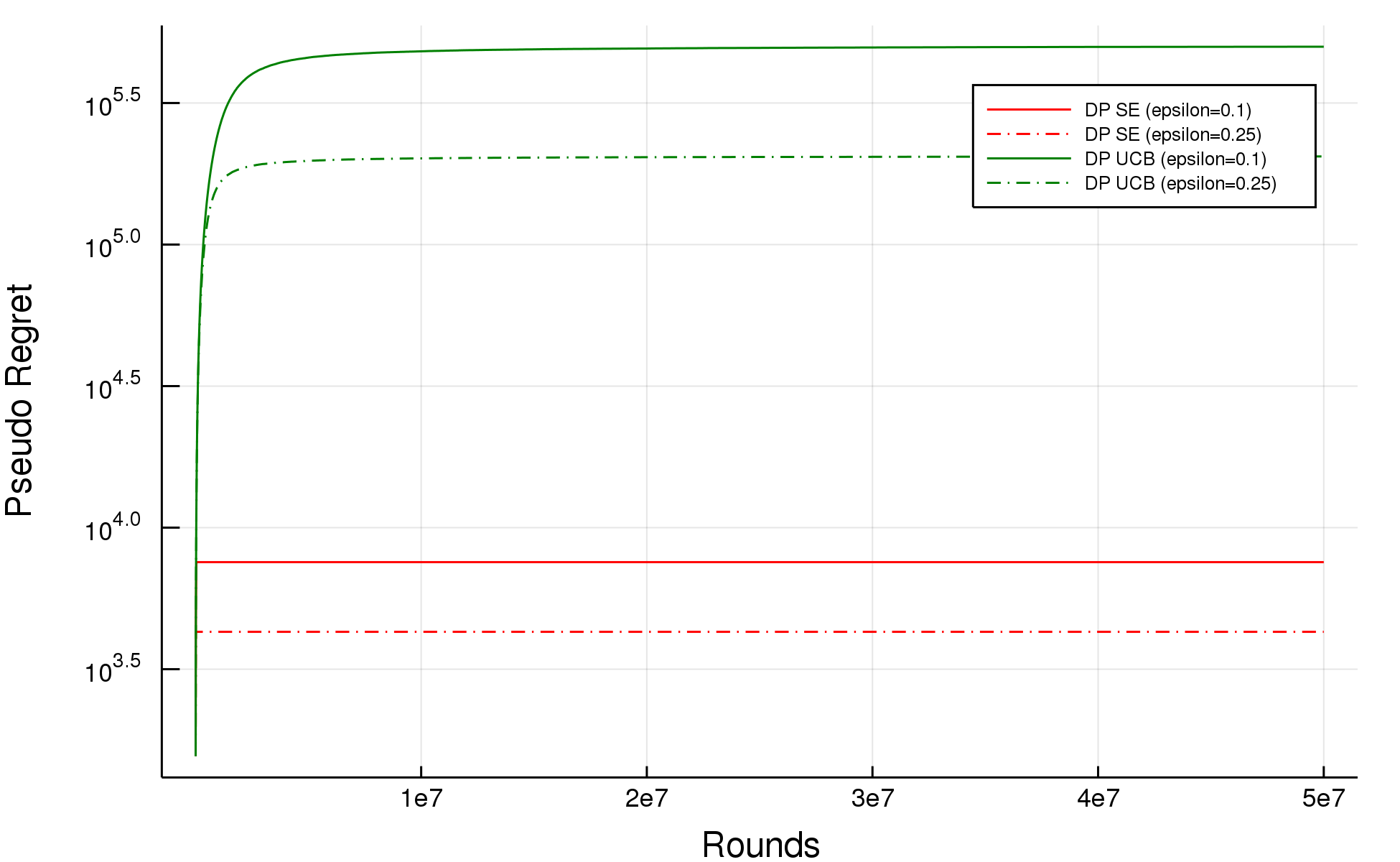}
    	\caption{$K=5$}
    \end{subfigure}
    
    \begin{subfigure}[h]{0.5\textwidth}
    	\includegraphics[width=0.95\textwidth]{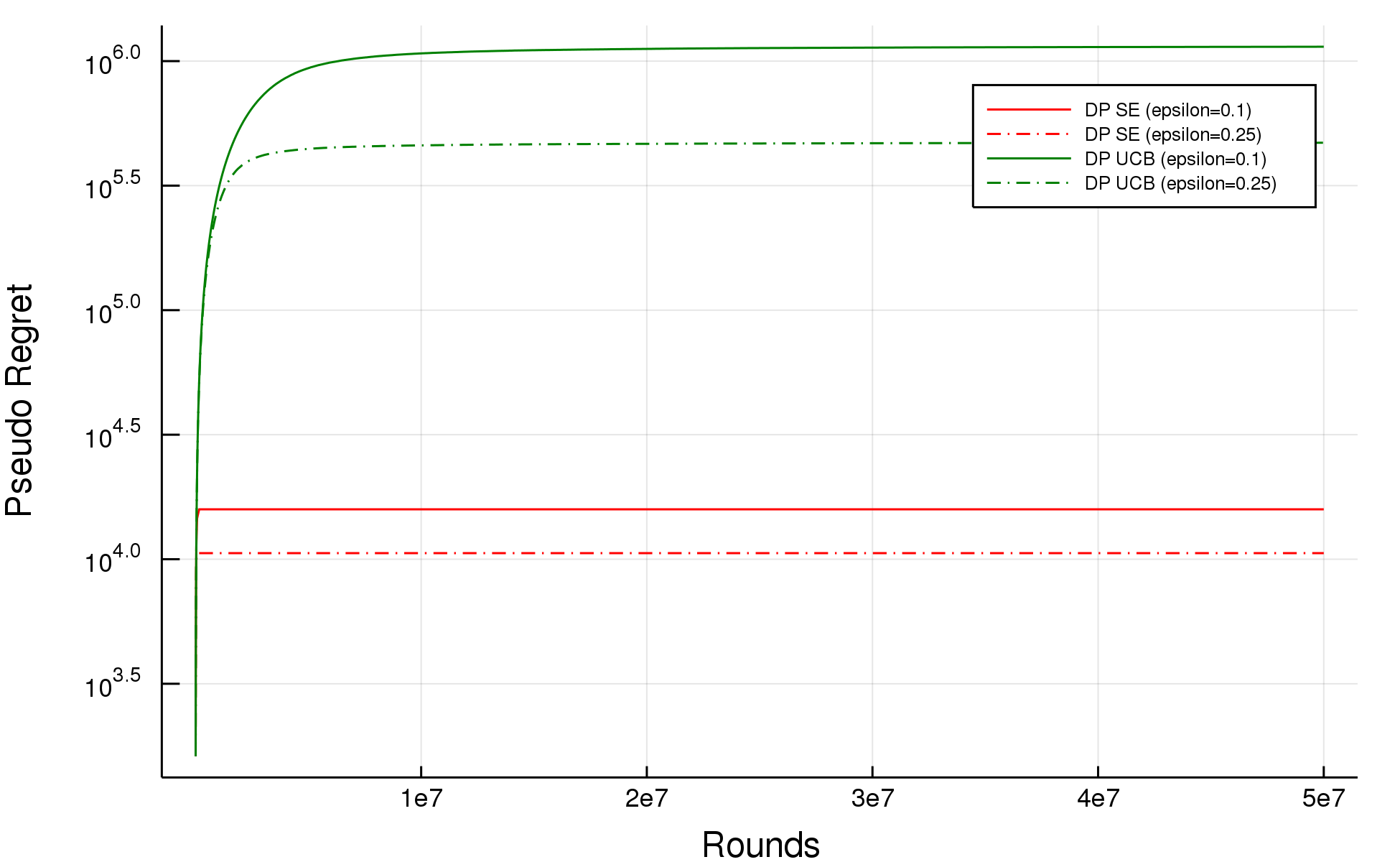}
    	\caption{$K=10$}
    \end{subfigure}
    
    \begin{subfigure}[h]{0.5\textwidth}
    	\includegraphics[width=0.95\textwidth]{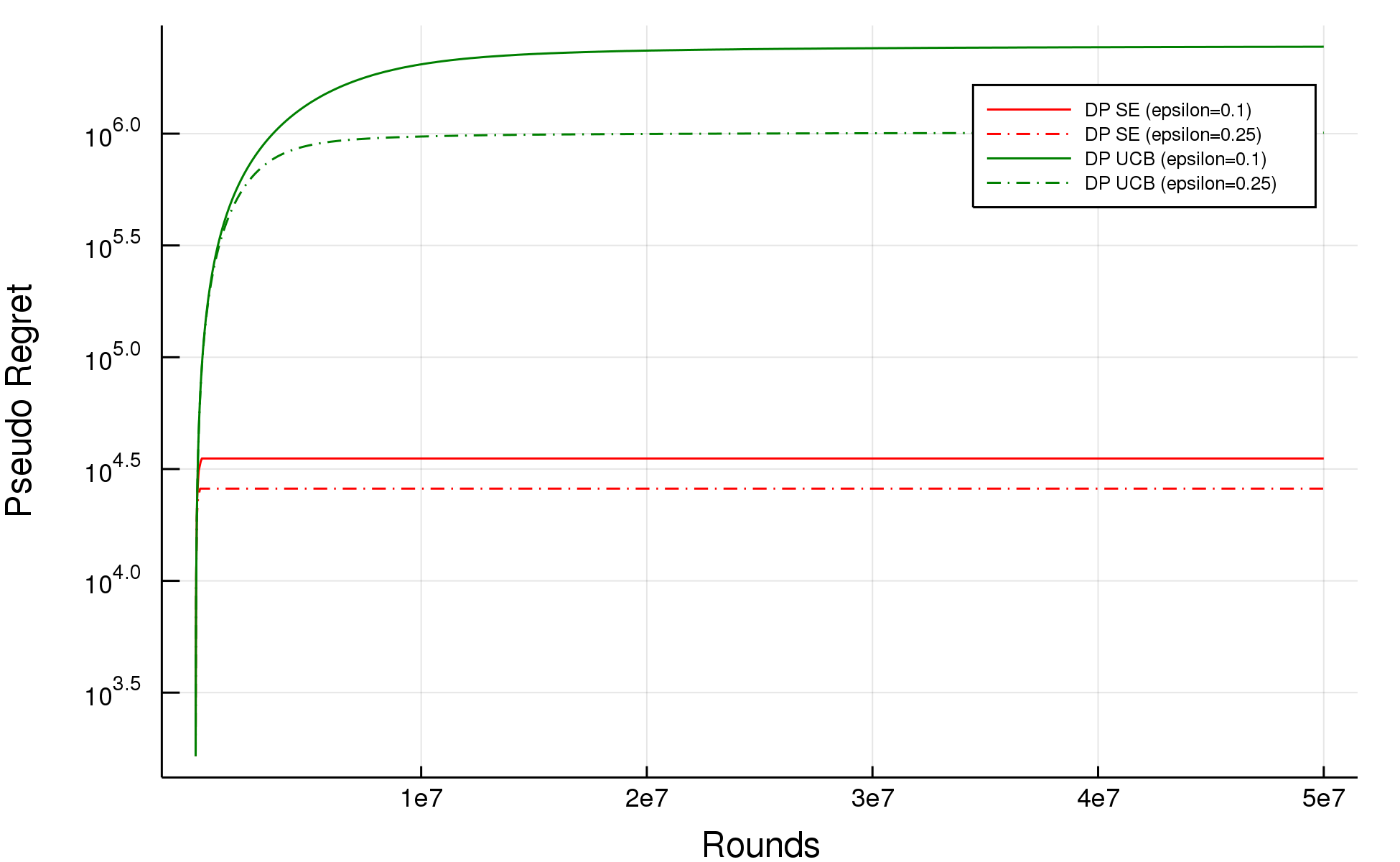}
    	\caption{$K=20$}
    \end{subfigure}
    
    \caption{\label{fig:varyK|eps0.1&0.25|setting3} Under $C_3$ with $\epsDP \in \{0.1,0.25\}, T=5 \times 10^7$}
    \end{center}
\end{figure}

\begin{figure}[h!]
    \begin{center}
    \begin{subfigure}[h]{0.5\textwidth}
    	\includegraphics[width=0.95\textwidth]{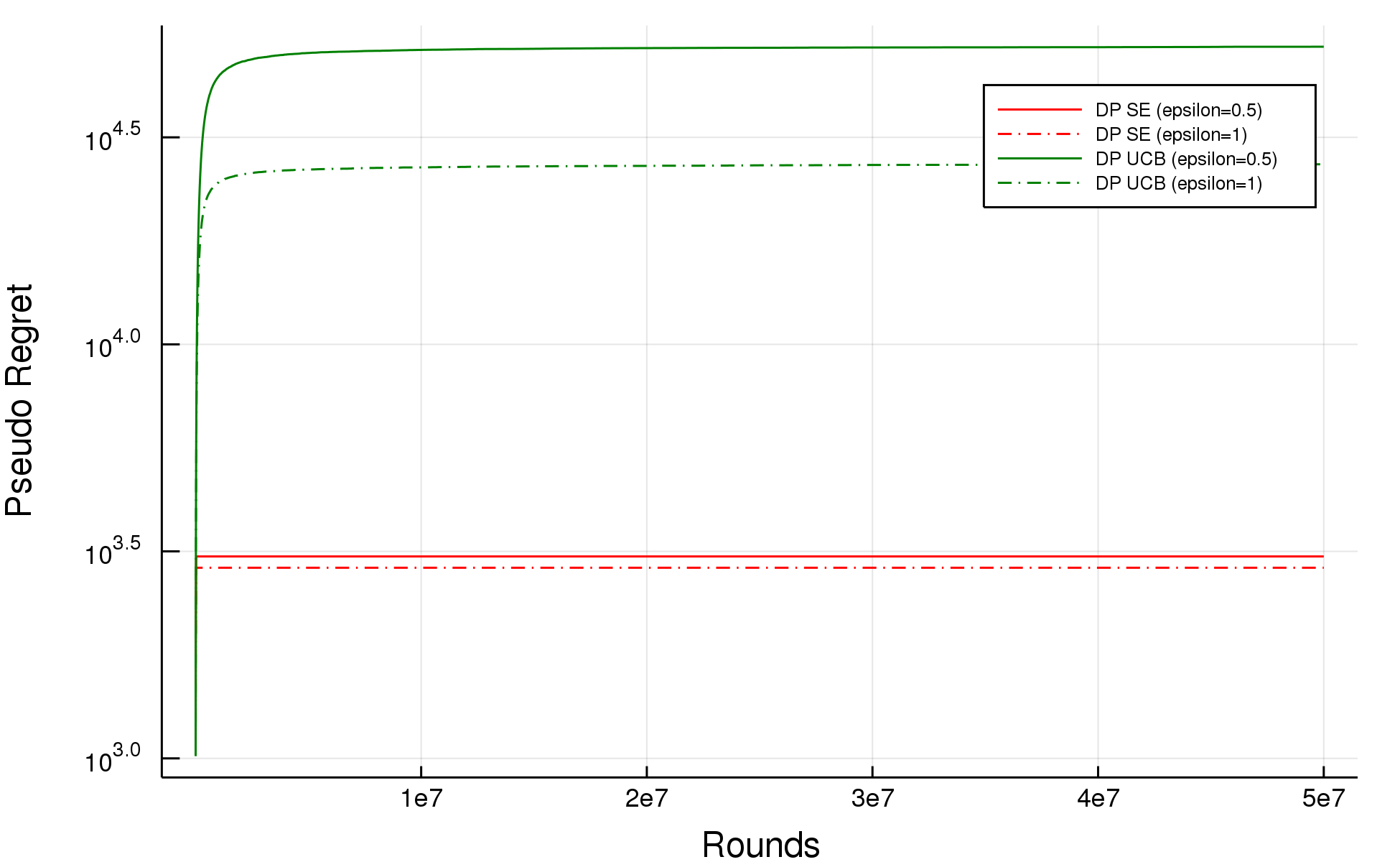}
    	\caption{$K=3$}
    \end{subfigure} 
    
    \begin{subfigure}[h]{0.5\textwidth}
    	\includegraphics[width=0.95\textwidth]{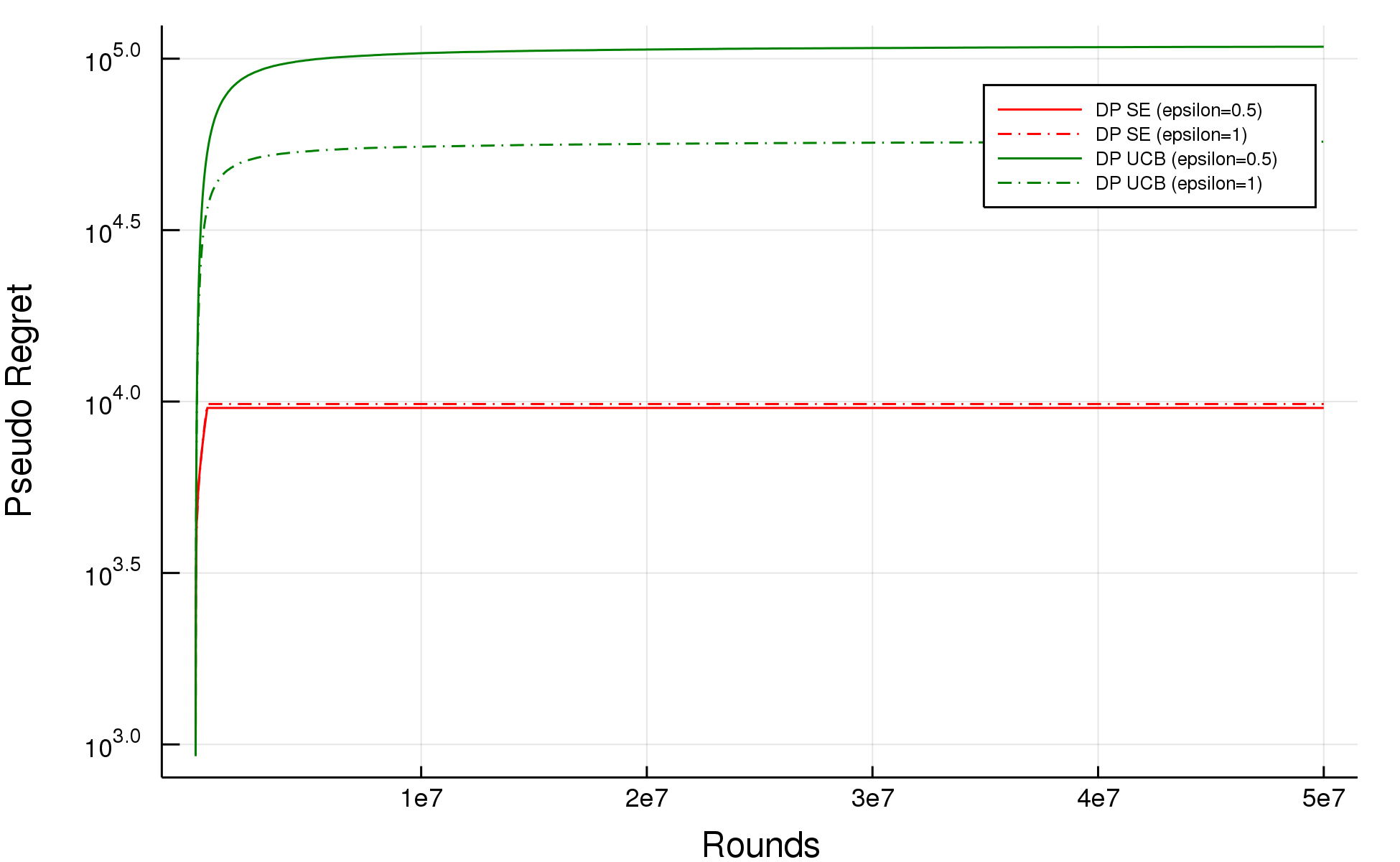}
    	\caption{$K=5$}
    \end{subfigure}
    
    \begin{subfigure}[h]{0.5\textwidth}
    	\includegraphics[width=0.95\textwidth]{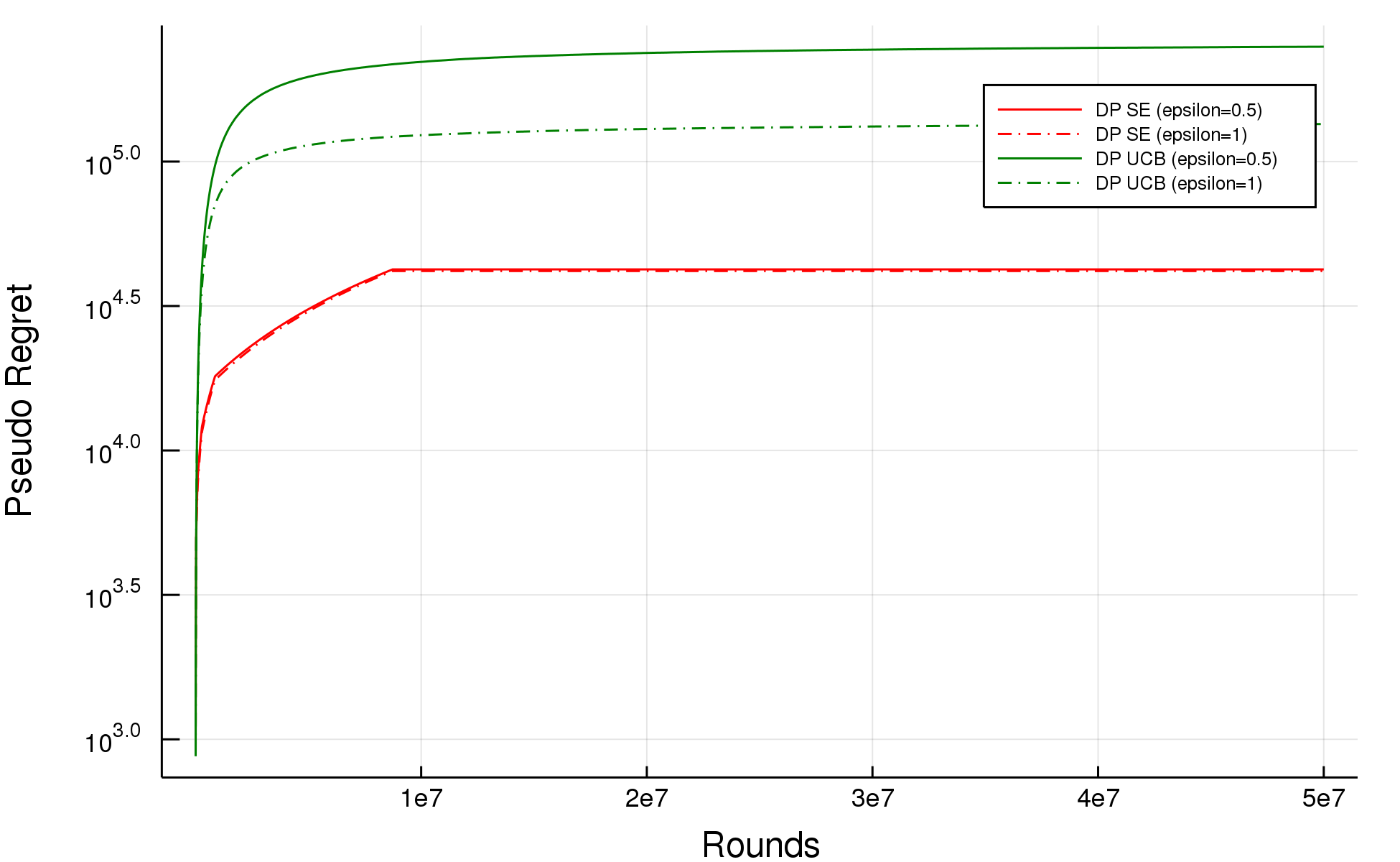}
    	\caption{$K=10$}
    \end{subfigure}
    
    \begin{subfigure}[h]{0.5\textwidth}
    	\includegraphics[width=0.95\textwidth]{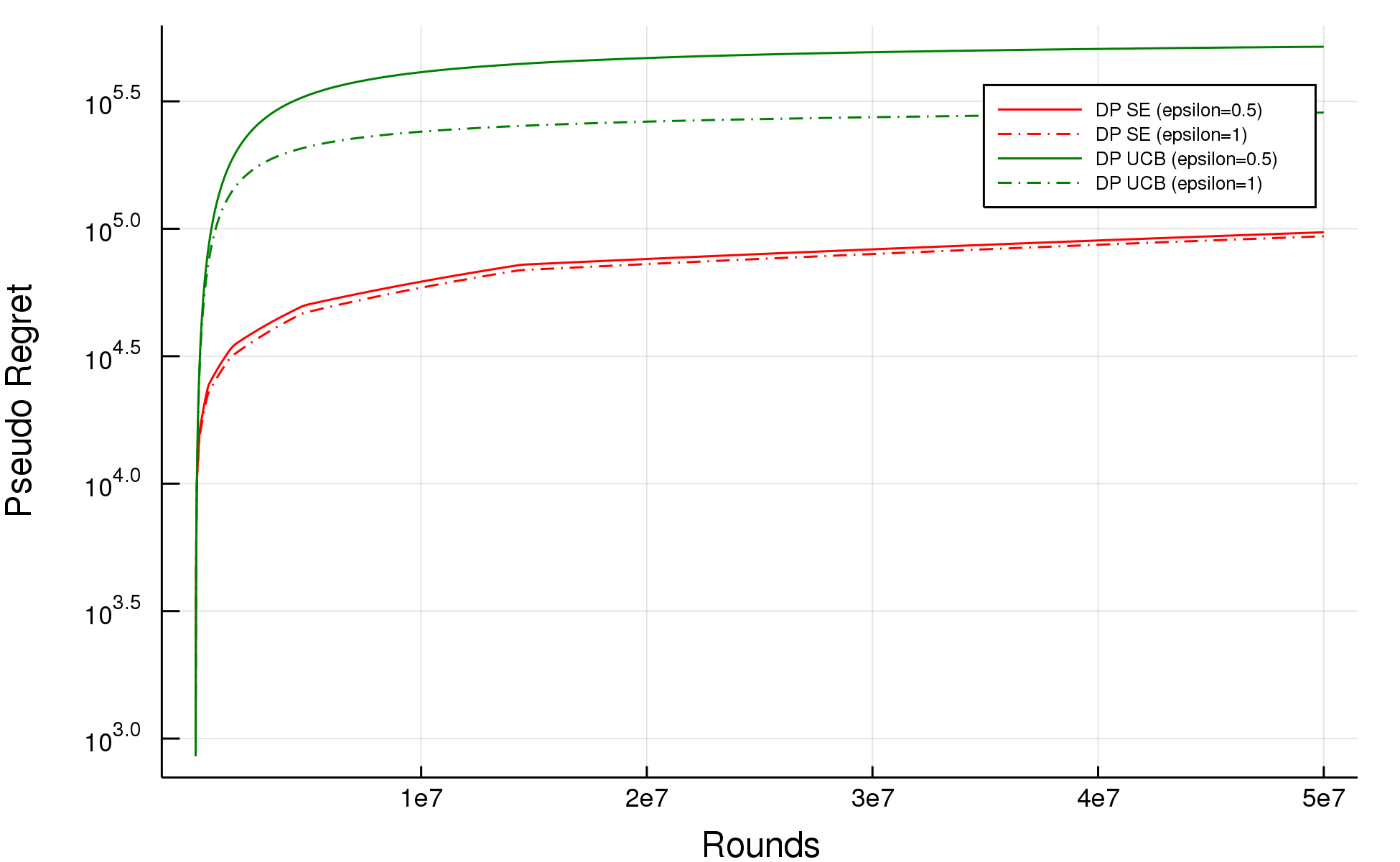}
    	\caption{$K=20$}
    \end{subfigure}
    
    \caption{\label{fig:varyK|eps0.5&1|setting4} Under $C_4$ with $\epsDP \in \{0.5,1\}, T=5 \times 10^7$}
    \end{center}
\end{figure}

\begin{figure}[h!]
    \begin{center}
    \begin{subfigure}[h]{0.5\textwidth}
    	\includegraphics[width=0.95\textwidth]{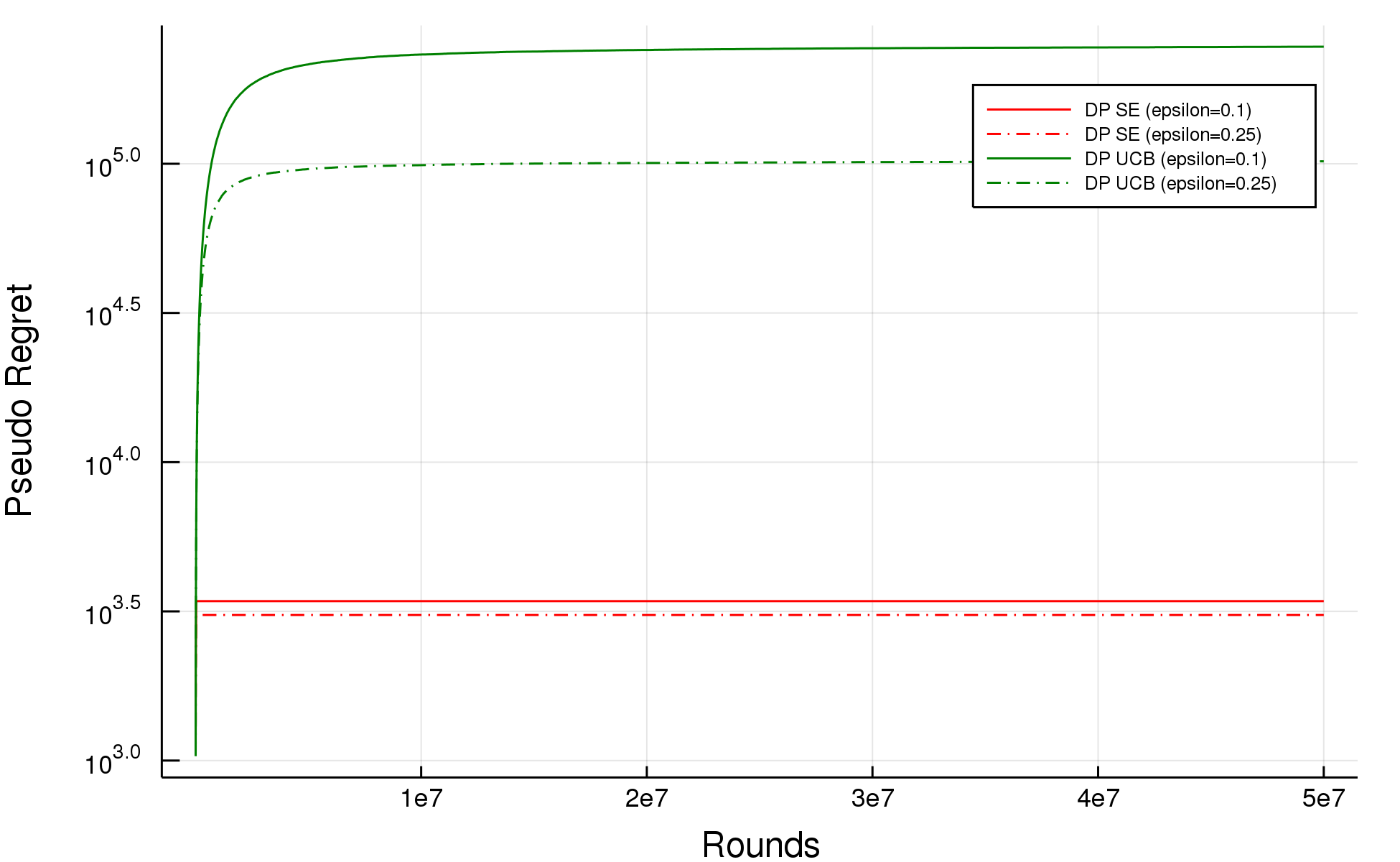}
    	\caption{$K=3$}
    \end{subfigure} 
    
    \begin{subfigure}[h]{0.5\textwidth}
    	\includegraphics[width=0.95\textwidth]{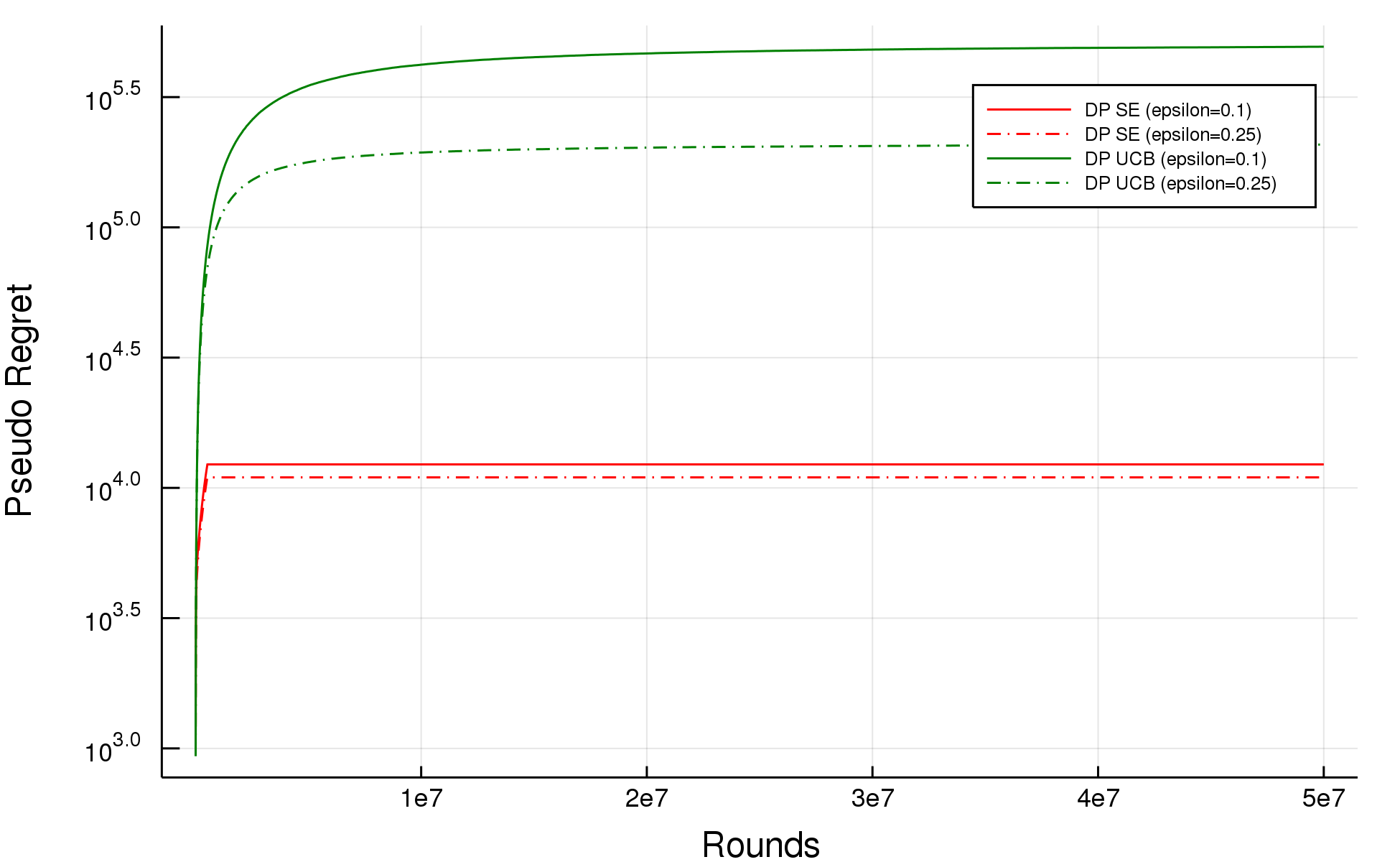}
    	\caption{$K=5$}
    \end{subfigure}
    
    \begin{subfigure}[h]{0.5\textwidth}
    	\includegraphics[width=0.95\textwidth]{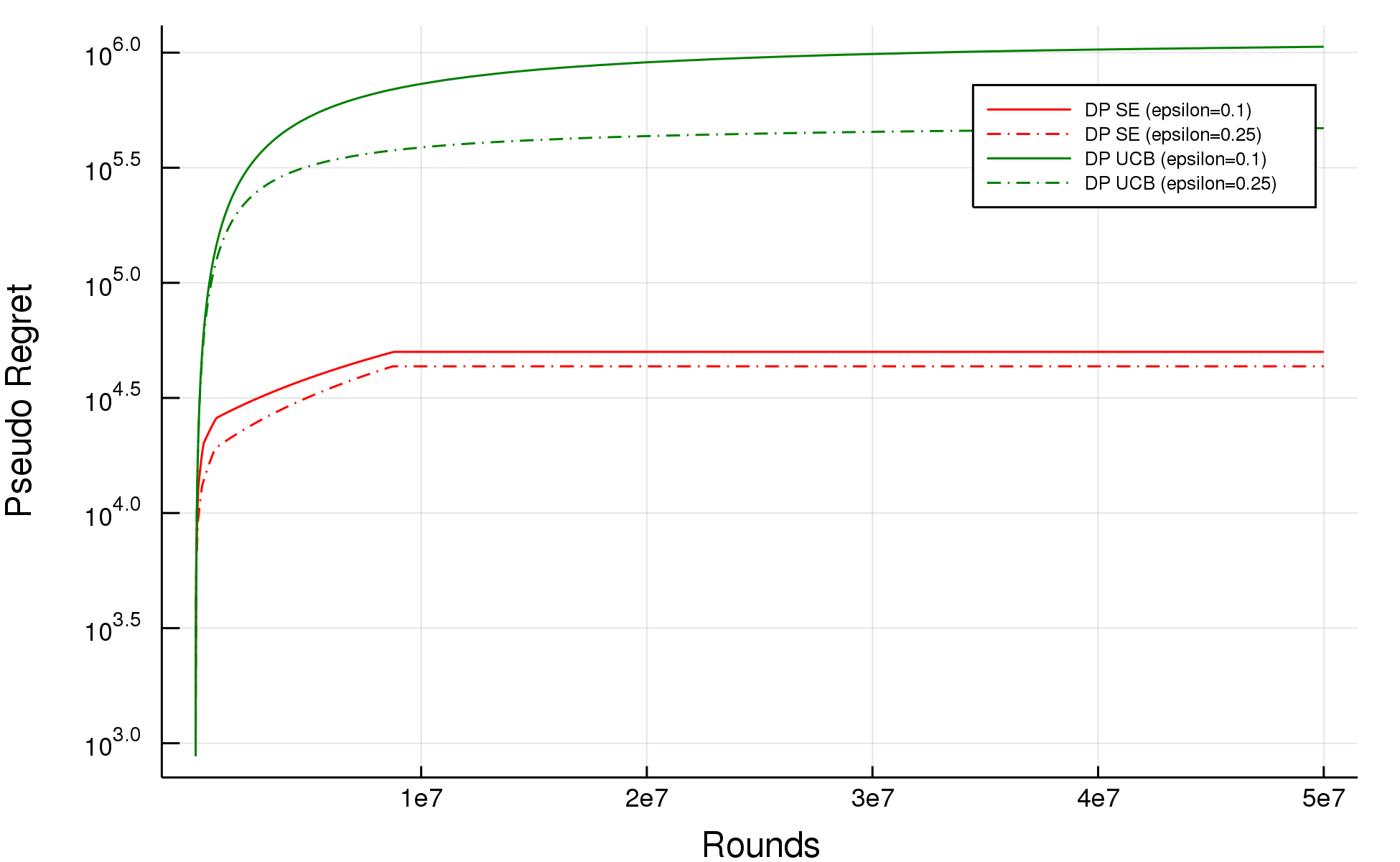}
    	\caption{$K=10$}
    \end{subfigure}
    
    \begin{subfigure}[h]{0.5\textwidth}
    	\includegraphics[width=0.95\textwidth]{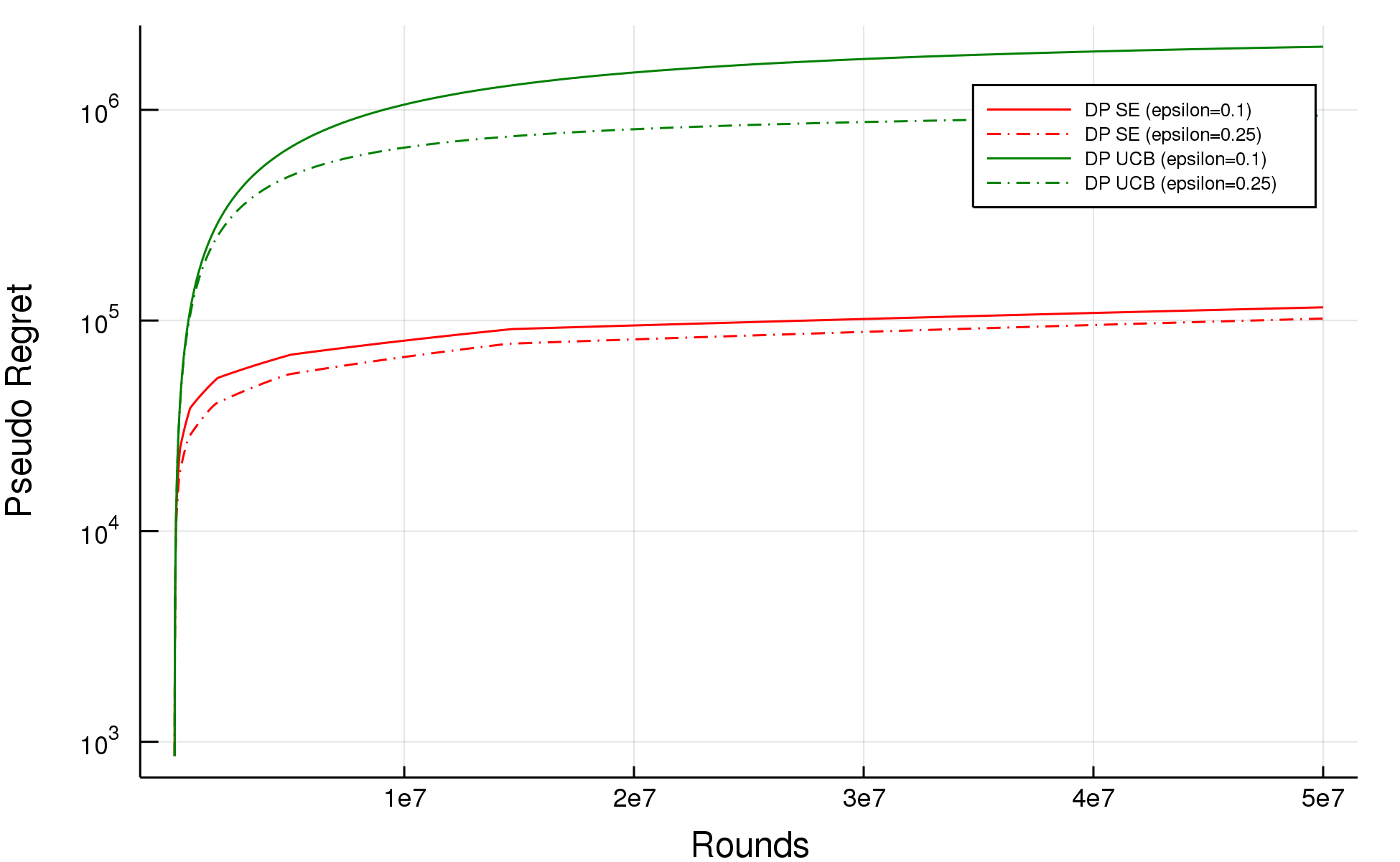}
    	\caption{$K=20$}
    \end{subfigure}
    
    \caption{\label{fig:varyK|eps0.1&0.25|setting4} Under $C_4$ with $\epsDP \in \{0.1,0.25\}, T=5 \times 10^7$}
    \end{center}
\end{figure}

\onecolumn
\restoregeometry

\section{Future Directions}
\label{sec:future_work}

While it seems this work ``closes the book'' on the private stochastic-MAB problem, we wish to point out a few future research directions. First, the MAB problem has actually multiple lower-bounds, where even low-order terms in the lower bound have been devised under different setting (see for example~\citep{Bubeck13}); so studying the lower-order terms of the pseudo regret of the private MAB problem may be of importance. Secondly, much of the work on stopping rules is devoted to the case where the variance $\sigma^2$ of the distribution is significantly smaller than its range. E.g. \citep{mnih2008empirical} give an algorithm whose sample complexity is actually $O\left(\max\{\frac{\sigma^2}{\epsStop^2\mu^2},\frac{R}{\epsStop|\mu|}\}(\log(\nicefrac{1}{\deltaStop} + \log\log(\nicefrac{R}{\epsStop|\mu|})) \right)$. Note that the lower-bound in Theorem~\ref{thm:stoppingrule_lowerbound_privacy} deals with a distribution of variance $\Theta(R^2)$, so by restricting our attention to distributions with much smaller variances we may bypass this lower-bound. We leave the problem of designing privacy-preserving analogues of the Bernstein stopping rule~\citep{mnih2008empirical} as an interesting open-problem.

Also, note that our entire analysis is restricted to $\epsDP$-DP. While our results extend to the more-recent notion of concentrated differential privacy~\citep{BunS16}, we do not know how to extend them to $(\epsDP,\delta)$-DP, nor do we know the lower-bounds for this setting. Similarly, we do not know the concrete privacy-utility bounds of the MAB problem in the local-model of DP. Lastly, it would be interesting to see if the overall approach of private Successive Elimination is applicable, and yields better bounds then currently known, for natural extensions of the MAB, such as in the linear and contextual settings. \citep{even2002pac} themselves motivated their work by various applications in a Markov-chain related setting. It is an interesting open problem of adjusting this work to such applications.

{
\bibliographystyle{abbrvnat}
\bibliography{paper}

\begin{thebibliography}{28}
\providecommand{\natexlab}[1]{#1}
\providecommand{\url}[1]{\texttt{#1}}
\expandafter\ifx\csname urlstyle\endcsname\relax
  \providecommand{\doi}[1]{doi: #1}\else
  \providecommand{\doi}{doi: \begingroup \urlstyle{rm}\Url}\fi

\bibitem[Agrawal(1995)]{Agrawal95}
R.~Agrawal.
\newblock \emph{Sample mean based index policies with O(log n) regret for the
  multi-armed bandit problem.}, volume~27, pages 1054--1078.
\newblock Applied Probability Trust, 1995.

\bibitem[Auer et~al.(2002{\natexlab{a}})Auer, Cesa-Bianchi, and
  Fischer]{Auer2002}
P.~Auer, N.~Cesa-Bianchi, and P.~Fischer.
\newblock Finite-time analysis of the multiarmed bandit problem.
\newblock \emph{{JMLR}}, 47\penalty0 (2-3):\penalty0 235--256,
  2002{\natexlab{a}}.

\bibitem[Auer et~al.(2002{\natexlab{b}})Auer, Cesa-Bianchi, Freund, and
  Schapire]{auer2002nonstochastic}
P.~Auer, N.~Cesa-Bianchi, Y.~Freund, and R.~E. Schapire.
\newblock The nonstochastic multiarmed bandit problem.
\newblock \emph{SIAM journal on computing}, 32\penalty0 (1):\penalty0 48--77,
  2002{\natexlab{b}}.

\bibitem[Berry and Fristedt(1985)]{BanditBook85}
D.~A. Berry and B.~Fristedt.
\newblock Bandit problems: sequential allocation of experiments (monographs on
  statistics and applied probability).
\newblock \emph{London: Chapman and Hall}, 5:\penalty0 71--87, 1985.

\bibitem[Bubeck et~al.(2013)Bubeck, Perchet, and Rigollet]{Bubeck13}
S.~Bubeck, V.~Perchet, and P.~Rigollet.
\newblock Bounded regret in stochastic multi-armed bandits.
\newblock In \emph{Proceedings of the 26th Annual Conference on Learning
  Theory}, pages 122--134. PMLR, 2013.

\bibitem[Bun and Steinke(2016)]{BunS16}
M.~Bun and T.~Steinke.
\newblock Concentrated differential privacy: Simplifications, extensions, and
  lower bounds.
\newblock In \emph{Theory of Cryptography - 14th International Conference,
  {TCC} 2016-B, Beijing, China, October 31 - November 3, 2016, Proceedings,
  Part {I}}, pages 635--658, 2016.

\bibitem[Caron and Bhagat(2013)]{caron2013snakdd}
S.~Caron and S.~Bhagat.
\newblock Mixing bandits: a recipe for improved cold-start recommendations in a
  social network.
\newblock In \emph{Proceedings of the 7th Workshop on Social Network Mining and
  Analysis}, page~11, 2013.

\bibitem[Chan et~al.(2010)Chan, Shi, and Song]{ChanPrivateContinualRelease2010}
T.-H.~H. Chan, E.~Shi, and D.~Song.
\newblock Private and continual release of statistics.
\newblock In \emph{Automata, {{Languages}} and {{Programming}}}, Lecture Notes
  in Computer Science, pages 405--417, 2010.

\bibitem[Dagum et~al.(2000)Dagum, Karp, Luby, and Ross]{dagum2000optimal}
P.~Dagum, R.~Karp, M.~Luby, and S.~Ross.
\newblock An optimal algorithm for monte carlo estimation.
\newblock \emph{SIAM Journal on computing}, 29\penalty0 (5):\penalty0
  1484--1496, 2000.

\bibitem[Domingo et~al.(2002)Domingo, Gavald{\`a}, and
  Watanabe]{domingo2002adaptive}
C.~Domingo, R.~Gavald{\`a}, and O.~Watanabe.
\newblock Adaptive sampling methods for scaling up knowledge discovery
  algorithms.
\newblock \emph{Data Mining and Knowledge Discovery}, 6\penalty0 (2):\penalty0
  131--152, 2002.

\bibitem[Dwork et~al.(2006)Dwork, McSherry, Nissim, and
  Smith]{DworkCalibratingNoiseSensitivity2006}
C.~Dwork, F.~McSherry, K.~Nissim, and A.~Smith.
\newblock Calibrating noise to sensitivity in private data analysis.
\newblock In \emph{Theory of {{Cryptography}}}, Lecture Notes in Computer
  Science, pages 265--284. {Springer, Berlin, Heidelberg}, 2006.

\bibitem[Dwork et~al.(2010)Dwork, Naor, Pitassi, and Rothblum]{Dwork2010}
C.~Dwork, M.~Naor, T.~Pitassi, and G.~N. Rothblum.
\newblock Differential privacy under continual observation.
\newblock In \emph{Proceedings of the Forty-second ACM Symposium on Theory of
  Computing}, STOC '10, pages 715--724, 2010.

\bibitem[Dwork et~al.(2014)Dwork, Roth, et~al.]{dwork2014algorithmic}
C.~Dwork, A.~Roth, et~al.
\newblock The algorithmic foundations of differential privacy.
\newblock \emph{Foundations and Trends{\textregistered} in Theoretical Computer
  Science}, 9\penalty0 (3--4):\penalty0 211--407, 2014.

\bibitem[Even-Dar et~al.(2002)Even-Dar, Mannor, and Mansour]{even2002pac}
E.~Even-Dar, S.~Mannor, and Y.~Mansour.
\newblock Pac bounds for multi-armed bandit and markov decision processes.
\newblock In \emph{International Conference on Computational Learning Theory},
  pages 255--270. Springer, 2002.

\bibitem[Hoeffding(1963)]{hoeffding1963probability}
W.~Hoeffding.
\newblock Probability inequalities for sums of bounded random variables.
\newblock \emph{Journal of the American statistical association}, 58\penalty0
  (301):\penalty0 13--30, 1963.

\bibitem[Hoffman et~al.(2011)Hoffman, Brochu, and de~Freitas]{Hoffman2011}
M.~Hoffman, E.~Brochu, and N.~de~Freitas.
\newblock Portfolio allocation for bayesian optimization.
\newblock In \emph{Proceedings of the Twenty-Seventh Conference on Uncertainty
  in Artificial Intelligence}, UAI'11, pages 327--336, 2011.

\bibitem[Kannan et~al.(2018)Kannan, Morgenstern, Roth, Waggoner, and
  Wu]{KannanMRWW18}
S.~Kannan, J.~H. Morgenstern, A.~Roth, B.~Waggoner, and Z.~S. Wu.
\newblock A smoothed analysis of the greedy algorithm for the linear contextual
  bandit problem.
\newblock In \emph{Advances in Neural Information Processing Systems 31: Annual
  Conference on Neural Information Processing Systems 2018, NeurIPS 2018, 3-8
  December 2018, Montr{\'{e}}al, Canada.}, pages 2231--2241, 2018.

\bibitem[Karwa and Vadhan(2018)]{karwa2017finite}
V.~Karwa and S.~Vadhan.
\newblock {Finite Sample Differentially Private Confidence Intervals}.
\newblock In \emph{9th Innovations in Theoretical Computer Science Conference
  (ITCS 2018)}, volume~94, pages 44:1--44:9, 2018.

\bibitem[Kveton et~al.(2015)Kveton, Szepesv\'{a}ri, Wen, and
  Ashkan]{Kveton2015}
B.~Kveton, C.~Szepesv\'{a}ri, Z.~Wen, and A.~Ashkan.
\newblock Cascading bandits: Learning to rank in the cascade model.
\newblock In \emph{Proceedings of the 32Nd International Conference on
  International Conference on Machine Learning - Volume 37}, ICML'15, pages
  767--776, 2015.

\bibitem[Lai and Robbins(1985)]{lai1985asymptotically}
T.~L. Lai and H.~Robbins.
\newblock Asymptotically efficient adaptive allocation rules.
\newblock \emph{Advances in applied mathematics}, 6\penalty0 (1):\penalty0
  4--22, 1985.

\bibitem[Mishra and Thakurta(2015)]{mishra2015nearly}
N.~Mishra and A.~Thakurta.
\newblock (nearly) optimal differentially private stochastic multi-arm bandits.
\newblock In \emph{Proceedings of the Thirty-First Conference on Uncertainty in
  Artificial Intelligence}, pages 592--601. AUAI Press, 2015.

\bibitem[Mnih et~al.(2008)Mnih, Szepesv{\'a}ri, and
  Audibert]{mnih2008empirical}
V.~Mnih, C.~Szepesv{\'a}ri, and J.-Y. Audibert.
\newblock Empirical bernstein stopping.
\newblock In \emph{Proceedings of the 25th international conference on Machine
  learning}, pages 672--679. ACM, 2008.

\bibitem[Robbins(1952)]{robbins1952}
H.~Robbins.
\newblock Some aspects of the sequential design of experiments.
\newblock \emph{Bull. Amer. Math. Soc.}, 58\penalty0 (5):\penalty0 527--535, 09
  1952.

\bibitem[Sajed et~al.(2018)Sajed, Chung, and White]{sajed2018}
T.~Sajed, W.~Chung, and M.~White.
\newblock High-confidence error estimates for learned value functions.
\newblock In \emph{Proceedings of the Thirty-Fourth Conference on Uncertainty
  in Artificial Intelligence, {UAI} 2018, Monterey, California, USA, August
  6-10, 2018}, pages 683--692, 2018.

\bibitem[Schwartz et~al.(2017)Schwartz, Bradlow, and Fader]{Schwartz2017}
E.~M. Schwartz, E.~T. Bradlow, and P.~S. Fader.
\newblock Customer acquisition via display advertising using multi-armed bandit
  experiments.
\newblock \emph{Marketing Science}, 36\penalty0 (4):\penalty0 500--522, July
  2017.

\bibitem[Shariff and Sheffet(2018)]{shariff2018differentially}
R.~Shariff and O.~Sheffet.
\newblock Differentially private contextual linear bandits.
\newblock In \emph{Advances in Neural Information Processing Systems}, pages
  4301--4311, 2018.

\bibitem[Smith and Thakurta(2013)]{SmithT2013}
A.~Smith and A.~Thakurta.
\newblock ({{Nearly}}) optimal algorithms for private online learning in
  full-information and bandit settings.
\newblock In \emph{{NIPS}}, pages 2733--2741, 2013.

\bibitem[Tossou and Dimitrakakis(2016)]{tossou2016algorithms}
A.~C. Tossou and C.~Dimitrakakis.
\newblock Algorithms for differentially private multi-armed bandits.
\newblock In \emph{AAAI}, pages 2087--2093, 2016.

\end{thebibliography}
}

\newpage
\appendix
\phantomsection{}
\addcontentsline{toc}{chapter}{Supplementary Material}
\begin{center}
  \LARGE\bf Supplementary Material
\end{center}


\section{Missing Proofs}
\label{apx_sex:proofs}

For completeness, we provide the proof of Fact \ref{fact:loglog_solution} below.

\paragraph{Fact from Preliminaries.}

\begin{fact}
\label{apx_fact:loglogsolution}
[Fact~\ref{fact:loglog_solution} restated.] \factLogLogSolution
\end{fact}

\begin{proof}
It is clear that the function $f(x) = \frac{\log(a\log(x))}x$ is a monotonically decreasing function for $x>e$. Plugging-in $x_0 = \nicefrac{\log(a\log(1/b))} b$ we get that 
\begin{align*}
    f(x_0) &= \frac{\log(a) + \log\log(\log(a\log(1/b))/b)}{\nicefrac{\log(a\log(1/b))} b}
    \cr &= b \cdot \frac{\log(a) + \log\log(1/b) +\log\log \left(\log a + \log\log(1/b) \right)}{\log(a) + \log\log(1/b)} > b
\end{align*}
Plugging-in $x_1 = \nicefrac{2\log(a\log(1/b))} b$
we get that 
\begin{align*}
    f(x_1) &= \frac{\log(a) + \log\log(2\log(a\log(1/b))/b)}{\nicefrac{2\log(a\log(1/b))} b}
    \cr &= b \cdot \frac{\log(a) + \log\log 2 + \log\log(1/b) +\log\log \left(\log a + \log\log(1/b) \right)}{2\log(a) + 2\log\log(1/b)} 
    \cr &< b\cdot \frac{\log(a) + \log\log(1/b)  +\log\log \left(\log a + \log\log(1/b) \right)}{2\log(a) + 2\log\log(1/b)} 
    \cr &= b \left(\frac 1 2 + \frac {\log\log \left(\log a + \log\log(1/b) \right)}{2\log(a) + 2\log\log(1/b)}  \right) < b
\end{align*}
And so due to monotonicity, the claim follows.
\end{proof}

\end{document}